\documentclass[11pt]{article}

\usepackage[final]{acl}

\usepackage{times}
\usepackage{latexsym}
\usepackage{tabularx}
\usepackage{array}

\usepackage[T1]{fontenc}

\usepackage[utf8]{inputenc}

\usepackage{microtype}

\usepackage{inconsolata}

\usepackage{graphicx}

\usepackage{multirow}
\usepackage{graphicx}
\usepackage{algorithm}
\usepackage{algpseudocode}
\usepackage{booktabs}
\usepackage{tcolorbox}
\usepackage{amssymb}
\usepackage{amsmath}
\usepackage{array}
\usepackage{longtable}
\usepackage{xcolor}
\usepackage{makecell}
\usepackage{adjustbox}
\usepackage{circledsteps}
\usepackage{threeparttable}
\usepackage{amsthm}
\usepackage[table]{xcolor}
\usepackage{float}
\usepackage{listings}
\usepackage{tikz}
\usetikzlibrary{arrows.meta,positioning,fit,shapes.geometric}
\title{Tool-Aware Planning in Contact Center AI: Evaluating LLMs through Lineage-Guided Query Decomposition}

\author{Varun Nathan, Shreyas Guha\thanks{\hspace{1pt} Work done during internship at Observe.AI}, \and Ayush Kumar \\
        \texttt{\{varun.nathan,  shreyas.guha, ayush\}@observe.ai}\\
        Observe.AI \\ Bangalore, India}

\begin{document}
\maketitle

\begin{abstract}
We present a domain-grounded framework and benchmark for \emph{tool-aware plan generation} in contact centers, where answering a query for business insights, our target use case, requires decomposing it into executable steps over structured tools (Text2SQL (T2S)/Snowflake) and unstructured tools (RAG/transcripts) with explicit \texttt{depends\_on} for parallelism. Our contributions are threefold: (i) a reference-based plan evaluation framework operating in two modes - a metric-wise evaluator spanning seven dimensions (e.g., tool–prompt alignment, query adherence) and a one-shot evaluator; (ii) a data curation methodology that iteratively refines plans via an evaluator$\rightarrow$optimizer loop to produce high-quality plan lineages (\emph{ordered plan revisions}) while reducing manual effort; and (iii) a large-scale study of 14 LLMs across sizes and families for their ability to decompose queries into step-by-step, executable, and tool-assigned plans, evaluated under prompts with and without lineage.
Empirically, LLMs struggle on compound queries and on plans exceeding 4 steps (typically 5–15); the best total metric score reaches \textbf{84.8\%} (Claude-3-7-Sonnet), while the strongest one-shot match rate at the ``A+'' tier (Extremely Good, Very Good) is only \textbf{49.75\%} (o3-mini). Plan lineage yields mixed gains overall but benefits several top models and improves step executability for many. Our results highlight persistent gaps in tool-understanding, especially in tool–prompt alignment and tool-usage completeness, and show that shorter, simpler plans are markedly easier. The framework and findings provide a reproducible path for assessing and improving agentic planning with tools for answering data-analysis queries in contact-center settings.
\end{abstract}

\section{Introduction and Related Works}

\noindent\textbf{Use case and motivation.}
We target contact-center Question Answering (Insights) and data analytics, where plans must orchestrate \emph{structured} analysis via Text2SQL over Snowflake (T2S) and extract \emph{unstructured} insights via RAG over transcripts under tight latency and correctness constraints. LLMs are increasingly used as \emph{agents} that decompose goals into multi-step plans and invoke external tools, and a growing body of benchmarks probes tool-augmented reasoning and agent behavior across APIs, web environments, and multi-modal settings \citep{apibank,agentbench,callnavi,mind2web,mmsbenchmark,hastemakeswaste,planbench,wang2024mintevaluatingllmsmultiturn,insight_bench} as well as neuro-symbolic and optimization hybrids \citep{lip_llm}. However, these works do not focus on contact-center planning with overlapping tools and parallel execution.

\textbf{Why contact-centers need a different planning lens.} Queries in contact-centers (e.g., “How did escalation rates and QA outcomes differ by time zone?”) require \emph{tool-aware plans} that orchestrate structured data (e.g., Text2SQL over Snowflake) and unstructured data (RAG over Transcripts), with \emph{correct argument binding} (filters vs. prompts) and \emph{explicit dependencies} for \emph{parallel} execution under latency constraints. Errors in tool choice, placeholders, or dependency wiring can silently degrade results - making plan \emph{correctness}, \emph{executability}, and \emph{format fidelity} first-class requirements in production systems.

\textbf{Limitations of existing benchmarks.}
Prior agent/planning benchmarks typically assume (i) \emph{one correct tool per sub-task}, (ii) \emph{tight I/O coupling} (step outputs directly consumable by the next), and (iii) primarily \emph{sequential} execution, rarely modeling domains with \emph{tool overlap}, \emph{loose I/O coupling}, and required concurrency. Survey works \citep{llm_planning_survey,llm_planning_modelers_survey,llm_auto_planning_scheduling} argue that LLM planners need \emph{structure}, \emph{feedback}, and often \emph{hybridization} to be reliable on long-horizon plans. Recent work on structured planning interfaces \citep{nl2flow,catp_llm,prompt2dag} and on hybrid querying or data-agent benchmarks \citep{liu-suql,wang2025fdabench} improves executability or final-answer quality over heterogeneous data, but does not define multi-step, tool-aware plan representations or rubrics for planning quality in contact-center settings with overlapping tools and lineage-guided revision. Complementary evaluations \citep{llm_reasoning_models_replace_classical_planning} show that LLMs struggle with \emph{constraint compliance} and \emph{state consistency} in complex tasks, reinforcing the need for explicit plan evaluation.

\textbf{Gap.}
To our knowledge, there is \emph{no benchmark or evaluation protocol} that (1) targets \emph{contact-center} queries, (2) requires \emph{parallel tool usage} when multiple tools (e.g., RAG and Text2SQL) contain relevant information that must be combined, (3) enforces \emph{argument/placeholder correctness} and \emph{dependency wiring} for \emph{parallel} execution, and (4) captures \emph{plan evolution} through \emph{lineage} (intermediate, interpretable revisions) driven by an iteratively run \emph{step-wise evaluator} and \emph{plan optimizer}. Existing datasets and metrics do not jointly assess \emph{tool–prompt alignment}, \emph{step executability}, \emph{format/placeholder correctness}, \emph{dependency correctness}, \emph{redundancy}, and \emph{tool-usage completeness}, nor relate these to a one-shot \emph{plan-to-reference} comparison of structural closeness (precision/recall/F1).

\textbf{Our approach and contributions.} We study tool-aware plan generation for \emph{contact-centers} with a fixed triad of tools:
\textbf{T2S} (Text-to-SQL over Snowflake), \textbf{RAG} (transcripts), and \textbf{LLM} (synthesis/reformatting).
Our contributions are:

\begin{enumerate}
    \item \textbf{Dual-perspective evaluation framework:} A \emph{7-metric} rubric with learned aggregation (0–100) and a complementary \emph{one-shot} plan-to-reference evaluator (precision/recall/F1 + format) with a 7-point quality rating.
    \item \textbf{Curation for lineage-driven planning:} An iterative \emph{step-wise evaluator} $\rightarrow$ \emph{plan optimizer} loop that produces improving plan lineages and reduces human editing effort; while the dataset is proprietary, we release the \emph{schema, prompts, and methodology}.
    \item \textbf{Model study:} A comparison of \emph{14 LLMs} (e.g., o3-mini \cite{openai-o3-mini-2025}, GPT-4o/mini \cite{openai2024gpt4ocard}, Claude, Llama, and Nova families \cite{claude-3_5,llama3_meta_research,llama4maverick,nova_models}) across \emph{subjectivity}, \emph{compoundness}, and \emph{plan length}, including the effect of \emph{plan-lineage prompting}.
\end{enumerate}


\section{Task Formalization}
\label{sec:task-formalization}

\paragraph{Problem setup.}
Given a fixed tool set $T$ and a natural-language query $Q$, a planner must return an \emph{executable, tool-aware plan} $P$.

\paragraph{Plan schema.}
We represent a plan as an ordered sequence of steps:
\[
P \;=\; \langle s_k \rangle_{k=1}^{n}, \quad 
s_k \;=\; (\, t_k,\, p_k,\, D_k \,)
\]
where $n$ is the number of steps, $t_k \in T$ is the tool chosen for step $k$, $p_k$ is the tool instruction (prompt) for step $k$, and $D_k \subseteq \{1,\dots,k-1\}$ is the set of dependency indices indicating which prior steps $s_i$ must finish before $s_k$ can execute. By construction, each $D_k$ may only contain indices of earlier steps, so all edges point from smaller to larger $k$, which rules out cycles. The dependency graph induced by $\{D_k\}_{k=1}^n$ is required to be acyclic (a DAG), enabling safe parallel execution of independent steps. Before scoring, we also run a simple plan validator that checks these index constraints and performs a topological pass to reject any plan whose induced graph is not acyclic. In practice, plans are materialized as JSON objects whose keys are step indices and whose values store \texttt{query} and \texttt{depends\_on} fields.

\paragraph{Plan lineage.}
For a query $Q_i$, our curated data stores a \emph{plan lineage}, i.e., an ordered list
\[
\mathcal{L}(Q_i) \;=\; \langle P^{(0)}_i,\, P^{(1)}_i,\,\dots,\, P^{(M_i)}_i \rangle,
\]
where $P^{(0)}_i$ is the initial (typically weakest) plan and $P^{(M_i)}_i$ is the best (reference) plan. Each subsequent plan is produced from the previous by an evaluator$\rightarrow$optimizer loop that corrects tools, prompts, and dependencies (see \S\ref{sec:dataset}).

\paragraph{Tool set.}
We use three internally built tools (full details in Appendix~\ref{app:tools}):
\begin{itemize}
    \item \textbf{T2S (Text-to-SQL over Snowflake):} answers queries using structured contact-center data (e.g., call drivers, key moments during interactions, Qualtiy-Assurance metrics).
    \item \textbf{RAG (Retrieval-Augmented Generation):} answers queries using customer--agent transcripts.
    \item \textbf{LLM (Synthesis/Reformatting):} composes, reformats, and aggregates outputs (e.g., extracting \texttt{call\_ids} from RAG tables; merging multi-tool evidence).
\end{itemize}

\paragraph{Execution constraints.}
Given dependencies $D_k$, step $s_k$ may only reference prior outputs via placeholders (e.g., ``(3)'') and must not redundantly re-filter on criteria already encoded by its dependency inputs.

\begin{lstlisting}[caption={Example query-plan pair}, label={lst:example-plan}]
# Query: Compare QA scores for professionalism and resolution procedures in unresolved calls where sentiment transitioned from negative to positive.
# Plan:
{
  "1": {"query": "T2S([], 'Fetch interaction_ids of unresolved calls')", "depends_on": []},
  "2": {"query": "RAG((1), 'Fetch calls where the sentiment transitioned from negative to positive within the transcript')", "depends_on": [1]},
  "3": {"query": "LLM('Extract interaction_ids from Data Insights in (2).')", "depends_on": [2]},
  "4": {"query": "T2S((3), 'Retrieve QA scores for resolution procedures in these calls.')", "depends_on": [3]},
  "5": {"query": "T2S((3), 'Retrieve QA scores for professionalism in these calls.')", "depends_on": [3]},
  "6": {"query": "LLM('Compare QA scores from (4) vs. (5) in light of unresolved status and sentiment transitions.')", "depends_on": [4,5]}
}
\end{lstlisting}

\paragraph{Example lineage (excerpt).}
Due to space limits we show and describe the final plan; full lineage appears in Appendix~\ref{app:lineage-example}.\\
\textbf{Final Plan (best):} the 6-step plan (shown above) separates filtering for unresolved calls (T2S) from within-call sentiment shift (RAG), extracts IDs (LLM), and aggregates QA metrics (T2S) before synthesis.

\section{Dataset Generation Methodology}
\label{sec:dataset}

We build a dataset of queries and plan lineages via a four-stage pipeline:
\Circled{1} controlled \emph{query generation};
\Circled{2} one-shot \emph{initial plan generation};
\Circled{3} an \emph{iterative evaluator$\rightarrow$optimizer} loop that produces a \emph{plan lineage} per query; and
\Circled{4} \emph{human verification} of the final plans.

\paragraph{\Circled{1} Query generation.}
We use \textbf{GPT-4o} to synthesize queries along two axes: (i) \emph{subjectivity} (objective vs.\ subjective) and (ii) \emph{compoundness} (simple vs.\ compound), covering measurable vs.\ interpretive asks and single vs.\ multi-ask structures. Appendix~\ref{app:query-gen} and Tables~\ref{tab:prompts_for_query_and_plan_gen},~\ref{tab:examples_by_query_types} give dimension definitions, prompts, and examples. Model details appear in Appx.~\ref{app:model-configs}.

\paragraph{\Circled{2} Initial plan generation.}
Plans are produced one-shot by an LLM using a two-part prompt (system: task + tool schema; user: formatting + examples). We compare low/medium/high tool-detail and few-shot counts; a \emph{medium-detail} prompt with $6$--$8$ examples gives the best accuracy without overload (App.~\ref{app:initial-plan}, Table~\ref{tab:prompts_for_query_and_plan_gen}).

\paragraph{\Circled{3} Iterative Evaluator$\rightarrow$Optimizer Feedback Loop.}
\label{sec:loop-brief}
We refine one-shot initial plans via a lightweight, non-executing feedback loop (Fig.~\ref{fig:iterative-loop}). Each \emph{pass} freezes the current plan \(P\) and visits each step once. A \emph{Step-wise Evaluator} inspects the step’s tool, prompt, and dependencies and emits diagnostic tags (e.g., \textsc{IncorrectTool}, \textsc{ComplexPrompt}, \textsc{RepeatedDetail}, \textsc{NoChange}). A \emph{Plan Optimizer} consumes these tags and decides whether to edit the step, applying local repairs and global coherence fixes while maintaining a valid DAG (placeholders, dependency closure), possibly splitting or merging steps.

Let \(P'\) be the optimizer’s output after processing the current step. We append \(P'\) to the \emph{plan lineage} only if \(P' \neq P\), then set \(P \leftarrow P'\); otherwise we keep \(P\) and move on. A pass ends once every step is visited. The loop stops when (i) a full pass makes no change or (ii) the pass count exceeds \(\texttt{max\_passes}=4\). Because we record only distinct plans, the lineage forms an ordered sequence from weakest to best. The loop never executes tools; it edits plan \emph{text} only, enabling scalable curation without runtime calls. Full pseudocode and prompts appear in App.~\ref{app:iterative-loop}, Alg.~\ref{alg:plan-optimizer}, and Tables~\ref{tab:prompts_for_step_wise_eval}--\ref{tab:prompts_for_plan_optimizer}.

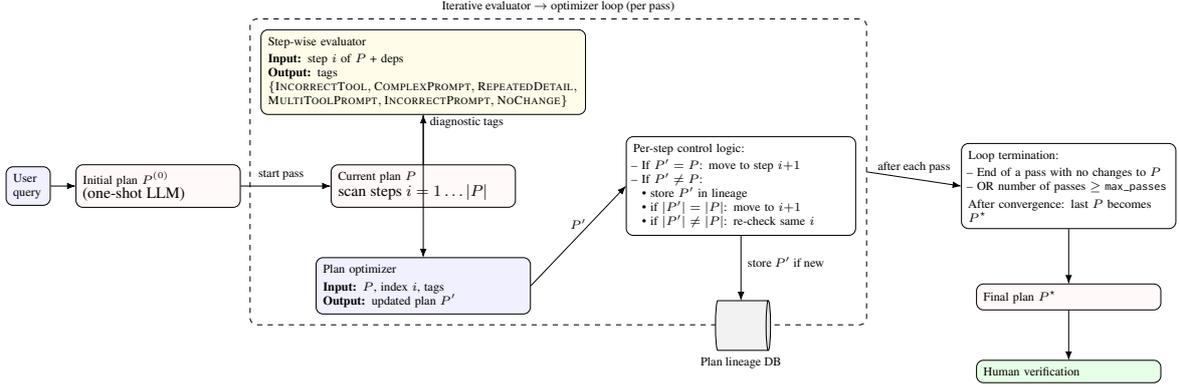
\begin{figure*}[t]
  \centering
  \resizebox{2.0\columnwidth}{!}{%
  \begin{tikzpicture}[
      >=Latex,
      node distance=1.5cm and 2.1cm,
      every node/.style={font=\scriptsize},
      box/.style={rectangle, rounded corners, draw, align=left, inner sep=4pt, text width=3.4cm},
      plan/.style={box, fill=pink!10},
      proc/.style={box, fill=blue!6},
      judge/.style={box, fill=yellow!10},
      human/.style={box, fill=green!10},
      loopbox/.style={draw,dashed,rounded corners,inner sep=6pt},
      db/.style={cylinder, draw, shape aspect=0.35,
                 cylinder uses custom fill, cylinder body fill=gray!10,
                 cylinder end fill=gray!20, minimum height=1.1cm, minimum width=1cm}
    ]

    \node[proc, text width=0.6cm] (query) {User query};
    \node[plan, right=0.5cm of query, text width=3.0cm] (p0) {Initial plan $P^{(0)}$\\\footnotesize(one-shot LLM)};
    \draw[->] (query) -- (p0);

    \node[plan, right=1.8cm of p0] (curr) {Current plan $P$\\\footnotesize{scan steps $i=1\ldots |P|$}};

    \node[judge, above=1.0cm of curr, text width=6.2cm] (eval) {Step-wise evaluator\\[2pt]
      \textbf{Input:} step $i$ of $P$ + deps\\
      \textbf{Output:} tags $\{\textsc{IncorrectTool}, \textsc{ComplexPrompt}, \textsc{RepeatedDetail}, $\\ $\textsc{MultiToolPrompt}, \textsc{IncorrectPrompt}, \textsc{NoChange}\}$
    };

    \node[proc, below=1.0cm of curr, text width=4.0cm] (opt) {Plan optimizer\\[2pt]
      \textbf{Input:} $P$, index $i$, tags\\
      \textbf{Output:} updated plan $P'$
    };

    \node[box, right=2.2cm of curr, text width=4.3cm] (ctrl) {Per-step control logic:\\[2pt]
      -- If $P'=P$: move to step $i{+}1$\\
      -- If $P'\neq P$:\newline\hspace*{0.7em}• store $P'$ in lineage\\
      \hspace*{0.7em}• if $|P'|=|P|$: move to $i{+}1$\newline
      \hspace*{0.7em}• if $|P'|\neq|P|$: re-check same $i$
    };

    \draw[->] (p0) -- node[above,pos=0.45]{start pass} (curr);
    \draw[->] (curr.north) -- (eval.south);
    \draw[->] (eval.south) -- node[right,pos=0.05]{diagnostic tags} (opt.north);
    \draw[->] (opt.east)  -- node[above,pos=0.5]{$P'$} (ctrl.west);

    \node[loopbox, fit=(curr) (eval) (opt) (ctrl),
          label={[font=\scriptsize]above:Iterative evaluator $\rightarrow$ optimizer loop (per pass)}] (loop) {};

    \node[db, below=1.3cm of ctrl] (dbnode) {};
    \node[below=0.0cm of dbnode] {\scriptsize{Plan lineage DB}};
    \draw[->] (ctrl.south) -- node[right,pos=0.4]{store $P'$ if new} (dbnode.north);

    \node[box, right=2.1cm of ctrl, text width=4.0cm] (term) {Loop termination:\\[2pt]
      -- End of a pass with no changes to $P$\\
      -- OR number of passes $\geq$ \texttt{max\_passes}\\[2pt]
      After convergence: last $P$ becomes $P^\star$
    };

    \node[plan, below=1.1cm of term] (final) {Final plan $P^\star$};
    \draw[->] (term.south) -- (final.north);

    \node[human, below=1.0cm of final] (human) {Human verification};
    \draw[->] (final.south) -- (human.north);

    \draw[->] (loop.east) -- node[above,pos=0.5]{after each pass} (term.west);

  \end{tikzpicture}%
  }
  \caption{Iterative step-wise evaluator $\rightarrow$ plan optimizer loop producing a plan lineage. For each pass, every step $i$ of the current plan $P$ is diagnosed by the step-wise evaluator and optionally edited by the plan optimizer. Any updated plan $P'$ is appended to the lineage database. The loop stops when a full pass yields no changes or when the maximum number of passes is reached, and the final plan $P^\star$ is human-verified.}
  \label{fig:iterative-loop}
\end{figure*}

\paragraph{\Circled{4} Human verification.}
Final lineage heads are reviewed by expert annotators. Minor fixes (if any) are applied to certify the \emph{best} plan. The lineage (all intermediate plans) is retained for analysis and training.




\paragraph{Evaluation metadata.}
After obtaining the human-verified best plan, we derive lightweight \emph{metadata} from the lineage and final plan (without executing tools) for stratified analyses: (i) lineage length (distinct plans), (ii) number of passes to convergence, (iii) per-pass revision count, (iv) step count, and (v) \emph{number of hops} (Table~\ref{tab:loop-stats}).

\paragraph{Number of Hops}
We compute a query’s \emph{number of hops} from the best plan’s dependency graph as the length of the longest dependency chain (App.~\ref{app:hops-definition}): 
\emph{zero-hop} plans have no dependencies (direct tool calls yield the answer); 
\emph{one-hop} plans have exactly one dependency layer (e.g., an LLM step depending on two independent producers); 
\emph{two-hop} plans have two sequential dependency layers (e.g., a final LLM step depends on RAG/T2S steps that themselves depend on a T2S filter); 
\emph{three-plus} plans have three or more sequential dependency layers.
We use this hop count as metadata for stratifying LLM performance.


\noindent\emph{Extended details.} See Tables~\ref{tab:example_lineage_evaluation_1}--\ref{tab:example_lineage_evaluation_3} for an end-to-end lineage from \textbf{Query} $\rightarrow$ \textbf{Step-Wise Evaluator / Plan Optimizer} outputs $\rightarrow$ \textbf{Final Plan}.

\section{Plan Evaluation Methodology}
\label{sec:evaluation}

We evaluate plans in two complementary modes: a \emph{reference-based metric-wise} framework that aggregates seven targeted metrics into a 0–100 score, and a \emph{reference-based one-shot} evaluator that measures closeness to a best-possible plan via step-level Precision/Recall/F$_1$ and a seven-point rating.

\paragraph{Overall score (metric-wise framework).}
Let $\mathcal{E}$ (Effectiveness) and $\mathcal{F}$ (Efficiency) be metric sets with scores $m_k \in [0, 1]$ and weights $w_k$ such that $\sum_{k \in \mathcal{E}} w_k = 0.7$ and $\sum_{k \in \mathcal{F}} w_k = 0.3$. The overall score is
\[
\textsc{Score}(P) \;=\; 100 \cdot \sum_{k \in \mathcal{E} \cup \mathcal{F}} w_k m_k.
\]

\subsection{Metric-Wise Evaluation}
\label{sec:metricwise}
\textbf{Effectiveness (70 pts)} assesses correctness and executability via:
(i) \textbf{Tool–Prompt Alignment} (20): tool choice matches prompt intent;
(ii) \textbf{Format Correctness} (20): JSON-parseable, structurally valid output;
(iii) \textbf{Step Executability / Atomicity} (15): each step is a single, atomic operation;
(iv) \textbf{Query Adherence} (15): a perfect execution would fully answer the query.

\medskip
\noindent\textbf{Efficiency (30 pts)} assesses optimality via:
(v) \textbf{Dependencies} (10): correct, minimal \texttt{depends\_on} wiring;
(vi) \textbf{Redundancy} (10): no duplicated work or unnecessary repeated filters;
(vii) \textbf{Tool-Usage Completeness} (10): when a sub-task clearly benefits from both \textsc{T2S} and \textsc{RAG}, the plan includes both (in sequence or parallel).

\paragraph{Evaluator pipeline and metric weighting.}
All seven metrics are reference-based and scored by specialist LLM evaluators using rubric prompts (Tables~\ref{tab:prompts_for_metric_wise_eval_1}--\ref{tab:prompts_for_metric_wise_eval_3}; App.~\ref{app:metric-details}). Per-metric weights are learned on the validation set to enforce monotonic improvement across the top three plans in each lineage (best $>$ penultimate $>$ antepenultimate) under fixed group budgets (Effectiveness $=0.7$, Efficiency $=0.3$; App.~\ref{app:weight-learning}).

\subsection{One-Shot Overall Evaluation}
\label{sec:oneshot}
Given candidate plan $P$ and best plan $P^\star$, the Judge LLM computes step-level Precision/Recall/F$_1$ and checks Format, Dependencies, and Placeholders on $P$. We map $F_1$ to a seven-point rating: \emph{Extremely Good} ($>$95\%), \emph{Very Good} ($>$85\%), \emph{Good} ($>$75\%), \emph{Acceptable} ($>$60\%), \emph{Bad} ($>$45\%), \emph{Very Bad} ($>$30\%), \emph{Extremely Bad} ($\le 30\%$). Non-JSON-parseable plans are assigned \emph{Extremely Bad} regardless of $F_1$. Full prompt and schema appear in Table~\ref{tab:prompts_for_oneshot_evaluator} and App.~\ref{app:oneshot-details}.

\section{Experimental Setup}
\label{sec:exp-setup}

\paragraph{Data splits and usage.}
We curate 600 contact-center queries and their plan lineages using the iterative Evaluator$\rightarrow$Optimizer loop (Sec.~\ref{sec:dataset} and Appx.~\ref{app:dataset-details}). We stratify by subjectivity and compoundness. Splits: \emph{Train} (20) for prompt construction, \emph{Validation} (80) for module tuning (metric-wise evaluators, one-shot Judge, step-wise evaluator, plan optimizer), and \emph{Test} (500) for model benchmarking.
Due to proprietary constraints, the full dataset used in our experiments (queries from real Observe.AI customers) cannot be released.\footnote{We provide full prompts and scoring rubrics to facilitate replication on non-proprietary corpora; see Tables~\ref{tab:prompts_for_query_and_plan_gen}--\ref{tab:prompts_for_oneshot_evaluator} and Appx.~\ref{app:infra}. In addition, we release a separate 200-query public dataset built from LLM-generated queries modeled on a synthetic test account, with all queries, plans, and lineages anonymized (account names/IDs, policies, PII, and person names). It contains 100 Train and 100 Test queries with human-annotated best plans; Train lineages are fully human-edited, Test lineages are produced by our evaluator$\rightarrow$optimizer loop. We also include per-planner generated plans for all the 14 planners; see Appx.~\ref{app:public-dataset}.}
A detailed, step-by-step description of how we sampled, annotated, validated, and finalized lineages and gold plans is provided in Appx.~\ref{app:dataset-curation-steps}.

\paragraph{Plan generation models.}
We benchmark 14 LLMs for one-shot plan generation under two prompts: \emph{without lineage} and \emph{with lineage} (the latter includes per-query plan lineage examples sourced via our feedback loop). Models: \textit{Claude-3-7-Sonnet/Claude-3-5-Haiku \cite{claude-3_5}, Claude-Sonnet-4 \cite{claude-4}, Nova-Pro/Nova-Lite/Nova-Micro \cite{nova_models}, Llama3-2-1B-Instruct/Llama3-2-3B-Instruct/Llama3-70B-Instruct \cite{llama3_meta_research}, Llama4-Maverick-17B-Instruct \cite{llama4maverick}, GPT-4o/GPT-4o-Mini \cite{openai2024gpt4ocard}, GPT-4.1-Nano \cite{gpt4_1}, o3-Mini (medium reasoning) \cite{openai-o3-mini-2025}}. Model configurations and prompt templates are in ~\ref{app:model-configs} and Tables~\ref{tab:decoding} \& \ref{tab:prompts_for_query_and_plan_gen}.

\paragraph{Evaluation models.}
Both evaluators are LLM-based. We use \textit{Claude-Sonnet-4} as the Judge LLM for (i) the one-shot overall evaluation (Precision/Recall/F$_1$ + 7-point rating) and (ii) all seven metric-wise evaluators (Sec.~\ref{sec:evaluation}). The same rubric prompts are used across splits (~\ref{app:metric-details}, \ref{app:oneshot-details}).

\paragraph{Feedback-loop modules.}
For the iterative plan lineage construction, we use \textit{GPT-4o} for both the Step-wise Evaluator and the Plan Optimizer (~\ref{app:loop-details}). We cap passes via a single hyperparameter \texttt{max\_passes} for which we adopt \textbf{4} (empirically balances gains vs.\ latency) in our runs.

\paragraph{Implementation.}
All generation, evaluation, and feedback loops are implemented in Python on local infrastructure. Model access uses Bedrock and LiteLLM; we fix random seeds and use deterministic decoding for evaluators. Infra, seeds, rate limiting, and caching are detailed in ~\ref{app:infra}.

\paragraph{Human agreement.}
All human annotations used for tuning and validating evaluators exhibit substantial agreement; detailed protocol and $\kappa$ statistics (with CIs) are reported in ~\ref{app:human-agreement}.

\begin{table*}[!ht]
    \centering
    \small

    \setlength{\tabcolsep}{6pt}

    \begin{tabularx}{\textwidth}{l|ccc|ccc}
        \toprule
        \multirow{2}{*}{\textbf{LLM}} 
        & \multicolumn{3}{c|}{\textbf{With Lineage}} 
        & \multicolumn{3}{c}{\textbf{Without Lineage}} \\
        \cmidrule(lr){2-4} \cmidrule(lr){5-7}
            & \makecell{(\emph{A+})\\ Extr. good,\\ very good \\(\%)} 
            & \makecell{(\emph{A})\\ Extr. good,\\ very good,\\ good \\ (\%)} 
            & \makecell{(\emph{B})\\ Extr. good,\\ very good,\\ good,\\ acceptable \\ (\%)} 
            & \makecell{(\emph{A+})\\Extr. good,\\ very good \\ (\%)} 
            & \makecell{(\emph{A})\\Extr. good,\\ very good,\\ good \\ (\%)} 
            & \makecell{(\emph{B})\\Extr. good,\\ very good,\\ good,\\ acceptable \\ (\%)} \\
        \midrule
        o3-mini  & 43.15 & 53.30 & 72.59 & \cellcolor{green}\textcolor{magenta}{\textbf{49.75}} & 65.99 & 80.20 \\
        gpt-4o  & \textcolor{blue}{\textbf{45.69}} & 62.44 & 81.22 & 41.12 & 57.87 & 82.74 \\
        gpt-4o-mini  & 31.98 & 48.73 & 64.47 & 31.98 & 52.28 & 71.57 \\
        claude-3-5-haiku  & 24.87 & 45.69 & 62.44 & \textcolor{magenta}{\textbf{30.46}} & 47.72 & 72.59 \\
        claude-sonnet-4  & \textcolor{blue}{\textbf{30.46}} & 52.79 & 73.60 & 29.95 & 49.75 & 74.11 \\
        llama4-maverick-17b-instruct  & 20.30 & 36.55 & 54.31 & \textcolor{magenta}{\textbf{20.81}} & 40.61 & 62.44 \\
        nova-pro  & 18.27 & 42.64 & 61.93 & \textcolor{magenta}{\textbf{19.80}} & 41.62 & 59.39 \\
        claude-3-7-sonnet & \textcolor{blue}{\textbf{30.46}} & 48.22 & 70.56 & 19.80 & 39.09 & 69.04 \\
        llama3-70b-instruct  & \textcolor{blue}{\textbf{18.78}} & 44.16 & 59.90 & 17.26 & 35.53 & 55.33 \\
        gpt-4.1-nano  & \textcolor{blue}{\textbf{17.26}} & 28.43 & 55.33 & 16.24 & 30.46 & 53.30 \\
        nova-micro & \textcolor{blue}{\textbf{20.81}} & 35.53 & 53.30 & 15.74 & 37.06 & 53.81 \\
        nova-lite  & 13.71 & 30.96 & 47.24 & 13.71 & 34.52 & 47.72 \\
        llama3-2-3b-instruct  & 5.08 & 10.66 & 17.26 & \textcolor{magenta}{\textbf{5.58}} & 11.68 & 26.90 \\
        llama3-2-1b-instruct  &  0.00 &  0.00 &  0.00 &  0.00 &  0.00 &  1.52 \\
        \midrule
        \textbf{Grand Total (\#)} & \multicolumn{3}{c|}{\textbf{500}} & \multicolumn{3}{c}{\textbf{500}} \\
        \bottomrule
    \end{tabularx}
    \caption{Plan generation quality comparison using prompts \textbf{with and without lineage}, evaluated with the \textbf{one-shot evaluator} on test data. The highest score in the \textbf{Extremely Good, Very Good} bucket is highlighted in green; blue indicates better performance with lineage, and magenta indicates better performance without lineage.}
    \label{tab:plan_gen_result_overall_one_shot_eval}
\end{table*}

\newcolumntype{Y}{>{\centering\arraybackslash}X}
\begin{table*}[!ht]
    \centering
    \small
    
    
    \setlength{\tabcolsep}{6pt}

    \begin{tabularx}{\textwidth}{l|*{8}{Y}}
        \hline
        \thead{\textbf{LLM}} & \thead{\textbf{Overall} \\ {[0-100]}} & \thead{\textbf{Format} \\ {[0-20]}} & \thead{\textbf{Tool} \\ \textbf{Prompt} \\ \textbf{Align.} \\ {[0-20]}} & \thead{\textbf{Step} \\ \textbf{Exec.} \\ {[0-15]}} & \thead{\textbf{Query} \\ \textbf{Adhr.} \\ {[0-15]}} & \thead{\textbf{Depend.} \\ {[0-10]}} & \thead{\textbf{Redund.} \\ {[0-10]}} & \thead{\textbf{Tool} \\ \textbf{Usage} \\ \textbf{Compl.} \\ {[0-10]}}\\
        \hline
        claude-3-7-sonnet & \textcolor{blue}{\textbf{84.8}} & \textcolor{blue}{\textbf{18.46}} & \textcolor{blue}{\textbf{15.32}} & 12.61 & 12.6 & 9.78 & 8.97 & 7.07 \\ 
        llama4-maverick-17b-instruct & 82.26 & 17.14 & 14 & 13.09 & 12.69 & 9.35 & 9.41 & 6.59 \\ 
        gpt-4o & 81.47 & 16.42 & 13.77 & 13 & 12.85 & 9.66 & 8.58 & 7.2 \\ 
        claude-sonnet-4 & 79.9 & 12.89 & 14.68 & 13.18 & 12.35 & 9.68 & 8.59 & \textcolor{blue}{\textbf{8.54}} \\ 
        llama3-3-70b-instruct & 79.77 & 17.04 & 14.52 & 13.54 & 12.1 & 9.55 & 9.36 & 3.66 \\ 
        nova-pro & 79.24 & 15.51 & 14.57 & \textcolor{blue}{\textbf{14.32}} & 12.31 & \textcolor{blue}{\textbf{9.81}} & 9.36 & 3.35 \\ 
        nova-micro & 77.28 & 18.25 & 13.42 & 12.92 & 12.08 & 9.58 & \textcolor{blue}{\textbf{9.44}} & 1.59 \\ 
        gpt-4o-mini & 77.26 & 14.06 & 14.54 & 13.58 & 12.34 & 8.99 & 9.12 & 4.63 \\ 
        gpt-4.1-nano & 77.16 & 13.48 & 13.88 & 13.36 & 11.92 & 9.09 & 8.79 & 6.65 \\ 
        o3-mini & 71.85 & 10.9 & 14.79 & 10.91 & \textcolor{blue}{\textbf{13.58}} & 8.68 & 9.32 & 3.66 \\ 
        claude-3-5-haiku & 71.27 & 12.62 & 12.77 & 11.6 & 11.37 & 8.42 & 8.14 & 6.34 \\ 
        nova-lite & 70.53 & 13.84 & 12.92 & 12.58 & 11.93 & 8.37 & 9.42 & 1.46 \\ 
        llama3-2-3b-instruct & 70.51 & 16.62 & 10.94 & 10.65 & 9.74 & 9.23 & 9.07 & 4.27 \\ 
        llama3-2-1b-instruct & 20.27 & 0 & 0.07 & 0 & 10.5 & 4.23 & 2.24 & 3.23 \\ \hline
        \textbf{Average (Normalized)} & \textbf{73.11} & \textcolor{red}{\textbf{70.43}} & \textcolor{red}{\textbf{64.36}} & \textbf{78.73} & \textbf{80.18} & \textbf{88.88} & \textbf{85.56} & \textcolor{red}{\textbf{48.74}} \\ \hline
    \end{tabularx}
    \caption{Plan generation quality using prompts \textbf{with lineage}, evaluated with the \textbf{metric-wise evaluator} on test data (\textbf{500} queries). The \textbf{Normalized Average} in the last row shows the average per metric normalized by that metric's maximum score. Highest scores per metric are highlighted in blue, and the three metrics with the lowest normalized scores are highlighted in red.}
    \label{tab:plan_gen_result_with_lineage_metric_wise_eval}
\end{table*}

\begin{table*}[!ht]
    \centering
    \small
    
    
    \setlength{\tabcolsep}{3pt}

    \begin{tabularx}{\textwidth}{l|*{8}{Y}}
        \hline
        \thead{\textbf{LLM}} & \thead{\textbf{Overall} \\ {[0-100]}} & \thead{\textbf{Format} \\ {[0-20]}} & \thead{\textbf{Tool} \\ \textbf{Prompt} \\ \textbf{Align.} \\ {[0-20]}} & \thead{\textbf{Step} \\ \textbf{Exec.} \\ {[0-15]}} & \thead{\textbf{Query} \\ \textbf{Adhr.} \\ {[0-15]}} & \thead{\textbf{Depend.} \\ {[0-10]}} & \thead{\textbf{Redund.} \\ {[0-10]}} & \thead{\textbf{Tool} \\ \textbf{Usage} \\ \textbf{Compl.} \\ {[0-10]}}\\
        \hline
        claude-3-7-sonnet & \textcolor{blue}{\textbf{83.33}} & 17.36 & 15.48 & 11.46 & \textcolor{blue}{\textbf{13.37}} & 9.63 & 9.08 & 6.95 \\ 
        llama4-maverick-17b-instruct & 82.5 & 17.35 & 13.59 & 12.67 & 12.35 & 9.77 & 8.84 & 7.93 \\ 
        gpt-4o & 82.82 & 15.97 & 14.6 & 12.84 & 12.78 & 9.75 & 8.59 & 8.29 \\ 
        claude-sonnet-4 & 77.98 & 10.91 & \textcolor{blue}{\textbf{16.67}} & 11.78 & 12.91 & 9.7 & 8.7 & 7.32 \\ 
        llama3-3-70b-instruct & 81.11 & 15.25 & 14.59 & 13.32 & 12.34 & 9.59 & 9.08 & 6.95 \\ 
        nova-pro & 82.7 & 15.92 & 14.27 & \textcolor{blue}{\textbf{14.11}} & 12.69 & 9.78 & 8.86 & 7.07 \\ 
        nova-micro & 82.07 & \textcolor{blue}{\textbf{17.9}} & 13.25 & 13.44 & 12.13 & 9.5 & 8.65 & 7.2 \\ 
        gpt-4o-mini & 82.78 & 16.78 & 13.72 & 13.97 & 12.27 & 9.62 & 8.54 & 7.88 \\ 
        gpt-4.1-nano & 76.05 & 12.92 & 13.53 & 12.76 & 11.98 & 8.93 & 8.13 & 7.8 \\ 
        o3-mini & 74.06 & 12.28 & 14.89 & 10.56 & 13.09 & 9.06 & \textcolor{blue}{\textbf{9.42}} & 4.76 \\ 
        claude-3-5-haiku & 82.06 & 14.67 & 15.92 & 14.01 & 12.11 & \textcolor{blue}{\textbf{9.79}} & 9.14 & 6.43 \\ 
        nova-lite & 75.14 & 15.27 & 13.09 & 13.05 & 11.9 & 8.88 & 9.3 & 3.66 \\ 
        llama3-2-3b-instruct & 75.84 & 16.33 & 10.8 & 11.98 & 11.64 & 8.81 & 7.74 & 8.54 \\ 
        llama3-2-1b-instruct & 56.79 & 9.03 & 9.07 & 8.86 & 7.56 & 7.09 & 6.34 & \textcolor{blue}{\textbf{8.84}} \\ \hline
        \textbf{Average (Normalized)} & 78.23 & \textcolor{red}{\textbf{74.26}} & \textcolor{red}{\textbf{69.10}} & \textbf{83.24} & \textbf{80.53} & \textbf{92.78} & \textbf{86.01} & \textcolor{red}{\textbf{71.12}} \\ \hline
    \end{tabularx}
    \caption{Plan generation quality using prompts \textbf{without lineage}, evaluated with the \textbf{metric-wise evaluator} on test data (\textbf{500} queries). The \textbf{Normalized Average} in the last row shows the average per metric normalized by that metric's maximum score. Highest scores per metric are highlighted in blue, and the three metrics with the lowest normalized scores are highlighted in red.}
    \label{tab:plan_gen_result_without_lineage_metric_wise_eval}
\end{table*}

\begin{table*}[!ht]
    \centering
    \small

    \setlength{\tabcolsep}{12pt}

    \begin{tabularx}{\textwidth}{>{\raggedright\arraybackslash}l*{4}{|>{\centering\arraybackslash}X}}
        \toprule
        \multirow{2}{*}{\textbf{Tag}} 
        & \multicolumn{2}{c|}{\textbf{Pre-Loop}} 
        & \multicolumn{2}{c}{\textbf{Post-Loop}} \\
        \cmidrule(lr){2-3} \cmidrule(lr){4-5}
            & \makecell{\#} 
            & \makecell{\%} 
            & \makecell{\#} 
            & \makecell{\%} \\
        \midrule
        Extremely Bad   & 28  & 5.60   & 23  & 4.60  \\ 
        Very Bad        & 127 & 25.40  & 109 & 21.80 \\ 
        Bad             & 107 & 21.40  & 122 & 24.40 \\ 
        Acceptable      & 66  & 13.20  & 48  & 9.60  \\ 
        Good            & 104 & \cellcolor{green}{\textbf{20.80}}  & 84  & 16.80 \\ 
        Very Good       & 46  & 9.20   & 74  & \cellcolor{green}{\textbf{14.80}} \\ 
        Extremely Good  & 22  & 4.40   & 40  & \cellcolor{green}{\textbf{8.00}}  \\ 
        \midrule
        \textbf{Grand Total} & \textbf{500} & \textbf{100.00} & \textbf{500} & \textbf{100.00} \\
        \bottomrule
    \end{tabularx}
    \caption{Effectiveness of the \textbf{iterative Evaluator$\rightarrow$Optimizer loop}, measured using the one-shot evaluator on test data. Higher values between \textbf{Pre-Loop} and \textbf{Post-Loop} in the \textbf{Good, Very Good, Extremely Good} buckets are highlighted in green.}
    \label{tab:feedback_loop_performance_overall}
\end{table*}

\section{Results}
\subsection{Overall Plan Generation Quality}
\label{sec:overall-llm-quality}

\paragraph{One-shot evaluator (reference-based).}
Table~\ref{tab:plan_gen_result_overall_one_shot_eval} shows the proportion of plans in each quality tier: \emph{A+} (Extremely or Very Good), \emph{A} (\emph{A+} + Good), and \emph{B} (\emph{A} + Acceptable) for each model and prompt setting.
Overall performance is modest: the best result is \textbf{o3-mini} with \textbf{49.75\%} in \emph{A+} (and \textbf{65.99\%} in \emph{A}) without lineage, indicating that even strong models struggle to reliably produce near-gold plans; \textbf{GPT-4o} attains the highest \emph{B} coverage at \textbf{82.74\%}. Other models (including \emph{Claude-} and smaller \emph{Llama3} variants) trail substantially.

\emph{Effect of lineage in prompt.} Of 14 LLMs, 6 improve, 5 degrade, and 3 show no change on the \emph{A+} bucket when lineage is included, yielding no clear overall benefit.

\paragraph{Metric-wise evaluator (reference-informed).}
Tables~\ref{tab:plan_gen_result_with_lineage_metric_wise_eval}--\ref{tab:plan_gen_result_without_lineage_metric_wise_eval} report aggregate (0--100) and per-metric scores. 
\textbf{Claude-3-7-Sonnet} is highest overall at \textbf{84.8} (with lineage), followed by \textbf{llama4-maverick-17b-instruct} and \textbf{GPT-4o} (both $\approx\!82$ without lineage). 
Across models, \emph{Format} and \emph{Tool-Prompt Alignment} remain challenging (typical scores $\approx\!13$--14/20), and \emph{Tool Usage Completeness} is low, reflecting difficulty recognizing when both \textsc{T2S} and \textsc{RAG} are jointly warranted. 
On average, prompts \emph{without} lineage score higher overall; only 3/14 models improve with lineage, though \emph{Step Executability} improves for 9/14 models.

\paragraph{Sensitivity to metric weights.}
Our aggregate ``learned'' score is derived from metric-wise evaluator outputs and weights learned from human-preferred lineage triples (App.~\ref{app:weight-learning}). To test whether planner rankings are overly sensitive to these weights, we recompute scores under (i) an equal-weights scheme over normalized metrics and (ii) ten random weight vectors that preserve the 70{:}30 Effectiveness--Efficiency budget via Dirichlet draws. For both prompt settings (with-lineage and without-lineage), planner rankings remain highly stable: Spearman rank correlation \citep{zar2005spearman} between the learned and equal-weights rankings is $\rho=0.934$ (with-lineage) and $\rho=0.894$ (no-lineage), while random-weight rankings achieve median correlations of $\rho=0.890$ and $\rho=0.842$ respectively (Table~\ref{tab:weight-sensitivity-rho}; App. \ref{app:weight-sensitivity}).
This empirical analysis demonstrates that LLM rankings are stable across weight schemes, so our main conclusions remain unchanged.

\paragraph{Key grouped takeaways (details in Appendix~\ref{app:grouped-analyses}).}
Across grouped analyses, \textbf{Simple} queries and shorter plans ([1,4] steps) outperform \textbf{Compound} queries and longer plans; one-shot results slightly favor Objective queries, while metric-wise scores sometimes tilt toward Subjective ones, and no stable pattern emerges w.r.t.\ hop count.

\paragraph{Correlation with end-to-end QA quality.}
Although our benchmark is non-executing by design, we conducted a small end-to-end study on 200 LLM-generated test queries, comparing the no-plan baseline (\textbf{R1}), a north-star system that executes \emph{human-annotated reference plans} (\textbf{R2}), and the same stack driven by \emph{LLM-generated plans} (\textbf{R3}). Final answers are scored by an in-house Judge LLM using a four-metric rubric (Validity, Consistency, Completeness, Redundancy; App.~\ref{app:end2end}). R2 achieves a win rate of \textbf{58.7\%}, versus \textbf{42.75\%} for R1 and \textbf{33.33\%} for R3, and reference plans also obtain substantially higher planning scores than LLM-generated plans, supporting a positive correlation between planner quality and end-to-end QA quality.

\subsection{Effectiveness of the Iterative Evaluator\texorpdfstring{$\rightarrow$}{->}Optimizer Loop}
\label{sec:loop-effectiveness}

We quantify the impact of our Iterative Evaluator\texorpdfstring{$\rightarrow$}{->}Optimizer loop using the one-shot evaluator, comparing \emph{initial} Nova-Lite plans (\emph{pre-loop}) to \emph{final} plans (\emph{post-loop}) in Table~\ref{tab:feedback_loop_performance_overall}, which buckets plans into seven quality tiers.

\paragraph{Overall.}
Post-loop plans improve the top buckets: \textbf{Extremely Good} rises from \textbf{4.4\%} to \textbf{8.0\%} and \textbf{Very Good} from \textbf{9.2\%} to \textbf{14.8\%}, with a net shift toward higher-quality brackets despite a small drop in \emph{Good}. Gains are consistent across Objective/Subjective, Simple/Compound, and hop-count groupings (App.~\ref{app:loop-grouped}), with the largest relative improvements for \textbf{Objective}, \textbf{Simple}, and \textbf{3+ hop} plans.

\subsection{Module Validation on Validation Set}
\label{sec:module-validation}

We validate the four core evaluators/modules on the held-out validation split (Sec.~\ref{app:dataset-curation-steps}).

\paragraph{Metric-wise evaluator.}
We compare a \emph{Single} (all 7 metrics in one prompt) vs.\ \emph{Deconstructed} (one prompt per metric) setup, each in \emph{reference-free} and \emph{reference-based} modes. For each query, we form triplets of the three highest-quality plans in its lineage and assess \emph{relaxed triplet ranking agreement} with human inequalities per metric (Table~\ref{tab:validation_result_metric_wise_eval}). The deconstructed, reference-based setting performs best across all seven metrics, with $>90\%$ agreement for \textsc{Dependency}, \textsc{Format}, \textsc{Tool Usage Completeness}, and $>80\%$ for \textsc{Query Adherence}, \textsc{Redundancy}, \textsc{Tool-Prompt Alignment}; \textsc{Step Executability} attains $79.43\%$.

\paragraph{One-shot overall evaluator.}
We report macro Precision/Recall/F1 for label agreement over seven quality tags (\textit{Extremely Bad}~\texorpdfstring{$\rightarrow$}{->}~\textit{Extremely Good}) between the LLM judge and humans (Table~\ref{tab:validation_result_one_shot_eval}). Macro scores are $0.92/0.93/0.92$, with every tag at F1~$\ge 0.85$.

\paragraph{Step-wise evaluator.}
On $N{=}400$ step instances (80 queries $\times$ $\approx$5 steps/plan), multi-label tag agreement (Precision/Recall/F1) appears in Table~\ref{tab:validation_result_step_wise_eval}. Macro F1 is $0.84$; \textsc{Incorrect Prompt} and \textsc{Repeated Detail} reach $0.91$ F1, while \textsc{No Change} is hardest at $0.75$.

\paragraph{Plan optimizer.}
Using the tuned one-shot judge to compare optimizer revisions against human gold for $160$ pairs (Table~\ref{tab:validation_result_plan_optimizer}), $74.5\%$ of optimizer outputs land in the top buckets (\textit{Extremely Good} $28.13\%$, \textit{Very Good} $24.38\%$, \textit{Good} $21.88\%$), and $10.63\%$ are \textit{Acceptable}.

\paragraph{Judge robustness across models.}
Additionally, to mitigate the risk of single-judge bias, we replicated the validation and a small-scale test analysis with an alternative judge (GPT-5 \citep{OpenAI2025GPT5}) in the reference-based, deconstructed configuration. On the validation set, both Sonnet-4 and GPT-5 achieve high triplet-ranking agreement with humans for the metric-wise evaluator (e.g., $>90\%$ for \textsc{Dependency}, \textsc{Format}, \textsc{Redundancy}, \textsc{Tool Usage Completeness}; Tab.~\ref{tab:validation_result_metric_wise_eval}), and Sonnet-4 attains slightly higher macro F1 as a one-shot judge than GPT-5 ($0.921$ vs.\ $0.882$; Tabs.~\ref{tab:validation_result_one_shot_eval}--\ref{tab:oneshot-validation-gpt5}). On a test subset of 50 queries per planner (700 query–planner pairs across 14 planners), planner rankings under Sonnet-4 and GPT-5 are strongly correlated for the overall metric-wise score ($\rho{=}0.60$) and for most individual metrics (e.g., \textsc{Dependency} $\rho{=}0.84$, \textsc{Query Adherence} $\rho{=}0.79$; Tab.~\ref{tab:judge-corr}), confirming that our conclusions are robust to the choice of judge model.

\paragraph{Takeaway.}
The deconstructed, reference-based metric suite is reliable; the one-shot judge shows strong label fidelity; the step-wise evaluator identifies actionable errors; and the plan optimizer substantially improves plan quality relative to initial drafts. Additional experiments with GPT-5 as an alternative judge (Appx.~\ref{app:judge-robustness}) show similar human alignment and strongly correlated planner rankings, indicating that our conclusions are robust to the choice of judge model.

\section{Conclusion}
We present a domain-grounded framework for \emph{tool-aware plan generation} in contact-centers that unifies (i) a formal, executable plan schema with explicit \texttt{depends\_on} for parallelism, (ii) a tool interface spanning structured (T2S/Snowflake) and unstructured (RAG/transcripts) evidence plus LLM synthesis, (iii) an iterative evaluator\texorpdfstring{$\rightarrow$}{->}optimizer loop that yields \textbf{plan lineage}, and (iv) a two-track evaluation methodology (metric-wise and one-shot, reference-based). Beyond a benchmark, this constitutes a practical recipe for designing, critiquing, and improving planning agents in this domain.

Empirically, a 14-LLM study shows that - even with careful prompting - models struggle with query adherence and tool-usage completeness; simpler queries and shorter plans remain markedly easier. Practically, our framework supports higher-fidelity, auditable analytics by exposing where tool choices, prompts, or dependencies fail, and by encoding safe parallel execution. Importantly, \textbf{plan lineage} is not only interpretable evidence for debugging but also a training signal: it can supervise or reward better planners via SFT or RLVR. Finally, our current offline evaluator\texorpdfstring{$\rightarrow$}{->}optimizer loop naturally sets the stage for \textbf{online, tool-aware replanning}: adding a Step Executor to form an executor\texorpdfstring{$\rightarrow$}{->}evaluator\texorpdfstring{$\rightarrow$}{->}optimizer triad enables real-time plan updates conditioned on actual tool outputs.

\section{Limitations}

\paragraph{Domain scope.}
Our study is focused on contact-center analytics with domain-specific tools and schemas; we do not measure transfer to other enterprise verticals or domains.

\paragraph{Proprietary vs.\ public data.}
The main 600-query benchmark used for the core results is based on production-style queries and remains proprietary due to contractual restrictions. We release a separate, anonymized 200-query public dataset built from LLM-generated queries on a synthetic test account.

\paragraph{Tool palette and interfaces.}
We fix the tool set (T2S, RAG, LLM) and a specific JSON-based plan schema. While this reflects a common production stack, it does not cover richer tool ecosystems (e.g., BI connectors, policy engines) or alternative planning formalisms, which we leave for future work.

\paragraph{Offline, non-executing loop.}
The evaluator$\rightarrow$optimizer loop operates purely on plan text without executing tools; runtime issues such as bad SQL, empty joins, or latency-induced race conditions are not directly measured. Our small end-to-end correlation study provides supporting evidence that higher plan scores correlate with better answers, but a full executor$\rightarrow$evaluator$\rightarrow$optimizer loop is out of scope.

\paragraph{LLM-based judging.}
Both the metric-wise evaluator and one-shot judge are LLM-driven. We mitigate bias via structured rubrics, human-verified reference plans, inter-annotator and Human–LLM agreement checks, and a robustness study with GPT-5 as an alternative judge. Nonetheless, some residual judge bias and prompt sensitivity may remain, especially on edge cases.

\paragraph{Cost/latency modeling.}
Although plans encode explicit dependencies and potential parallelism, we do not optimize for cost or latency during planning or scheduling. Execution-time trade-offs are therefore only indirectly reflected in our metrics.

\paragraph{Future directions.}
(i) \textbf{Public proxy tasks:} extend the released synthetic dataset into broader, domain-agnostic plan-generation suites so that others can benchmark planners and judges in the open. 
(ii) \textbf{Learning from lineage:} use lineage traces for SFT and RLVR to teach planners revision policies, robust tool selection, and dependency wiring. 
(iii) \textbf{Cost-aware planning:} integrate resource/budget objectives and dynamic branching for latency- and cost-efficient execution. 
(iv) \textbf{Neuro-symbolic hybrids:} combine LLM planning with classical verification/optimization to enforce hard constraints on dependencies and formats. 
(v) \textbf{Online replanning:} instantiate an executor$\rightarrow$evaluator$\rightarrow$optimizer loop that updates plans at runtime based on tool outputs, with safeguards for rollback and partial recomputation, and study its interaction with our offline metrics.

\bibliography{custom}

@misc{callnavi,
      title={CallNavi, A Challenge and Empirical Study on LLM Function Calling and Routing}, 
      author={Yewei Song and Xunzhu Tang and Cedric Lothritz and Saad Ezzini and Jacques Klein and Tegawendé F. Bissyandé and Andrey Boytsov and Ulrick Ble and Anne Goujon},
      year={2025},
      eprint={2501.05255},
      archivePrefix={arXiv},
      primaryClass={cs.SE},
      url={https://arxiv.org/abs/2501.05255}, 
}

@misc{wang2025fdabench,
      title={FDABench: A Benchmark for Data Agents on Analytical Queries over Heterogeneous Data}, 
      author={Ziting Wang and Shize Zhang and Haitao Yuan and Jinwei Zhu and Shifu Li and Wei Dong and Gao Cong},
      year={2025},
      eprint={2509.02473},
      archivePrefix={arXiv},
      primaryClass={cs.DB},
      url={https://arxiv.org/abs/2509.02473}, 
}

@inproceedings{liu-suql,
    title = "{SUQL}: Conversational Search over Structured and Unstructured Data with Large Language Models",
    author = "Liu, Shicheng  and
      Xu, Jialiang  and
      Tjangnaka, Wesley  and
      Semnani, Sina  and
      Yu, Chen  and
      Lam, Monica",
    editor = "Duh, Kevin  and
      Gomez, Helena  and
      Bethard, Steven",
    booktitle = "Findings of the Association for Computational Linguistics: NAACL 2024",
    month = jun,
    year = "2024",
    address = "Mexico City, Mexico",
    publisher = "Association for Computational Linguistics",
    url = "https://aclanthology.org/2024.findings-naacl.283/",
    doi = "10.18653/v1/2024.findings-naacl.283",
    pages = "4535--4555"
}

@inproceedings{apibank,
    title = "{API}-Bank: A Comprehensive Benchmark for Tool-Augmented {LLM}s",
    author = "Li, Minghao  and
      Zhao, Yingxiu  and
      Yu, Bowen  and
      Song, Feifan  and
      Li, Hangyu  and
      Yu, Haiyang  and
      Li, Zhoujun  and
      Huang, Fei  and
      Li, Yongbin",
    editor = "Bouamor, Houda  and
      Pino, Juan  and
      Bali, Kalika",
    booktitle = "Proceedings of the 2023 Conference on Empirical Methods in Natural Language Processing",
    month = dec,
    year = "2023",
    address = "Singapore",
    publisher = "Association for Computational Linguistics",
    url = "https://aclanthology.org/2023.emnlp-main.187/",
    doi = "10.18653/v1/2023.emnlp-main.187",
    pages = "3102--3116"
}

@misc{agentbench,
      title={AgentBench: Evaluating LLMs as Agents}, 
      author={Xiao Liu and Hao Yu and Hanchen Zhang and Yifan Xu and Xuanyu Lei and Hanyu Lai and Yu Gu and Hangliang Ding and Kaiwen Men and Kejuan Yang and Shudan Zhang and Xiang Deng and Aohan Zeng and Zhengxiao Du and Chenhui Zhang and Sheng Shen and Tianjun Zhang and Yu Su and Huan Sun and Minlie Huang and Yuxiao Dong and Jie Tang},
      year={2023},
      eprint={2308.03688},
      archivePrefix={arXiv},
      primaryClass={cs.AI},
      url={https://arxiv.org/abs/2308.03688}, 
}

@misc{hastemakeswaste,
      title={Haste Makes Waste: Evaluating Planning Abilities of LLMs for Efficient and Feasible Multitasking with Time Constraints Between Actions}, 
      author={Zirui Wu and Xiao Liu and Jiayi Li and Lingpeng Kong and Yansong Feng},
      year={2025},
      eprint={2503.02238},
      archivePrefix={arXiv},
      primaryClass={cs.CL},
      url={https://arxiv.org/abs/2503.02238}, 
}

@misc{lip_llm,
      title={LiP-LLM: Integrating Linear Programming and dependency graph with Large Language Models for multi-robot task planning}, 
      author={Kazuma Obata and Tatsuya Aoki and Takato Horii and Tadahiro Taniguchi and Takayuki Nagai},
      year={2024},
      eprint={2410.21040},
      archivePrefix={arXiv},
      primaryClass={cs.RO},
      url={https://arxiv.org/abs/2410.21040}, 
}

@misc{OpenAI2025GPT5,
  title = {{GPT-5 System Card}},
  author = {{OpenAI}},
  howpublished = {\url{https://openai.com/index/gpt-5-system-card/}},
  year = {2025},
  month = {Aug},
  note = {{Accessed: 2 January 2026}}
}

@misc{mmsbenchmark,
      title={m\&m's: A Benchmark to Evaluate Tool-Use for multi-step multi-modal Tasks}, 
      author={Zixian Ma and Weikai Huang and Jieyu Zhang and Tanmay Gupta and Ranjay Krishna},
      year={2024},
      eprint={2403.11085},
      archivePrefix={arXiv},
      primaryClass={cs.CV},
      url={https://arxiv.org/abs/2403.11085}, 
}

@misc{mind2web,
      title={Mind2Web: Towards a Generalist Agent for the Web}, 
      author={Xiang Deng and Yu Gu and Boyuan Zheng and Shijie Chen and Samuel Stevens and Boshi Wang and Huan Sun and Yu Su},
      year={2023},
      eprint={2306.06070},
      archivePrefix={arXiv},
      primaryClass={cs.CL},
      url={https://arxiv.org/abs/2306.06070}, 
}

@misc{planbench,
      title={PlanBench: An Extensible Benchmark for Evaluating Large Language Models on Planning and Reasoning about Change}, 
      author={Karthik Valmeekam and Matthew Marquez and Alberto Olmo and Sarath Sreedharan and Subbarao Kambhampati},
      year={2023},
      eprint={2206.10498},
      archivePrefix={arXiv},
      primaryClass={cs.CL},
      url={https://arxiv.org/abs/2206.10498}, 
}

@misc{wang2024mintevaluatingllmsmultiturn,
      title={MINT: Evaluating LLMs in Multi-turn Interaction with Tools and Language Feedback}, 
      author={Xingyao Wang and Zihan Wang and Jiateng Liu and Yangyi Chen and Lifan Yuan and Hao Peng and Heng Ji},
      year={2024},
      eprint={2309.10691},
      archivePrefix={arXiv},
      primaryClass={cs.CL},
      url={https://arxiv.org/abs/2309.10691}, 
}

@misc{nl2flow,
      title={Scaling LLM Planning: NL2FLOW for Parametric Problem Generation and Rigorous Evaluation}, 
      author={Jungkoo Kang},
      year={2025},
      eprint={2507.02253},
      archivePrefix={arXiv},
      primaryClass={cs.AI},
      url={https://arxiv.org/abs/2507.02253}, 
}

@misc{llm_planning_survey,
      title={Large Language Models for Planning: A Comprehensive and Systematic Survey}, 
      author={Pengfei Cao and Tianyi Men and Wencan Liu and Jingwen Zhang and Xuzhao Li and Xixun Lin and Dianbo Sui and Yanan Cao and Kang Liu and Jun Zhao},
      year={2025},
      eprint={2505.19683},
      archivePrefix={arXiv},
      primaryClass={cs.AI},
      url={https://arxiv.org/abs/2505.19683}, 
}

@misc{llm_reasoning_models_replace_classical_planning,
      title={Can LLM-Reasoning Models Replace Classical Planning? A Benchmark Study}, 
      author={Kai Goebel and Patrik Zips},
      year={2025},
      eprint={2507.23589},
      archivePrefix={arXiv},
      primaryClass={cs.RO},
      url={https://arxiv.org/abs/2507.23589}, 
}

@misc{llm_planning_modelers_survey,
      title={LLMs as Planning Modelers: A Survey for Leveraging Large Language Models to Construct Automated Planning Models}, 
      author={Marcus Tantakoun and Xiaodan Zhu and Christian Muise},
      year={2025},
      eprint={2503.18971},
      archivePrefix={arXiv},
      primaryClass={cs.AI},
      url={https://arxiv.org/abs/2503.18971}, 
}

@misc{catp_llm,
      title={CATP-LLM: Empowering Large Language Models for Cost-Aware Tool Planning}, 
      author={Duo Wu and Jinghe Wang and Yuan Meng and Yanning Zhang and Le Sun and Zhi Wang},
      year={2025},
      eprint={2411.16313},
      archivePrefix={arXiv},
      primaryClass={cs.AI},
      url={https://arxiv.org/abs/2411.16313}, 
}

@article{llm_auto_planning_scheduling,
   title={On the Prospects of Incorporating Large Language Models (LLMs) in Automated Planning and Scheduling (APS)},
   volume={34},
   ISSN={2334-0835},
   url={http://dx.doi.org/10.1609/icaps.v34i1.31503},
   DOI={10.1609/icaps.v34i1.31503},
   journal={Proceedings of the International Conference on Automated Planning and Scheduling},
   publisher={Association for the Advancement of Artificial Intelligence (AAAI)},
   author={Pallagani, Vishal and Muppasani, Bharath Chandra and Roy, Kaushik and Fabiano, Francesco and Loreggia, Andrea and Murugesan, Keerthiram and Srivastava, Biplav and Rossi, Francesca and Horesh, Lior and Sheth, Amit},
   year={2024},
   month=may, pages={432–444} }

@misc{prompt2dag,
      title={Prompt2DAG: A Modular Methodology for LLM-Based Data Enrichment Pipeline Generation}, 
      author={Abubakari Alidu and Michele Ciavotta and Flavio DePaoli},
      year={2025},
      eprint={2509.13487},
      archivePrefix={arXiv},
      primaryClass={cs.SE},
      url={https://arxiv.org/abs/2509.13487}, 
}

@misc{openai2024gpt4ocard,
      title={GPT-4o System Card}, 
      author={OpenAI and : and Aaron Hurst and Adam Lerer and Adam P. Goucher and Adam Perelman and Aditya Ramesh and Aidan Clark and AJ Ostrow and Akila Welihinda and Alan Hayes and Alec Radford and Aleksander Mądry and Alex Baker-Whitcomb and Alex Beutel and Alex Borzunov and Alex Carney and Alex Chow and Alex Kirillov and Alex Nichol and Alex Paino and Alex Renzin and Alex Tachard Passos and Alexander Kirillov and Alexi Christakis and Alexis Conneau and Ali Kamali and Allan Jabri and Allison Moyer and Allison Tam and Amadou Crookes and Amin Tootoochian and Amin Tootoonchian and Ananya Kumar and Andrea Vallone and Andrej Karpathy and Andrew Braunstein and Andrew Cann and Andrew Codispoti and Andrew Galu and Andrew Kondrich and Andrew Tulloch and Andrey Mishchenko and Angela Baek and Angela Jiang and Antoine Pelisse and Antonia Woodford and Anuj Gosalia and Arka Dhar and Ashley Pantuliano and Avi Nayak and Avital Oliver and Barret Zoph and Behrooz Ghorbani and Ben Leimberger and Ben Rossen and Ben Sokolowsky and Ben Wang and Benjamin Zweig and Beth Hoover and Blake Samic and Bob McGrew and Bobby Spero and Bogo Giertler and Bowen Cheng and Brad Lightcap and Brandon Walkin and Brendan Quinn and Brian Guarraci and Brian Hsu and Bright Kellogg and Brydon Eastman and Camillo Lugaresi and Carroll Wainwright and Cary Bassin and Cary Hudson and Casey Chu and Chad Nelson and Chak Li and Chan Jun Shern and Channing Conger and Charlotte Barette and Chelsea Voss and Chen Ding and Cheng Lu and Chong Zhang and Chris Beaumont and Chris Hallacy and Chris Koch and Christian Gibson and Christina Kim and Christine Choi and Christine McLeavey and Christopher Hesse and Claudia Fischer and Clemens Winter and Coley Czarnecki and Colin Jarvis and Colin Wei and Constantin Koumouzelis and Dane Sherburn and Daniel Kappler and Daniel Levin and Daniel Levy and David Carr and David Farhi and David Mely and David Robinson and David Sasaki and Denny Jin and Dev Valladares and Dimitris Tsipras and Doug Li and Duc Phong Nguyen and Duncan Findlay and Edede Oiwoh and Edmund Wong and Ehsan Asdar and Elizabeth Proehl and Elizabeth Yang and Eric Antonow and Eric Kramer and Eric Peterson and Eric Sigler and Eric Wallace and Eugene Brevdo and Evan Mays and Farzad Khorasani and Felipe Petroski Such and Filippo Raso and Francis Zhang and Fred von Lohmann and Freddie Sulit and Gabriel Goh and Gene Oden and Geoff Salmon and Giulio Starace and Greg Brockman and Hadi Salman and Haiming Bao and Haitang Hu and Hannah Wong and Haoyu Wang and Heather Schmidt and Heather Whitney and Heewoo Jun and Hendrik Kirchner and Henrique Ponde de Oliveira Pinto and Hongyu Ren and Huiwen Chang and Hyung Won Chung and Ian Kivlichan and Ian O'Connell and Ian O'Connell and Ian Osband and Ian Silber and Ian Sohl and Ibrahim Okuyucu and Ikai Lan and Ilya Kostrikov and Ilya Sutskever and Ingmar Kanitscheider and Ishaan Gulrajani and Jacob Coxon and Jacob Menick and Jakub Pachocki and James Aung and James Betker and James Crooks and James Lennon and Jamie Kiros and Jan Leike and Jane Park and Jason Kwon and Jason Phang and Jason Teplitz and Jason Wei and Jason Wolfe and Jay Chen and Jeff Harris and Jenia Varavva and Jessica Gan Lee and Jessica Shieh and Ji Lin and Jiahui Yu and Jiayi Weng and Jie Tang and Jieqi Yu and Joanne Jang and Joaquin Quinonero Candela and Joe Beutler and Joe Landers and Joel Parish and Johannes Heidecke and John Schulman and Jonathan Lachman and Jonathan McKay and Jonathan Uesato and Jonathan Ward and Jong Wook Kim and Joost Huizinga and Jordan Sitkin and Jos Kraaijeveld and Josh Gross and Josh Kaplan and Josh Snyder and Joshua Achiam and Joy Jiao and Joyce Lee and Juntang Zhuang and Justyn Harriman and Kai Fricke and Kai Hayashi and Karan Singhal and Katy Shi and Kavin Karthik and Kayla Wood and Kendra Rimbach and Kenny Hsu and Kenny Nguyen and Keren Gu-Lemberg and Kevin Button and Kevin Liu and Kiel Howe and Krithika Muthukumar and Kyle Luther and Lama Ahmad and Larry Kai and Lauren Itow and Lauren Workman and Leher Pathak and Leo Chen and Li Jing and Lia Guy and Liam Fedus and Liang Zhou and Lien Mamitsuka and Lilian Weng and Lindsay McCallum and Lindsey Held and Long Ouyang and Louis Feuvrier and Lu Zhang and Lukas Kondraciuk and Lukasz Kaiser and Luke Hewitt and Luke Metz and Lyric Doshi and Mada Aflak and Maddie Simens and Madelaine Boyd and Madeleine Thompson and Marat Dukhan and Mark Chen and Mark Gray and Mark Hudnall and Marvin Zhang and Marwan Aljubeh and Mateusz Litwin and Matthew Zeng and Max Johnson and Maya Shetty and Mayank Gupta and Meghan Shah and Mehmet Yatbaz and Meng Jia Yang and Mengchao Zhong and Mia Glaese and Mianna Chen and Michael Janner and Michael Lampe and Michael Petrov and Michael Wu and Michele Wang and Michelle Fradin and Michelle Pokrass and Miguel Castro and Miguel Oom Temudo de Castro and Mikhail Pavlov and Miles Brundage and Miles Wang and Minal Khan and Mira Murati and Mo Bavarian and Molly Lin and Murat Yesildal and Nacho Soto and Natalia Gimelshein and Natalie Cone and Natalie Staudacher and Natalie Summers and Natan LaFontaine and Neil Chowdhury and Nick Ryder and Nick Stathas and Nick Turley and Nik Tezak and Niko Felix and Nithanth Kudige and Nitish Keskar and Noah Deutsch and Noel Bundick and Nora Puckett and Ofir Nachum and Ola Okelola and Oleg Boiko and Oleg Murk and Oliver Jaffe and Olivia Watkins and Olivier Godement and Owen Campbell-Moore and Patrick Chao and Paul McMillan and Pavel Belov and Peng Su and Peter Bak and Peter Bakkum and Peter Deng and Peter Dolan and Peter Hoeschele and Peter Welinder and Phil Tillet and Philip Pronin and Philippe Tillet and Prafulla Dhariwal and Qiming Yuan and Rachel Dias and Rachel Lim and Rahul Arora and Rajan Troll and Randall Lin and Rapha Gontijo Lopes and Raul Puri and Reah Miyara and Reimar Leike and Renaud Gaubert and Reza Zamani and Ricky Wang and Rob Donnelly and Rob Honsby and Rocky Smith and Rohan Sahai and Rohit Ramchandani and Romain Huet and Rory Carmichael and Rowan Zellers and Roy Chen and Ruby Chen and Ruslan Nigmatullin and Ryan Cheu and Saachi Jain and Sam Altman and Sam Schoenholz and Sam Toizer and Samuel Miserendino and Sandhini Agarwal and Sara Culver and Scott Ethersmith and Scott Gray and Sean Grove and Sean Metzger and Shamez Hermani and Shantanu Jain and Shengjia Zhao and Sherwin Wu and Shino Jomoto and Shirong Wu and Shuaiqi and Xia and Sonia Phene and Spencer Papay and Srinivas Narayanan and Steve Coffey and Steve Lee and Stewart Hall and Suchir Balaji and Tal Broda and Tal Stramer and Tao Xu and Tarun Gogineni and Taya Christianson and Ted Sanders and Tejal Patwardhan and Thomas Cunninghman and Thomas Degry and Thomas Dimson and Thomas Raoux and Thomas Shadwell and Tianhao Zheng and Todd Underwood and Todor Markov and Toki Sherbakov and Tom Rubin and Tom Stasi and Tomer Kaftan and Tristan Heywood and Troy Peterson and Tyce Walters and Tyna Eloundou and Valerie Qi and Veit Moeller and Vinnie Monaco and Vishal Kuo and Vlad Fomenko and Wayne Chang and Weiyi Zheng and Wenda Zhou and Wesam Manassra and Will Sheu and Wojciech Zaremba and Yash Patil and Yilei Qian and Yongjik Kim and Youlong Cheng and Yu Zhang and Yuchen He and Yuchen Zhang and Yujia Jin and Yunxing Dai and Yury Malkov},
      year={2024},
      eprint={2410.21276},
      archivePrefix={arXiv},
      primaryClass={cs.CL},
      url={https://arxiv.org/abs/2410.21276}, 
}

@misc{openai-o3-mini-2025,
  author = {{OpenAI}},
  title = {{OpenAI o3-mini}},
  howpublished = {\url{https://openai.com/index/openai-o3-mini/}},
  year = {2025},
  note = {Announced January 31, 2025}
}

@misc{gpt4_1,
  author = {OpenAI},
  title = {GPT-4.1-mini},
  howpublished = {API Model},
  year = {2025},
  month = {April},
  note = {Available from {OpenAI API} https://openai.com/index/gpt-4-1/},
  url = {https://openai.com/index/gpt-4-1/}
}

@misc{claude-3_5,
  author = {{Anthropic}},
  title = {{Claude 3.5 Haiku}},
  howpublished = {\url{https://www.anthropic.com/news/claude-3-5}},
  year = {2024},
  note = {Accessed: [Date of access]},
}

@misc{claude-4,
  title = {Claude-Sonnet-4},
  author = {{Anthropic}},
  year = {2025},
  url = {https://www.anthropic.com/news/claude-4},
  note = {{Large language model, accessed [date]}},
}

@misc{llama4maverick,
    author = {{Meta}},
    title = {Llama 4 Maverick},
    howpublished = {\url{https://ai.meta.com/blog/llama-4-multimodal-intelligence/}},
    year = {2025},
    note = {Model released as a part of the Llama 4 family.}
}

@misc{llama3_meta_research,
  title = {Llama 3: An Open-Source Framework for Language Modeling},
  author = {Meta AI Research},
  year = {2024},
  howpublished = {Published via blog posts and official releases, as no specific public paper yet. Refer to the official announcements.},
  note = {The Llama 3.2 models, including the 1B and 3B instruct variants, are a subsequent release built upon the Llama 3 framework. The primary source for the 3.2 models is the AI at Meta blog post on Sept 25, 2024.}
}

@Article{nova_models,
 author = {Amazon Artificial General Intelligence},
 title = {The Amazon Nova family of models: Technical report and model card},
 year = {2024},
 url = {https://www.amazon.science/publications/the-amazon-nova-family-of-models-technical-report-and-model-card},
 journal = {Amazon Technical Reports},
}

@article{cohen_kappa,
  title={A Coefficient of Agreement for Nominal Scales},
  author={Jacob Cohen},
  journal={Educational and Psychological Measurement},
  year={1960},
  pages={37–46},
  volume={20(1)}
}

@article{cohen_weighted_kappa,
  title={Weighted kappa: Nominal scale agreement with provision for scaled disagreement or partial credit},
  author={Cohen, Jacob},
  journal={Psychological bulletin},
  volume={70},
  number={4},
  pages={213},
  year={1968},
  publisher={American Psychological Association},
  doi={10.1037/h0026256}
}

@article{fleiss1971,
  title={Measuring nominal scale agreement among many raters},
  author={Fleiss, Joseph L},
  journal={Psychological Bulletin},
  year={1971}
}

@article{landiskoch1977,
  title={The measurement of observer agreement for categorical data},
  author={Landis, J Richard and Koch, Gary G},
  journal={Biometrics},
  year={1977}
}

@misc{insight_bench,
      title={InsightBench: Evaluating Business Analytics Agents Through Multi-Step Insight Generation}, 
      author={Gaurav Sahu and Abhay Puri and Juan Rodriguez and Amirhossein Abaskohi and Mohammad Chegini and Alexandre Drouin and Perouz Taslakian and Valentina Zantedeschi and Alexandre Lacoste and David Vazquez and Nicolas Chapados and Christopher Pal and Sai Rajeswar Mudumba and Issam Hadj Laradji},
      year={2025},
      eprint={2407.06423},
      archivePrefix={arXiv},
      primaryClass={cs.AI},
      url={https://arxiv.org/abs/2407.06423}, 
}

@article{zar2005spearman,
  title={Spearman rank correlation},
  author={Zar, Jerrold H},
  journal={Encyclopedia of Biostatistics},
  volume={7},
  year={2005},
  publisher={Wiley Online Library}
}

\appendix

\begin{table*}[!ht]
    \centering
    \small

    \setlength{\tabcolsep}{6pt}

    \begin{tabularx}{\textwidth}{l|ccc|ccc}
        \toprule
        \multirow{2}{*}{\textbf{LLM}} 
        & \multicolumn{3}{c|}{\textbf{Objective Queries}} 
        & \multicolumn{3}{c}{\textbf{Subjective Queries}} \\
        \cmidrule(lr){2-4} \cmidrule(lr){5-7}
            & \makecell{(\emph{A+})\\Extr. good,\\ very good \\ (\%)} 
            & \makecell{(\emph{A})\\Extr. good,\\ very good,\\ good \\ (\%)} 
            & \makecell{(\emph{B})\\Extr. good,\\ very good,\\ good,\\ acceptable \\ (\%)} 
            & \makecell{(\emph{A+})\\Extr. good,\\ very good \\ (\%)} 
            & \makecell{(\emph{A})\\Extr. good,\\ very good,\\ good \\ (\%)} 
            & \makecell{(\emph{B})\\Extr. good,\\ very good,\\ good,\\ acceptable \\ (\%)} \\
        \midrule
        gpt-4o & \cellcolor{green}\textcolor{blue}{\textbf{50.53}} & 62.11 & 74.74 & \cellcolor{green}{\textbf{41.18}} & 62.75 & 87.25 \\ 
        o3-mini & \textcolor{blue}{\textbf{47.37}} & 55.79 & 72.63 & 39.22 & 50.98 & 72.55 \\ 
        claude-3-7-sonnet & \textcolor{blue}{\textbf{33.68}} & 46.32 & 62.11 & 27.45 & 50.00 & 78.43 \\ 
        claude-sonnet-4 & \textcolor{blue}{\textbf{32.63}} & 51.58 & 68.42 & 28.43 & 53.92 & 78.43 \\ 
        gpt-4o-mini & 31.58 & 51.58 & 62.11 & \textcolor{magenta}{\textbf{32.35}} & 46.08 & 66.67 \\ 
        claude-3-5-haiku & \textcolor{blue}{\textbf{27.37}} & 44.21 & 56.84 & 22.55 & 47.06 & 67.65 \\ 
        nova-pro & \textcolor{blue}{\textbf{20.00}} & 36.84 & 55.79 & 16.67 & 48.04 & 67.65 \\ 
        llama3-70b-instruct  & \textcolor{blue}{\textbf{18.95}} &	37.89 &	53.68 &  18.63 &	50.00 &	65.69 \\ 
        llama4-maverick-17b-instruct & 18.95 & 32.63 & 48.42 & \textcolor{magenta}{\textbf{21.57}} & 40.20 & 59.80 \\ 
        nova-micro & 17.89 & 30.53 & 49.47 & \textcolor{magenta}{\textbf{23.53}} & 40.20 & 56.86 \\ 
        gpt-4.1-nano & 12.63 & 16.84 & 37.89 & \textcolor{magenta}{\textbf{21.57}} & 39.22 & 71.57 \\ 
        nova-lite & 9.47  & 18.95 & 33.68 & \textcolor{magenta}{\textbf{17.65}} & 42.16 & 63.73 \\ 
        llama3-2-3b-instruct & 1.05  & 3.16  & 9.47  &  \textcolor{magenta}{\textbf{8.82}} & 17.65 & 24.51 \\ 
        llama3-2-1b-instruct & 0.00  & 0.00  & 0.00  &  0.00 & 0.00 & 0.00 \\ 
        \midrule
        \textbf{Grand Total (\#)} & \multicolumn{3}{c|}{\textbf{241}} & \multicolumn{3}{c}{\textbf{259}} \\
        \bottomrule
    \end{tabularx}

    \caption{Plan Generation quality comparison using prompts \textbf{with lineage} for \textbf{objective} and \textbf{subjective} queries, judged using the \textbf{one-shot evaluator} on test data. The highest score in the \textbf{Extremely Good, Very Good} bucket is highlighted in green for both query categories; blue indicates better performance on Objective queries, and magenta indicates better performance on Subjective queries.}
    \label{tab:plan_gen_result_objective_vs_subjective_with_lineage_one_shot_eval}
\end{table*}

\begin{table*}[!ht]
    \centering
    \small

    \setlength{\tabcolsep}{6pt}

    \begin{tabularx}{\textwidth}{l|ccc|ccc}
        \toprule
        \multirow{2}{*}{\textbf{LLM}} 
        & \multicolumn{3}{c|}{\textbf{Objective Queries}} 
        & \multicolumn{3}{c}{\textbf{Subjective Queries}} \\
        \cmidrule(lr){2-4} \cmidrule(lr){5-7}
            & \makecell{(\emph{A+})\\Extr. good,\\ very good \\ (\%)} 
            & \makecell{(\emph{A})\\Extr. good,\\ very good,\\ good \\ (\%)} 
            & \makecell{(\emph{B})\\Extr. good,\\ very good,\\ good,\\ acceptable \\ (\%)} 
            & \makecell{(\emph{A+})\\Extr. good,\\ very good \\ (\%)} 
            & \makecell{(\emph{A})\\Extr. good,\\ very good,\\ good \\ (\%)} 
            & \makecell{(\emph{B})\\Extr. good,\\ very good,\\ good,\\ acceptable \\ (\%)} \\
        \midrule
        o3-mini & \cellcolor{green}\textcolor{blue}{\textbf{55.79}} & 67.37 & 80.00 & \cellcolor{green}{\textbf{44.12}} & 64.71 & 80.39 \\ 
        gpt-4o & \textcolor{blue}{\textbf{43.16}} & 56.84 & 80.00 & 39.22 & 58.82 & 85.29 \\ 
        claude-3-5-haiku & \textcolor{blue}{\textbf{33.68}} & 49.47 & 69.47 & 27.45 & 46.08 & 75.49 \\ 
        gpt-4o-mini & \textcolor{blue}{\textbf{32.63}} & 44.21 & 58.95 & 31.37 & 59.80 & 83.33 \\ 
        claude-sonnet-4 & 27.37 & 41.05 & 65.26 & \textcolor{magenta}{\textbf{32.35}} & 57.84 & 82.35 \\ 
        claude-3-7-sonnet & \textcolor{blue}{\textbf{25.26}} & 43.16 & 68.42 & 14.71 & 35.29 & 69.61 \\ 
        nova-pro & \textcolor{blue}{\textbf{22.11}} & 40.00 & 52.63 & 17.65 & 43.14 & 65.69 \\ 
        llama3-70b-instruct  & \textcolor{blue}{\textbf{22.11}} &	37.89 &	51.58 & 12.75 &	33.33 &	58.82 \\ 
        llama4-maverick-17b-instruct & \textcolor{blue}{\textbf{21.05}} & 38.95 & 57.89 & 20.59 & 42.16 & 66.67 \\ 
        nova-micro & \textcolor{blue}{\textbf{15.79}} & 33.68 & 46.32 & 15.69 & 40.20 & 60.78 \\ 
        gpt-4.1-nano & 12.63 & 21.05 & 38.95 & \textcolor{magenta}{\textbf{19.61}} & 39.22 & 66.67 \\ 
        nova-lite & 6.54  & 25.48 & 34.96 & \textcolor{magenta}{\textbf{20.49}} & 43.04 & 59.71 \\ 
        llama3-2-3b-instruct & 3.16  & 5.26  & 14.74 & \textcolor{magenta}{\textbf{7.84}} & 17.65 & 38.24 \\ 
        llama3-2-1b-instruct & 0.00  & 0.00  & 2.11 & 0.00 & 0.00 & 0.98 \\ 
        \midrule
        \textbf{Grand Total (\#)} & \multicolumn{3}{c|}{\textbf{241}} & \multicolumn{3}{c}{\textbf{259}} \\
        \bottomrule
    \end{tabularx}
    \caption{Plan Generation quality comparison using prompts \textbf{without lineage} for \textbf{objective} and \textbf{subjective} queries, judged using the \textbf{one-shot evaluator} on test data. The highest score in the \textbf{Extremely Good, Very Good} bucket is highlighted in green for both query categories; blue indicates better performance on Objective queries, and magenta indicates better performance on Subjective queries.}
    \label{tab:plan_gen_result_objective_vs_subjective_without_lineage_one_shot_eval}
\end{table*}

\begin{table*}[!ht]
    \centering
    \small

    \setlength{\tabcolsep}{6pt}

    \begin{tabularx}{\textwidth}{l|ccc|ccc}
        \toprule
        \multirow{2}{*}{\textbf{LLM}} 
        & \multicolumn{3}{c|}{\textbf{Simple Queries}} 
        & \multicolumn{3}{c}{\textbf{Compound Queries}} \\
        \cmidrule(lr){2-4} \cmidrule(lr){5-7}
            & \makecell{(\emph{A+})\\Extr. good,\\ very good \\ (\%)} 
            & \makecell{(\emph{A})\\Extr. good,\\ very good,\\ good \\ (\%)} 
            & \makecell{(\emph{B})\\Extr. good,\\ very good,\\ good,\\ acceptable \\ (\%)} 
            & \makecell{(\emph{A+})\\Extr. good,\\ very good \\ (\%)} 
            & \makecell{(\emph{A})\\Extr. good,\\ very good,\\ good \\ (\%)} 
            & \makecell{(\emph{B})\\Extr. good,\\ very good,\\ good,\\ acceptable \\ (\%)} \\
        \midrule
        o3-mini & \cellcolor{green}\textcolor{blue}{\textbf{45.28}} & 51.89 & 69.81 & 40.66 & 54.95 & 75.82 \\ 
        gpt-4o & 42.45 & 55.66 & 73.58 & \cellcolor{green}\textcolor{magenta}{\textbf{49.45}} & 70.33 & 90.11 \\ 
        claude-3-7-sonnet & \textcolor{blue}{\textbf{41.51}} & 52.83 & 69.81 & 17.58 &	42.86 &	71.43 \\ 
        gpt-4o-mini & \textcolor{blue}{\textbf{36.79}} & 51.89 & 66.04 & 26.37 & 45.05 & 62.64 \\ 
        claude-sonnet-4 & \textcolor{blue}{\textbf{33.02}} & 51.89 & 67.92 & 27.47 & 53.85 & 80.22 \\ 
        claude-3-5-haiku & \textcolor{blue}{\textbf{29.25}} & 49.06 & 62.26 & 19.78 & 41.76 & 62.64 \\ 
        llama4-maverick-17b-instruct & \textcolor{blue}{\textbf{28.30}} & 47.17 & 61.32 & 10.99 & 24.18 & 46.15 \\ 
        gpt-4.1-nano & \textcolor{blue}{\textbf{27.36}} & 35.85 & 60.38 & 5.49  & 19.78 & 49.45 \\ 
        nova-micro & \textcolor{blue}{\textbf{25.47}} & 38.68 & 54.72 & 15.38 & 31.87 & 51.65 \\ 
        nova-pro & \textcolor{blue}{\textbf{19.81}} & 46.23 & 61.32 & 16.48 & 38.46 & 62.64 \\ 
        llama3-70b-instruct  & \textcolor{blue}{\textbf{19.81}} &	42.45 &	64.15 &  17.58 &	46.15 &	54.95 \\ 
        nova-lite & \textcolor{blue}{\textbf{16.98}} & 33.02 & 50.00 & 9.89  & 28.57 & 48.35 \\ 
        llama3-2-3b-instruct &  \textcolor{blue}{\textbf{6.60}} & 10.38 & 17.92 & 3.30  & 10.99 & 16.48 \\ 
        llama3-2-1b-instruct &  0.00 &  0.00 &  0.00 & 0.00  & 0.00  & 0.00  \\ 
        \midrule
        \textbf{Grand Total (\#)} & \multicolumn{3}{c|}{\textbf{269}} & \multicolumn{3}{c}{\textbf{231}} \\
        \bottomrule
    \end{tabularx}
    \caption{Plan Generation quality comparison using prompts \textbf{with lineage} for \textbf{Simple} and \textbf{Compound} queries using the \textbf{one-shot evaluator} on test data. The highest score in the \textbf{Extremely Good, Very Good} bucket is highlighted in green for both query categories; blue indicates better performance on Simple queries, and magenta indicates better performance on Compound queries.}
    \label{tab:plan_gen_result_simple_vs_compound_with_lineage_one_shot_eval}
\end{table*}

\begin{table*}[!ht]
    \centering
    \small

    \setlength{\tabcolsep}{6pt}

    \begin{tabularx}{\textwidth}{l|ccc|ccc}
        \toprule
        \multirow{2}{*}{\textbf{LLM}} 
        & \multicolumn{3}{c|}{\textbf{Simple Queries}} 
        & \multicolumn{3}{c}{\textbf{Compound Queries}} \\
        \cmidrule(lr){2-4} \cmidrule(lr){5-7}
            & \makecell{(\emph{A+})\\Extr. good,\\ very good \\ (\%)} 
            & \makecell{(\emph{A})\\Extr. good,\\ very good,\\ good \\ (\%)} 
            & \makecell{(\emph{B})\\Extr. good,\\ very good,\\ good,\\ acceptable \\ (\%)} 
            & \makecell{(\emph{A+})\\Extr. good,\\ very good \\ (\%)} 
            & \makecell{(\emph{A})\\Extr. good,\\ very good,\\ good \\ (\%)} 
            & \makecell{(\emph{B})\\Extr. good,\\ very good,\\ good,\\ acceptable \\ (\%)} \\
        \midrule
        o3-mini & \cellcolor{green}{\textbf{47.17}} & 63.21 & 75.47 & \cellcolor{green}\textcolor{magenta}{\textbf{52.75}} & 69.23 & 85.71 \\ 
        gpt-4o & \textcolor{blue}{\textbf{44.34}} & 57.55 & 83.02 & 37.36 &	58.24 &	82.42\\ 
        claude-sonnet-4 & \textcolor{blue}{\textbf{38.68}} & 50.00 & 67.92 & 19.78 & 49.45 & 81.32 \\ 
        gpt-4o-mini & \textcolor{blue}{\textbf{36.79}} & 52.83 & 66.98 & 26.37 & 51.65 & 76.92 \\ 
        claude-3-5-haiku & \textcolor{blue}{\textbf{33.96}} & 50.94 & 66.98 & 26.37 & 43.96 & 79.12 \\ 
        llama4-maverick-17b-instruct & \textcolor{blue}{\textbf{29.25}} & 43.40 & 68.87 & 10.99 & 37.36 & 54.95 \\ 
        gpt-4.1-nano & \textcolor{blue}{\textbf{28.30}} & 43.40 & 61.32 & 2.20  & 15.38 & 43.96 \\ 
        nova-pro & \textcolor{blue}{\textbf{26.42}} & 46.23 & 61.32 & 12.09 & 36.26 & 57.14 \\ 
        llama3-70b-instruct  & \textcolor{blue}{\textbf{23.58}} &	46.23 &	61.32 & 9.89 &	23.08 &	48.35\\ 
        claude-3-7-sonnet & \textcolor{blue}{\textbf{21.70}} & 33.96 & 61.32 &	17.58 &	45.05 & 78.02 \\ 
        nova-micro & \textcolor{blue}{\textbf{19.81}} & 40.57 & 60.38 & 10.99 & 32.97 & 46.15 \\ 
        nova-lite & \textcolor{blue}{\textbf{16.66}} & 40.25 & 53.45 & 10.22 & 27.80 & 40.99 \\ 
        llama3-2-3b-instruct & \textcolor{blue}{\textbf{9.43}} & 16.98 & 33.02 & 1.10  & 5.49  & 19.78 \\ 
        llama3-2-1b-instruct & 0.00 &  0.00 &  2.83  & 0.00  & 0.00  & 0.00 \\ 
        \midrule
        \textbf{Grand Total (\#)} & \multicolumn{3}{c|}{\textbf{269}} & \multicolumn{3}{c}{\textbf{231}} \\
        \bottomrule
    \end{tabularx}
    \caption{Plan Generation quality comparison using prompts \textbf{without lineage} for \textbf{Simple} and \textbf{Compound} queries using the \textbf{one-shot evaluator} on test data. The highest score in the \textbf{Extremely Good, Very Good} bucket is highlighted in green for both query categories; blue indicates better performance on Simple queries, and magenta indicates better performance on Compound queries.}
    \label{tab:plan_gen_result_simple_vs_compound_without_lineage_one_shot_eval}
\end{table*}

\begin{table*}[!ht]
    \centering
    \small

    \setlength{\tabcolsep}{6pt}

    \begin{tabularx}{\textwidth}{l|ccc|ccc}
        \toprule
        \multirow{2}{*}{\textbf{LLM}} 
        & \multicolumn{3}{c|}{\textbf{[1, 4] steps in Best Possible Plan}} 
        & \multicolumn{3}{c}{\textbf{[5, 15] steps in Best Possible Plan}} \\
        \cmidrule(lr){2-4} \cmidrule(lr){5-7}
            & \makecell{(\emph{A+})\\Extr. good,\\ very good \\ (\%)} 
            & \makecell{(\emph{A})\\Extr. good,\\ very good,\\ good \\ (\%)} 
            & \makecell{(\emph{B})\\Extr. good,\\ very good,\\ good,\\ acceptable \\ (\%)} 
            & \makecell{(\emph{A+})\\Extr. good,\\ very good \\ (\%)} 
            & \makecell{(\emph{A})\\Extr. good,\\ very good,\\ good \\ (\%)} 
            & \makecell{(\emph{B})\\Extr. good,\\ very good,\\ good,\\ acceptable \\ (\%)} \\
        \midrule
        o3-mini & \cellcolor{green}\textcolor{blue}{\textbf{52.48}} & 60.28 & 75.89 & 19.64 & 35.71 & 64.29 \\ 
        gpt-4o & \textcolor{blue}{\textbf{48.23}} &	62.41 &	78.72 & \cellcolor{green}{\textbf{39.29}} &	62.50 &	87.50 \\ 
        gpt-4o-mini & \textcolor{blue}{\textbf{38.30}} & 54.61 & 68.09 & 16.07 & 33.93 & 55.36 \\ 
        claude-3-7-sonnet & \textcolor{blue}{\textbf{37.59}} & 54.61 & 72.34 & 12.50 & 32.14 & 66.07 \\ 
        claude-sonnet-4 & \textcolor{blue}{\textbf{32.62}} & 52.48 & 70.21 & 25.00 & 53.57 & 82.14 \\ 
        claude-3-5-haiku & \textcolor{blue}{\textbf{28.37}} & 47.52 & 61.70 & 16.07 & 41.07 & 64.29 \\ 
        nova-micro & \textcolor{blue}{\textbf{26.24}} & 41.13 & 56.74 & 7.14 & 21.43 & 44.64 \\ 
        llama4-maverick-17b-instruct & \textcolor{blue}{\textbf{25.53}} & 44.68 & 60.28 & 7.14 & 16.07 & 39.29 \\ 
        llama3-70b-instruct & \textcolor{blue}{\textbf{24.11}} &	51.77 &	68.79 & 5.36 &	25.00 &	37.50 \\ 
        gpt-4.1-nano & \textcolor{blue}{\textbf{21.99}} & 32.62 & 58.16 & 5.36 & 17.86 & 48.21 \\ 
        nova-pro & \textcolor{blue}{\textbf{20.57}} & 46.10 & 63.83 & 12.50 & 33.93 & 57.14 \\ 
        nova-lite & \textcolor{blue}{\textbf{17.73}} & 34.04 & 51.77 & 3.57 & 23.21 & 42.86 \\ 
        llama3-2-3b-instruct & \textcolor{blue}{\textbf{7.09}} & 12.06 & 19.15 & 0.00 & 7.14 & 12.50 \\ 
        llama3-2-1b-instruct & 0.00 & 0.00 & 0.00 & 0.00 & 0.00 & 0.00 \\ 
        \midrule
        \textbf{Grand Total (\#)} & \multicolumn{3}{c|}{\textbf{358}} & \multicolumn{3}{c}{\textbf{142}} \\
        \bottomrule
    \end{tabularx}
    \caption{Plan Generation quality comparison using prompts \textbf{with lineage} for queries with \textbf{1 to 4 steps} and \textbf{5 to 15 steps} in the \textbf{Best Possible Plan} using the \textbf{one-shot evaluator} on test data. The highest score in the \textbf{Extremely Good, Very Good} bucket is highlighted in green for both the groups; blue indicates better performance on queries with [1,4] steps in the Best Possible Plan, and magenta indicates better performance on queries with [5,15] steps in the Best Possible Plan.}
    \label{tab:plan_gen_result_1to4steps_vs_5to15steps_with_lineage_one_shot_eval}
\end{table*}

\begin{table*}[!ht]
    \centering
    \small

    \setlength{\tabcolsep}{6pt}

    \begin{tabularx}{\textwidth}{l|ccc|ccc}
        \toprule
        \multirow{2}{*}{\textbf{LLM}} 
        & \multicolumn{3}{c|}{\textbf{[1, 4] steps in Best Possible Plan}} 
        & \multicolumn{3}{c}{\textbf{[5, 15] steps in Best Possible Plan}} \\
        \cmidrule(lr){2-4} \cmidrule(lr){5-7}
            & \makecell{(\emph{A+})\\Extr. good,\\ very good \\ (\%)} 
            & \makecell{(\emph{A})\\Extr. good,\\ very good,\\ good \\ (\%)} 
            & \makecell{(\emph{B})\\Extr. good,\\ very good,\\ good,\\ acceptable \\ (\%)} 
            & \makecell{(\emph{A+})\\Extr. good,\\ very good \\ (\%)} 
            & \makecell{(\emph{A})\\Extr. good,\\ very good,\\ good \\ (\%)} 
            & \makecell{(\emph{B})\\Extr. good,\\ very good,\\ good,\\ acceptable \\ (\%)} \\
        \midrule
        o3-mini & \cellcolor{green}\textcolor{blue}{\textbf{54.61}} & 71.63 & 80.85 & \cellcolor{green}{\textbf{37.50}} & 51.79 & 78.57 \\ 
        gpt-4o & \textcolor{blue}{\textbf{43.26}} & 58.16 & 82.27 & 35.71 &	57.14 &	83.93 \\ 
        gpt-4o-mini & \textcolor{blue}{\textbf{34.75}} & 50.35 & 68.79 & 25.00 & 57.14 & 78.57 \\ 
        claude-sonnet-4   & \textcolor{blue}{\textbf{34.75}} & 48.23 & 68.09 & 17.86 & 53.57 & 89.29 \\ 
        claude-3-5-haiku & \textcolor{blue}{\textbf{31.91}} & 51.06 & 70.92 & 26.79 & 39.29 & 76.79 \\ 
        llama4-maverick-17b-instruct & \textcolor{blue}{\textbf{25.53}} & 43.26 & 68.79 & 8.93 & 33.93 & 46.43 \\ 
        nova-pro & \textcolor{blue}{\textbf{23.40}} & 45.39 & 60.99 & 10.71 & 32.14 & 55.36 \\ 
        gpt-4.1-nano & \textcolor{blue}{\textbf{22.70}} & 37.59 & 57.45 & 0.00 & 12.50 & 42.86 \\ 
        claude-3-7-sonnet & \textcolor{blue}{\textbf{20.57}} & 36.17 & 63.83 & 17.86 & 46.43 & 82.14 \\ 
        llama3-70b-instruct & \textcolor{blue}{\textbf{20.57}} &	40.43 &	58.16 & 8.93 &	23.21 &	48.21 \\ 
        nova-micro & \textcolor{blue}{\textbf{16.31}} & 37.59 & 54.61 & 14.29 & 35.71 & 51.79 \\ 
        nova-lite & \textcolor{blue}{\textbf{15.60}} & 40.43 & 53.90 & 8.93 & 19.64 & 39.29 \\ 
        llama3-2-3b-instruct & \textcolor{blue}{\textbf{7.09}} & 12.77 & 28.37 & 1.79 & 8.93 & 23.21 \\ 
        llama3-2-1b-instruct & 0.00 & 0.00 & 2.13 & 0.00 & 0.00 & 0.00 \\ 
        \midrule
        \textbf{Grand Total (\#)} & \multicolumn{3}{c|}{\textbf{358}} & \multicolumn{3}{c}{\textbf{142}} \\
        \bottomrule
    \end{tabularx}
    \caption{Plan Generation quality comparison using prompts \textbf{without lineage} for queries with \textbf{1 to 4 steps} and \textbf{5 to 15 steps} in the \textbf{Best Possible Plan} using the \textbf{one-shot evaluator} on test data. The highest score in the \textbf{Extremely Good, Very Good} bucket is highlighted in green for both the groups; blue indicates better performance on queries with [1,4] steps in the Best Possible Plan, and magenta indicates better performance on queries with [5,15] steps in the Best Possible Plan.}
    \label{tab:plan_gen_result_1to4steps_vs_5to15steps_without_lineage_one_shot_eval}
\end{table*}

\begin{table*}[ht]
    \centering
    \small
    \setlength{\tabcolsep}{8pt}
    
    \begin{tabular}{|l|c|c|c|c|}
        \hline
        \multirow{2}{*}{\textbf{LLM}} 
        & \multicolumn{1}{c|}{\textbf{Zero Hop}} 
        & \multicolumn{1}{c|}{\textbf{One Hop}} 
        & \multicolumn{1}{c|}{\textbf{Two Hop}}
        & \multicolumn{1}{c|}{\textbf{Three-Plus Hop}} \\
        \cline{2-5}
        & \makecell{(\emph{A+})\\Extr. good, \\ very good \\ (\%)}
        & \makecell{(\emph{A+})\\Extr. good, \\ very good \\ (\%)}
        & \makecell{(\emph{A+})\\Extr. good, \\ very good \\ (\%)}
        & \makecell{(\emph{A+})\\Extr. good, \\ very good \\ (\%)} \\
        \hline
        claude-sonnet-4 & \cellcolor{green}{\textbf{25.00}} & 31.40 & \textcolor{blue}{\textbf{38.89}} & 17.24 \\
        llama4-maverick-17b-instruct & 21.43 & \textcolor{blue}{\textbf{26.74}} & 14.81 & 10.34 \\
        o3-mini & 17.86 & \cellcolor{green}\textcolor{blue}{\textbf{55.81}} & 51.85 & 13.79 \\
        gpt-4o-mini & 14.29 & \textcolor{blue}{\textbf{45.35}} & 31.48 & 10.34 \\
        claude-3-7-sonnet & 14.29 & \textcolor{blue}{\textbf{41.86}} & 35.19 & 3.45 \\
        gpt-4o & 10.71 & 51.16 & \cellcolor{green}\textcolor{blue}{\textbf{66.67}} & \cellcolor{green}{\textbf{24.14}} \\
        nova-pro & 7.14 & \textcolor{blue}{\textbf{25.58}} & 16.67 & 10.34 \\
        llama3-70b-instruct & 7.14 & \textcolor{blue}{\textbf{25.58}} & 18.52 & 10.34 \\
        gpt-4.1-nano & 7.14 & \textcolor{blue}{\textbf{25.58}} & 16.67 & 3.45 \\
        nova-micro & 0.00 & \textcolor{blue}{\textbf{34.88}} & 14.81 & 10.34 \\
        nova-lite & 0.00 & \textcolor{blue}{\textbf{24.42}} & 7.41 & 6.90 \\
        llama3-2-3b-instruct & 0.00 & \textcolor{blue}{\textbf{10.47}} & 1.85 & 0.00 \\
        llama3-2-1b-instruct & 0.00 & 0.00 & 0.00 & 0.00 \\
        claude-3-5-haiku & 0.00 & 30.23 & \textcolor{blue}{\textbf{33.33}} & 17.24 \\
        \midrule
        \textbf{Grand Total (\#)} & \textbf{88} & \textbf{203} & \textbf{140} & \textbf{69} \\
        \hline
    \end{tabular}
    \caption{Plan Generation quality comparison using prompts \textbf{with lineage} for \textbf{Zero-Hop}, \textbf{One-Hop}, \textbf{Two-Hop} and \textbf{Three-Plus Hop} queries using the \textbf{one-shot evaluator} on test data. The highest score in the \textbf{Extremely Good, Very Good} bucket is highlighted in green for each group; The highest score across groups for each model is highlighted in blue.}
    \label{tab:plan_gen_result_with_lineage_hops_one_shot_eval}
\end{table*}

\begin{table*}[ht]
    \centering
    \small
    \setlength{\tabcolsep}{8pt}
    
    \begin{tabular}{|l|c|c|c|c|}
        \hline
        \multirow{2}{*}{\textbf{LLM}} 
        & \multicolumn{1}{c|}{\textbf{Zero Hop}} 
        & \multicolumn{1}{c|}{\textbf{One Hop}} 
        & \multicolumn{1}{c|}{\textbf{Two Hop}}
        & \multicolumn{1}{c|}{\textbf{Three-Plus Hop}} \\
        \cline{2-5}
        & \makecell{(\emph{A+})\\Extr. good, \\ very good \\ (\%)}
        & \makecell{(\emph{A+})\\Extr. good, \\ very good \\ (\%)}
        & \makecell{(\emph{A+})\\Extr. good, \\ very good \\ (\%)}
        & \makecell{(\emph{A+})\\Extr. good, \\ very good \\ (\%)} \\
        \hline
        nova-pro & \cellcolor{green}{\textbf{14.29}} & 23.26 & \textcolor{blue}{\textbf{27.78}} & 0.00 \\
        llama4-maverick-17b-instruct & \cellcolor{green}{\textbf{14.29}} & \textcolor{blue}{\textbf{26.74}} & 24.07 & 3.45 \\
        claude-sonnet-4 & 10.71 & \textcolor{blue}{\textbf{39.53}} & 31.48 & 17.24 \\
        claude-3-7-sonnet & 10.71 & 20.93 & \textcolor{blue}{\textbf{25.93}} & 13.79 \\
        o3-mini & 7.14 & \cellcolor{green}{\textbf{62.79}} & \cellcolor{green}\textcolor{blue}{\textbf{62.96}} & 27.59 \\
        gpt-4o & 7.14 & 47.67 & \textcolor{blue}{\textbf{53.70}} & \cellcolor{green}{\textbf{31.03}} \\
        nova-micro & 3.57 & 17.44 & \textcolor{blue}{\textbf{24.07}} & 6.90 \\
        nova-lite & 3.57 & \textcolor{blue}{\textbf{17.44}} & 14.81 & 10.34 \\
        llama3-70b-instruct & 3.57 & \textcolor{blue}{\textbf{24.42}} & 20.37 & 3.45 \\
        llama3-2-3b-instruct & 3.57 & 5.81 & \textcolor{blue}{\textbf{7.41}} & 3.45 \\
        gpt-4o-mini & 3.57 & \textcolor{blue}{\textbf{40.70}} & 35.19 & 27.59 \\
        gpt-4.1-nano & 3.57 & \textcolor{blue}{\textbf{29.07}} & 11.11 & 0.00 \\
        claude-3-5-haiku & 3.57 & 34.88 & \textcolor{blue}{\textbf{40.74}} & 24.14 \\
        llama3-2-1b-instruct & 0.00 & 0.00 & 0.00 & 0.00 \\
        \midrule
        \textbf{Grand Total (\#)} & \textbf{88} & \textbf{203} & \textbf{140} & \textbf{69} \\
        \hline
    \end{tabular}
    \caption{Plan Generation quality comparison using prompts \textbf{without lineage} for \textbf{Zero-Hop}, \textbf{One-Hop}, \textbf{Two-Hop} and \textbf{Three-Plus Hop} queries using the \textbf{one-shot evaluator} on test data. The highest score in the \textbf{Extremely Good, Very Good} bucket is highlighted in green for each group; The highest score across groups for each model is highlighted in blue.}
    \label{tab:plan_gen_result_without_lineage_hops_one_shot_eval}
\end{table*}

\begin{table*}[!ht]
    \centering
    \small

    \setlength{\tabcolsep}{3pt}

    \begin{tabularx}{\textwidth}{>{\raggedright\arraybackslash}l*{4}{|>{\centering\arraybackslash}X}}
        \toprule
        \multirow{2}{*}{\textbf{LLM}} 
        & \multicolumn{2}{c|}{\textbf{With Lineage}} 
        & \multicolumn{2}{c}{\textbf{Without Lineage}} \\
        \cmidrule(lr){2-3} \cmidrule(lr){4-5}
            & \makecell{\textbf{Objective} \\ \textbf{Queries}}
            & \makecell{\textbf{Subjective} \\ \textbf{Queries}}
            & \makecell{\textbf{Objective} \\ \textbf{Queries}}
            & \makecell{\textbf{Subjective} \\ \textbf{Queries}} \\
        \midrule
        claude-3-7-sonnet & 82.9 & \textcolor{magenta}{\textbf{85.58}} & 82.44 & \textcolor{magenta}{\textbf{84.3}} \\
        gpt-4o & 81.94 & \textcolor{magenta}{\textbf{82.02}} & \textcolor{blue}{\textbf{83.49}} & 83.44 \\
        claude-sonnet-4 & \textcolor{blue}{\textbf{80.85}} & 80.09 & 76.1 & \textcolor{magenta}{\textbf{79.61}} \\
        llama4-maverick-17b-instruct & 79.49 & \textcolor{magenta}{\textbf{84.21}} & 83.04 & \textcolor{magenta}{\textbf{83.5}} \\
        gpt-4o-mini & \textcolor{blue}{\textbf{78.64}} & 76.76 & \textcolor{blue}{\textbf{84.35}} & 82.85 \\
        nova-pro & 78.33 & \textcolor{magenta}{\textbf{79.73}} & \textcolor{blue}{\textbf{82.78}} & 82.67 \\
        gpt-4.1-nano & \textcolor{blue}{\textbf{78.2}} & 77.59 & \textcolor{blue}{\textbf{78.33}} & 75.52 \\
        llama3-3-70b-instruct & 76.65 & \textcolor{magenta}{\textbf{81.06}} & 80.21 & \textcolor{magenta}{\textbf{82.09}} \\
        nova-micro & 76.05 & \textcolor{magenta}{\textbf{77.28}} & 79.62 & \textcolor{magenta}{\textbf{83.28}} \\
        o3-mini & 71.49 & \textcolor{magenta}{\textbf{72.64}} & 72.91 & \textcolor{magenta}{\textbf{75.83}} \\
        llama3-2-3b-instruct & 70.5 & \textcolor{magenta}{\textbf{71.58}} & 75.14 & \textcolor{magenta}{\textbf{76.52}} \\
        claude-3-5-haiku & 69.59 & \textcolor{magenta}{\textbf{72.39}} & 78.68 & \textcolor{magenta}{\textbf{83.63}} \\
        nova-lite & 68.01 & \textcolor{magenta}{\textbf{71.81}} & 73.08 & \textcolor{magenta}{\textbf{77.51}} \\
        llama3-2-1b-instruct & 20.18 & \textcolor{magenta}{\textbf{21.11}} & \textcolor{blue}{\textbf{57.94}} & 55.53 \\
        \midrule
        \textbf{Grand Total (\#)} & \textbf{241} & \textbf{259} & \textbf{241} & \textbf{259} \\
        \bottomrule
    \end{tabularx}
    \caption{Plan Generation quality comparison using prompts \textbf{with and without lineage} for \textbf{Objective and Subjective} queries, judged using the \textbf{Overall Score} of the \textbf{metric-wise evaluator} on test data (\textbf{500} queries). For the two groups - \textbf{With Lineage} and \textbf{Without Lineage} - blue cells denote better performance on Objective queries, while magenta cells denote better performance on Subjective queries.}
    \label{tab:plan_gen_result_object_and_subject_queries_metric_wise_eval}
\end{table*}

\begin{table*}[!ht]
    \centering
    \small

    \setlength{\tabcolsep}{3pt}

    \begin{tabularx}{\textwidth}{>{\raggedright\arraybackslash}l*{4}{|>{\centering\arraybackslash}X}}
        \toprule
        \multirow{2}{*}{\textbf{LLM}} 
        & \multicolumn{2}{c|}{\textbf{With Lineage}} 
        & \multicolumn{2}{c}{\textbf{Without Lineage}} \\
        \cmidrule(lr){2-3} \cmidrule(lr){4-5}
            & \makecell{\textbf{Simple} \\ \textbf{Queries}}
            & \makecell{\textbf{Compound} \\ \textbf{Queries}}
            & \makecell{\textbf{Simple} \\ \textbf{Queries}}
            & \makecell{\textbf{Compound} \\ \textbf{Queries}} \\
        \midrule
        claude-3-7-sonnet & \textcolor{blue}{\textbf{86.17}} & 82.99 & \textcolor{blue}{\textbf{84.6}} & 81.51 \\
        llama4-maverick-17b-instruct & \textcolor{blue}{\textbf{85.01}} & 78.83 & 82.39 & \textcolor{magenta}{\textbf{82.53}} \\
        gpt-4o & \textcolor{blue}{\textbf{82.83}} & 79.59 & \textcolor{blue}{\textbf{83.48}} & 81.73 \\
        nova-pro & \textcolor{blue}{\textbf{81.17}} & 77 & \textcolor{blue}{\textbf{83.72}} & 81.4 \\
        llama3-3-70b-instruct & \textcolor{blue}{\textbf{80.36}} & 78.95 & \textcolor{blue}{\textbf{82.42}} & 79.36 \\
        gpt-4o-mini & \textcolor{blue}{\textbf{78.64}} & 75.36 & \textcolor{blue}{\textbf{83.38}} & 81.79 \\
        claude-sonnet-4 & 78.42 & \textcolor{magenta}{\textbf{81.58}} & 77.41 & \textcolor{magenta}{\textbf{78.6}} \\
        gpt-4.1-nano & \textcolor{blue}{\textbf{78.25}} & 75.68 & \textcolor{blue}{\textbf{77.8}} & 73.77 \\
        nova-micro & \textcolor{blue}{\textbf{77.53}} & 77.18 & \textcolor{blue}{\textbf{82.79}} & 81.16 \\
        claude-3-5-haiku & \textcolor{blue}{\textbf{73.27}} & 68.76 & \textcolor{blue}{\textbf{83.02}} & 80.83 \\
        llama3-2-3b-instruct & \textcolor{blue}{\textbf{72.21}} & 68.35 & \textcolor{blue}{\textbf{78.63}} & 72.33 \\
        o3-mini & 71.05 & \textcolor{magenta}{\textbf{72.65}} & 73.64 & \textcolor{magenta}{\textbf{74.39}} \\
        nova-lite & \textcolor{blue}{\textbf{70.7}} & 70.49 & 74.22 & \textcolor{magenta}{\textbf{76.4}} \\
        llama3-2-1b-instruct & 19.6 & \textcolor{magenta}{\textbf{21.13}} & \textcolor{blue}{\textbf{56.93}} & 56.45 \\
        \midrule
        \textbf{Grand Total (\#)} & \textbf{269} & \textbf{231} & \textbf{269} & \textbf{231} \\
        \bottomrule
    \end{tabularx}
    \caption{Plan Generation quality comparison using prompts \textbf{with and without lineage} for \textbf{Simple and Compound} queries, judged using the \textbf{Overall Score} of the \textbf{metric-wise evaluator} on test data (\textbf{500} queries). For the two groups - \textbf{With Lineage} and \textbf{Without Lineage} - blue cells denote better performance on Simple queries, while magenta cells denote better performance on Compound queries.}
    \label{tab:plan_gen_result_simple_and_comp_queries_metric_wise_eval}
\end{table*}

\begin{table*}[!ht]
    \centering
    \small

    \setlength{\tabcolsep}{10pt}

    \begin{tabularx}{\textwidth}{l|cccc|cccc}
        \toprule
        \multirow{2}{*}{\textbf{LLM}} 
        & \multicolumn{4}{c|}{\textbf{With Lineage}} 
        & \multicolumn{4}{c}{\textbf{Without Lineage}} \\
        \cmidrule(lr){2-5} \cmidrule(lr){6-9}
            & \makecell{\textbf{Zero} \\ \textbf{Hop}}
            & \makecell{\textbf{One} \\ \textbf{Hop}}
            & \makecell{\textbf{Two} \\ \textbf{Hop}}
            & \makecell{\textbf{Three} \\ \textbf{Plus} \\ \textbf{Hop}}
            & \makecell{\textbf{Zero} \\ \textbf{Hop}}
            & \makecell{\textbf{One} \\ \textbf{Hop}}
            & \makecell{\textbf{Two} \\ \textbf{Hop}}
            & \makecell{\textbf{Three} \\ \textbf{Plus} \\ \textbf{Hop}} \\
        \midrule
        nova-pro & \textcolor{blue}{\textbf{87.01}} & 80.34 & 77.64 & 77.69 & 83.16 & \textcolor{blue}{\textbf{84.17}} & 81.82 & 83.2 \\
        llama3-3-70b-instruct & \textcolor{blue}{\textbf{85.4}} & 81.23 & 79.98 & 76.22 & 82.27 & \textcolor{blue}{\textbf{82.56}} & 80.95 & 78.19 \\
        claude-3-7-sonnet & 85.36 & \textcolor{blue}{\textbf{86.53}} & 84.05 & 83.49 & 82.52 & \textcolor{blue}{\textbf{84.81}} & 83.49 & 82.77 \\
        gpt-4o & \textcolor{blue}{\textbf{83.67}} & 83.41 & 79.84 & 79.39 & 72.55 & \textcolor{blue}{\textbf{85}} & 83.14 & 81.03 \\
        llama4-maverick-17b-instruct & 83.52 & \textcolor{blue}{\textbf{85.1}} & 79.31 & 82.6 & 80.71 & \textcolor{blue}{\textbf{84.43}} & 82.73 & 83.17 \\
        claude-sonnet-4 & 76.57 & 79.64 & \textcolor{blue}{\textbf{81.68}} & 81.52 & 76.58 & 77.82 & \textcolor{blue}{\textbf{79.79}} & 77.73 \\
        nova-micro & 74.77 & 77.24 & \textcolor{blue}{\textbf{77.97}} & 74.13 & 83.21 & \textcolor{blue}{\textbf{83.32}} & 81.6 & 81.33 \\
        gpt-4o-mini & 73.39 & \textcolor{blue}{\textbf{78.25}} & 78.24 & 72.9 & 74.07 & \textcolor{blue}{\textbf{84.35}} & 82.54 & 83.45 \\
        gpt-4.1-nano & 69.9 & \textcolor{blue}{\textbf{78.34}} & 77.4 & 75.94 & \textcolor{blue}{\textbf{80.78}} & 78.79 & 73.48 & 73.28 \\
        nova-lite & 67.94 & 70.47 & \textcolor{blue}{\textbf{70.86}} & 70.11 & 74.94 & 73.91 & \textcolor{blue}{\textbf{77.67}} & 76.22 \\
        o3-mini & 66.23 & 72.33 & \textcolor{blue}{\textbf{73.71}} & 70.57 & 76.14 & \textcolor{blue}{\textbf{77.02}} & 73.34 & 73.63 \\
        llama3-2-3b-instruct & 65.95 & \textcolor{blue}{\textbf{72.15}} & 69.73 & 69.6 & 77.39 & \textcolor{blue}{\textbf{78.31}} & 74.47 & 73.34 \\
        claude-3-5-haiku & 63.45 & \textcolor{blue}{\textbf{72.83}} & 69.08 & 71.49 & 74.17 & 82.29 & \textcolor{blue}{\textbf{83.39}} & 80.21 \\
        llama3-2-1b-instruct & 14.76 & 20.32 & 19.7 & \textcolor{blue}{\textbf{22.3}} & \textcolor{blue}{\textbf{61.93}} & 56.68 & 54.61 & 57.82 \\
        \midrule
        \textbf{Grand Total (\#)} & \textbf{88} & \textbf{203} & \textbf{140} & \textbf{69} & \textbf{88} & \textbf{203} & \textbf{140} & \textbf{69} \\
        \bottomrule
    \end{tabularx}
    \caption{Plan Generation quality comparison using prompts \textbf{with and without lineage} grouped by \textbf{Number of Hops}, judged using the \textbf{Overall Score} of the \textbf{metric-wise evaluator} on test data (\textbf{500} queries). For the two groups - \textbf{With Lineage} and \textbf{Without Lineage} - blue cells denote the best performance across categories of \textbf{Number of Hops} for each model.}
    \label{tab:plan_gen_result_num_hops_metric_wise_eval}
\end{table*}

\begin{table*}[!ht]
    \centering
    \small

    \setlength{\tabcolsep}{3pt}

    \begin{tabularx}{\textwidth}{>{\raggedright\arraybackslash}l*{4}{|>{\centering\arraybackslash}X}}
        \toprule
        \multirow{2}{*}{\textbf{LLM}} 
        & \multicolumn{2}{c|}{\textbf{With Lineage}} 
        & \multicolumn{2}{c}{\textbf{Without Lineage}} \\
        \cmidrule(lr){2-3} \cmidrule(lr){4-5}
            & \makecell{\textbf{[1, 4]} \\ \textbf{steps}}
            & \makecell{\textbf{[5, 15]} \\ \textbf{steps}}
            & \makecell{\textbf{[1, 4]} \\ \textbf{steps}}
            & \makecell{\textbf{[5, 15]} \\ \textbf{steps}} \\
        \midrule
        claude-3-7-sonnet & \textcolor{blue}{\textbf{86.07}} & 82.98 & \textcolor{blue}{\textbf{84.29}} & 82.46 \\
        nova-pro & \textcolor{blue}{\textbf{83.16}} & 80.56 & \textcolor{blue}{\textbf{82.85}} & 82.07 \\
        gpt-4o-mini & \textcolor{blue}{\textbf{82.92}} & 79.07 & \textcolor{blue}{\textbf{83.78}} & 81.37 \\
        gpt-4o & \textcolor{blue}{\textbf{81.84}} & 76.03 & \textcolor{blue}{\textbf{82.92}} & 78.1 \\
        llama4-maverick-17b-instruct & \textcolor{blue}{\textbf{80.63}} & 76.67 & \textcolor{blue}{\textbf{83.7}} & 81.12 \\
        claude-sonnet-4 & \textcolor{blue}{\textbf{79.78}} & 73.17 & \textcolor{blue}{\textbf{83.82}} & 81.62 \\
        gpt-4.1-nano & 79.3 & \textcolor{magenta}{\textbf{82.11}} & 77.7 & \textcolor{magenta}{\textbf{79.34}} \\
        llama3-3-70b-instruct & \textcolor{blue}{\textbf{78.19}} & 75.07 & \textcolor{blue}{\textbf{82.43}} & 81.49 \\
        claude-3-5-haiku & \textcolor{blue}{\textbf{78.17}} & 76.02 & \textcolor{blue}{\textbf{77.51}} & 73.41 \\
        nova-lite & \textcolor{blue}{\textbf{72.37}} & 69.79 & \textcolor{blue}{\textbf{82.76}} & 81.08 \\
        llama3-2-3b-instruct & \textcolor{blue}{\textbf{72.12}} & 71.72 & \textcolor{blue}{\textbf{74.44}} & 74.3 \\
        o3-mini & \textcolor{blue}{\textbf{72.02}} & 67.53 & \textcolor{blue}{\textbf{77.38}} & 72.65 \\
        nova-micro & 70.12 & \textcolor{magenta}{\textbf{71.34}} & 74.37 & \textcolor{magenta}{\textbf{76.73}} \\
        llama3-2-1b-instruct & 19.49 & \textcolor{magenta}{\textbf{21.43}} & \textcolor{blue}{\textbf{58.76}} & 52.68 \\
        \midrule
        \textbf{Grand Total (\#)} & \textbf{358} & \textbf{142} & \textbf{358} & \textbf{142} \\
        \bottomrule
    \end{tabularx}
    \caption{Plan Generation quality comparison using prompts \textbf{with and without lineage} grouped by \textbf{Number of steps in the Best Possible Plan}, judged using the \textbf{Overall Score} of the \textbf{metric-wise evaluator} on test data (\textbf{500} queries). For the two groups - \textbf{With Lineage} and \textbf{Without Lineage} - blue cells denote better performance on queries with [1, 4] steps in the best possible plan, while magenta cells denote better performance on queries with [5, 15] steps in the best possible plan.}
    \label{tab:plan_gen_result_num_steps_metric_wise_eval}
\end{table*}

\begin{table*}[ht]
    \centering
    \small
    \setlength{\tabcolsep}{6pt} 
    
    \begin{tabular}{|l|cc|cc|cc|cc|}
        \hline
        \multirow{3}{*}{\textbf{Tag}} 
        & \multicolumn{4}{c|}{\textbf{Simple Queries}} 
        & \multicolumn{4}{c|}{\textbf{Compound Queries}} \\
        \cline{2-9}
        & \multicolumn{2}{c|}{\textbf{Pre-Loop}} 
        & \multicolumn{2}{c|}{\textbf{Post-Loop}} 
        & \multicolumn{2}{c|}{\textbf{Pre-Loop}} 
        & \multicolumn{2}{c|}{\textbf{Post-Loop}} \\
        \cline{2-9}
        & \# & \% & \# & \% & \# & \% & \# & \% \\
        \hline
        Extremely Bad & 8 & 2.97 & 13 & 4.83 & 20 & 8.66 & 10 & 4.33 \\
        Very Bad & 71 & 26.39 & 58 & 21.56 & 56 & 24.24 & 51 & 22.08 \\
        Bad & 46 & 17.10 & 51 & 18.96 & 61 & 26.41 & 71 & 30.74 \\
        Acceptable & 36 & 13.38 & 23 & 8.55 & 30 & 12.99 & 25 & 10.82 \\
        Good & 63 & \cellcolor{green}{\textbf{23.42}} & 46 & 17.10 & 41 & \cellcolor{green}{\textbf{17.75}} & 38 & 16.45 \\
        Very Good & 30 & 11.15 & 48 & \cellcolor{green}{\textbf{17.84}} & 15 & 6.49 & 25 & \cellcolor{green}{\textbf{10.82}} \\
        Extremely Good & 15 & 5.58 & 30 & \cellcolor{green}{\textbf{11.15}} & 8 & 3.46 & 11 & \cellcolor{green}{\textbf{4.76}} \\
        \hline
        \textbf{Grand Total} & \textbf{269} & \textbf{100.00} & \textbf{269} & \textbf{100.00} & \textbf{231} & \textbf{100.00} & \textbf{231} & \textbf{100.00} \\
       \hline
    \end{tabular}
    \caption{Effectiveness of the \textbf{iterative Evaluator$\rightarrow$Optimizer loop} judged using \textbf{one-shot evaluator} broken down by simple and compound queries on test data. Higher values between \textbf{Pre-Loop} and \textbf{Post-Loop} in the \textbf{Good, Very Good, Extremely Good} buckets are highlighted in green for each group.}
    \label{tab:feedback_loop_performance_simple_compound}
\end{table*}

\begin{table*}[ht]
    \centering
    \small
    \setlength{\tabcolsep}{6pt} 
    
    \begin{tabular}{|l|cc|cc|cc|cc|}
        \hline
        \multirow{3}{*}{\textbf{Tag}} 
        & \multicolumn{4}{c|}{\textbf{Objective Queries}} 
        & \multicolumn{4}{c|}{\textbf{Subjective Queries}} \\
        \cline{2-9}
        & \multicolumn{2}{c|}{\textbf{Pre-Loop}} 
        & \multicolumn{2}{c|}{\textbf{Post-Loop}}
        & \multicolumn{2}{c|}{\textbf{Pre-Loop}}
        & \multicolumn{2}{c|}{\textbf{Post-Loop}} \\
        \cline{2-9}
        & \# & \% & \# & \% & \# & \% & \# & \% \\
        \hline
        Extremely Bad & 13 & 5.39 & 10 & 4.15 & 15 & 5.79 & 13 & 5.02 \\
        Very Bad & 86 & 35.68 & 66 & 27.39 & 41 & 15.83 & 43 & 16.60 \\
        Bad & 58 & 24.07 & 66 & 27.39 & 48 & 18.53 & 56 & 21.62 \\
        Acceptable & 23 & 9.54 & 18 & 7.47 & 43 & 16.60 & 30 & 11.58 \\
        Good & 46 & \cellcolor{green}{\textbf{19.09}} & 33 & 13.69 & 58 & \cellcolor{green}{\textbf{22.39}} & 51 & 19.69 \\
        Very Good & 5 & 2.07 & 25 & \cellcolor{green}{\textbf{10.37}} & 41 & 15.83 & 48 & \cellcolor{green}{\textbf{18.53}} \\
        Extremely Good & 10 & 4.15 & 23 & \cellcolor{green}{\textbf{9.54}} & 13 & 5.02 & 18 & \cellcolor{green}{\textbf{6.95}} \\
        \hline
        \textbf{Grand Total} & \textbf{259} & \textbf{100.00} & \textbf{259} & \textbf{100.00} & \textbf{241} & \textbf{100.00} & \textbf{241} & \textbf{100.00} \\
       \hline
    \end{tabular}

    \caption{Effectiveness of the \textbf{iterative Evaluator$\rightarrow$Optimizer loop} judged using \textbf{one-shot evaluator} broken down by objective and subjective queries on test data. Higher values between \textbf{Pre-Loop} and \textbf{Post-Loop} in the \textbf{Good, Very Good, Extremely Good} buckets are highlighted in green for each group.}
    \label{tab:feedback_loop_performance_objective_subjective}
    
\end{table*}

\begin{table*}[ht]
    \centering
    \small
    \setlength{\tabcolsep}{6pt} 
    
    \begin{tabular}{|l|cc|cc|cc|cc|}
        \hline
        \multirow{3}{*}{\textbf{Tag}} 
        & \multicolumn{4}{c|}{\textbf{Zero-Hop}} 
        & \multicolumn{4}{c|}{\textbf{One-Hop}} \\
        \cline{2-9}
        & \multicolumn{2}{c|}{\textbf{Pre-Loop}}
        & \multicolumn{2}{c|}{\textbf{Post-Loop}}
        & \multicolumn{2}{c|}{\textbf{Pre-Loop}} 
        & \multicolumn{2}{c|}{\textbf{Post-Loop}} \\
        \cline{2-9}
        & \# & \% & \# & \% & \# & \% & \# & \% \\
        \hline
        Extremely Bad & 3 & 3.41 & 0 & 0.00 & 8 & 3.94 & 10 & 4.93 \\
        Very Bad & 55 & 62.50 & 45 & 51.14 & 33 & 16.26 & 28 & 13.79 \\
        Bad & 18 & 20.45 & 28 & 31.82 & 28 & 13.79 & 36 & 17.73 \\
        Acceptable & 5 & 5.68 & 0 & 0.00 & 30 & 14.78 & 20 & 9.85 \\
        Good & 0 & 0.00 & 0 & 0.00 & 66 & \cellcolor{green}{\textbf{32.51}} & 51 & 25.12 \\
        Very Good & 0 & 0.00 & 8 & \cellcolor{green}{\textbf{9.09}} & 28 & 13.79 & 38 & \cellcolor{green}{\textbf{18.72}} \\
        Extremely Good & 7 & 7.95 & 7 & 7.95 & 10 & 4.93 & 20 & \cellcolor{green}{\textbf{9.85}} \\
        \hline
        \textbf{Grand Total} & \textbf{88} & \textbf{100.00} & \textbf{88} & \textbf{100.00} & \textbf{203} & \textbf{100.00} & \textbf{203} & \textbf{100.00} \\
       \hline
    \end{tabular}
    \caption{Effectiveness of the \textbf{iterative Evaluator$\rightarrow$Optimizer loop} judged using \textbf{one-shot evaluator} broken down by number of hops (0-hops and 1-hop) on test data. Higher values between \textbf{Pre-Loop} and \textbf{Post-Loop} in the \textbf{Good, Very Good, Extremely Good} buckets are highlighted in green for each group.}
    \label{tab:feedback_loop_performance_zerohop_onehop}
    
\end{table*}

\begin{table*}[ht]
    \centering
    \small
    \setlength{\tabcolsep}{6pt} 
    
    \begin{tabular}{|l|cc|cc|cc|cc|}
        \hline
        \multirow{3}{*}{\textbf{Tag}} 
        & \multicolumn{4}{c|}{\textbf{Two-Hops}} 
        & \multicolumn{4}{c|}{\textbf{Three-Plus-Hops}} \\
        \cline{2-9}
        & \multicolumn{2}{c|}{\textbf{Pre-Loop}} 
        & \multicolumn{2}{c|}{\textbf{Post-Loop}}
        & \multicolumn{2}{c|}{\textbf{Pre-Loop}}
        & \multicolumn{2}{c|}{\textbf{Post-Loop}} \\
        \cline{2-9}
        & \# & \% & \# & \% & \# & \% & \# & \% \\
        \hline
        Extremely Bad & 15 & 10.71 & 5 & 3.57 & 3 & 4.35 & 8 & 11.59 \\
        Very Bad & 13 & 9.29 & 20 & 14.29 & 26 & 37.68 & 15 & 21.74 \\
        Bad & 43 & 30.71 & 38 & 27.14 & 17 & 24.64 & 20 & 28.99 \\
        Acceptable & 20 & 14.29 & 23 & 16.43 & 10 & 14.49 & 5 & 7.25 \\
        Good & 31 & \cellcolor{green}{\textbf{22.14}} & 25 & 17.86 & 8 & 11.59 & 8 & 11.59 \\
        Very Good & 13 & 9.29 & 18 & \cellcolor{green}{\textbf{12.86}} & 5 & 7.25 & 10 & \cellcolor{green}{\textbf{14.49}} \\
        Extremely Good & 5 & 3.57 & 11 & \cellcolor{green}{\textbf{7.86}} & 0 & 0.00 & 3 & \cellcolor{green}{\textbf{4.35}} \\
        \hline
        \textbf{Grand Total} & \textbf{140} & \textbf{100.00} & \textbf{140} & \textbf{100.00} & \textbf{69} & \textbf{100.00} & \textbf{69} & \textbf{100.00} \\
       \hline
    \end{tabular}

    \caption{Effectiveness of the \textbf{iterative Evaluator$\rightarrow$Optimizer loop} judged using \textbf{one-shot evaluator} broken down by number of hops (2-hops and 3-plus-hops) on test data. Higher values between \textbf{Pre-Loop} and \textbf{Post-Loop} in the \textbf{Good, Very Good, Extremely Good} buckets are highlighted in green for each group.}
    \label{tab:feedback_loop_performance_twohops_threeplushops}
    
\end{table*}

\begin{table*}[ht]
    \centering
    \small
    \setlength{\tabcolsep}{6pt} 
    
    \begin{tabular}{|l|c|c|c|c|c|}
        \hline
        \multirow{4}{*}{\textbf{Metric}}
        & \multicolumn{4}{c|}{\textbf{Claude-Sonnet-4}}
        & \multicolumn{1}{c|}{\textbf{GPT-5}} \\
        \cline{2-6}
        & \multicolumn{2}{c|}{\textbf{Reference-Free}} 
        & \multicolumn{2}{c|}{\textbf{Reference-Based}}
        & \multicolumn{1}{c|}{\textbf{Reference-Based}} \\
        \cline{2-6}
        & \multicolumn{1}{c|}{\textbf{Single}}
        & \multicolumn{1}{c|}{\textbf{Deconstructed}}
        & \multicolumn{1}{c|}{\textbf{Single}}
        & \multicolumn{1}{c|}{\textbf{Deconstructed}}
        & \multicolumn{1}{c|}{\textbf{Deconstructed}} \\
        \cline{2-6}
        & \% & \% & \% & \% & \% \\
        \hline
        Dependency & 62.74 & 96.00 & 67.81 & \textcolor{blue}{\textbf{97.3}} & 93.38 \\
        Format & 39.21 & 93.46 & 53.85 & 94.02 & \textcolor{blue} {\textbf{99.23}} \\
        Step Executability & 58.36 & 72.55 & 60.24 & \textcolor{red}{\textbf{79.43}} & 65.70 \\
        Query Adherence & 39.84 & 75.82 & 47.52 & 85.42 & \textcolor{magenta}{\textbf{87.38}} \\
        Redundancy & 53.00 & 85.63 & 56.33 & 88.85 & \textcolor{blue}{\textbf{90.82}} \\
        Tool Prompt Alignment & 49.02 & 65.36 & 52.94 & \textcolor{magenta}{\textbf{81.29}} & 77.37 \\
        Tool Usage Completeness & 11.76 & 60.13 & 42.79 & \textcolor{blue}{\textbf{94.12}} & \textcolor{blue}{\textbf{94.12}} \\
        \hline
    \end{tabular}

    \caption{Evaluation of \textbf{Metric-Wise Evaluator}: Triplet Ranking Agreement (Relaxed) Between \textbf{Metric-Wise Evaluator} and \textbf{Human Annotations} on the Validation Set (N=80) for \textbf{Claude-Sonnet-4} and \textbf{GPT-5}: Percent of Correct Inequalities per Metric. The best-performing setting is highlighted for every metric. The top-performing setting is highlighted for each metric: values $\geq 90\%$ in blue, $\geq 80\%$ in magenta, and $< 80\%$ in red.}
    \label{tab:validation_result_metric_wise_eval}
    
\end{table*}

\begin{table*}[t]
\centering
\small
\begin{tabular}{|l|c|c|c|c|c|c|}
\hline
\textbf{Tag} & \thead{\textbf{\# Annotator} \\ \textbf{Assigned}} & \thead{\textbf{\# Evaluator} \\ \textbf{Assigned}} & \thead{\textbf{\# Overlap}} & \thead{\textbf{Precision}} & \thead{\textbf{Recall}} & \thead{\textbf{F1-Score}} \\ \hline
Extremely Bad & 9 & 9 & 9 & 1.00 & 1.00 & \textcolor{blue}{\textbf{1.00}} \\ \hline
Very Bad & 7 & 6 & 6 & 1.00 & 0.86 & \textcolor{blue}{\textbf{0.92}} \\ \hline
Bad & 16 & 15 & 14 & 0.93 & 0.88 & \textcolor{blue}{\textbf{0.90}} \\ \hline
Acceptable & 19 & 17 & 16 & 0.94 & 0.84 & \textcolor{magenta}{\textbf{0.89}} \\ \hline
Good & 9 & 9 & 9 & 1.00 & 1.00 & \textcolor{blue}{\textbf{1.00}} \\ \hline
Very Good & 12 & 14 & 11 & 0.79 & 0.92 & \textcolor{magenta}{\textbf{0.85}} \\ \hline
Extremely Good & 8 & 10 & 8 & 0.80 & 1.00 & \textcolor{magenta}{\textbf{0.89}} \\ \hline
\midrule
\textbf{Overall (Macro Average)} & 80 & 80 & 73 & 0.92 & 0.93 & \textcolor{blue}{\textbf{0.92}} \\ \hline
\end{tabular}
\caption{Evaluation of \textbf{Claude-Sonnet-4} as \textbf{One-Shot Evaluator}: Agreement in Assigning Tag Between \textbf{One-Shot Evaluator} and \textbf{Human Annotations} on the Validation Set (N=80). F1-Scores $\geq 90\%$ in blue, and $\geq 80\%$ in magenta.}
\label{tab:validation_result_one_shot_eval}
\end{table*}

\begin{table*}[t]
\centering
\small
\begin{tabular}{|l|c|c|c|c|c|c|}
\hline
\textbf{Tag} & \thead{\textbf{\# Annotator} \\ \textbf{Assigned}} & \thead{\textbf{\# Evaluator} \\ \textbf{Assigned}} & \thead{\textbf{\# Overlap}} & \thead{\textbf{Precision}} & \thead{\textbf{Recall}} & \thead{\textbf{F1-Score}} \\ \hline
Extremely Bad & 9 & 11 & 9 & 0.82 & 1.00 & \textcolor{blue}{\textbf{0.90}} \\ \hline
Very Bad & 7 & 8 & 5 & 0.62 & 0.71 & 0.67 \\ \hline
Bad & 16 & 15 & 15 & 1.00 & 0.94 & \textcolor{blue}{\textbf{0.97}} \\ \hline
Acceptable & 19 & 16 & 16 & 1.00 & 0.84 & \textcolor{blue}{\textbf{0.91}} \\ \hline
Good & 9 & 8 & 8 & 1.00 & 0.89 & \textcolor{blue}{\textbf{0.94}} \\ \hline
Very Good & 12 & 13 & 12 & 0.92 & 1.00 & \textcolor{blue}{\textbf{0.96}} \\ \hline
Extremely Good & 8 & 9 & 7 & 0.78 & 0.88 & \textcolor{magenta}{\textbf{0.82}} \\ \hline
\midrule
\textbf{Overall (Macro Average)} & 80 & 80 & 72 & 0.88 & 0.89 & \textcolor{magenta}{\textbf{0.88}} \\ \hline
\end{tabular}
\caption{Evaluation of \textbf{GPT-5} as \textbf{One-Shot Evaluator}: Agreement in Assigning Tag Between \textbf{One-Shot Evaluator} and \textbf{Human Annotations} on the Validation Set (N=80). F1-Scores $\geq 90\%$ in blue, and $\geq 80\%$ in magenta.}
\label{tab:oneshot-validation-gpt5}
\end{table*}

\begin{table*}[t]
\centering
\small
\begin{tabular}{lrr}
\toprule
\textbf{Metric} & $\rho$ & $p-value$ \\
\midrule
Total metric-wise score        & 0.60 & 0.023 \\
Oneshot overall quality        & 0.52 & 0.059 \\
Dependency                     & 0.84 & 0.0002 \\
Format                         & 0.58 & 0.029 \\
Redundancy                     & 0.55 & 0.043 \\
Step Executability             & 0.54 & 0.047 \\
Tool Usage Completeness        & \textcolor{red}{\textbf{0.35}} & \textcolor{red}{\textbf{0.217}} \\
Query Adherence                & 0.79 & 0.0007 \\
Tool–Prompt Alignment          & 0.72 & 0.0035 \\
\bottomrule
\end{tabular}
\caption{Spearman rank correlations between LLM rankings under \textbf{Claude-Sonnet-4} and \textbf{GPT-5} as Judges on a 50-query test subset (14 planners, prompt type = \textit{Without Lineage}). $\rho$ < 0.5, and $p-value$ > 0.1 are highlighted in red.}
\label{tab:judge-corr}
\end{table*}

\begin{table*}[t]
\centering
\small
\begin{tabular}{|l|c|c|c|c|c|c|}
\hline
\textbf{Tag} & \thead{\textbf{\# Annotator} \\ \textbf{Assigned}} & \thead{\textbf{\# Evaluator} \\ \textbf{Assigned}} & \thead{\textbf{\# Overlap}} & \thead{\textbf{Precision}} & \thead{\textbf{Recall}} & \thead{\textbf{F1-Score}} \\ \hline
No Change & 66 & 70 & 51 & 0.77 & 0.73 & \textcolor{red}{\textbf{0.75}} \\ \hline
Incorrect Tool & 66 & 62 & 55 & 0.83 & 0.89 & \textcolor{magenta}{\textbf{0.86}} \\ \hline
Incorrect Prompt & 68 & 66 & 61 & 0.9 & 0.92 & \textcolor{blue}{\textbf{0.91}} \\ \hline
Complex Prompt & 70 & 63 & 53 & 0.76 & 0.84 & \textcolor{magenta}{\textbf{0.8}} \\ \hline
Repeated Detail & 74 & 71 & 66 & 0.89 & 0.93 & \textcolor{blue}{\textbf{0.91}} \\ \hline
Multi-tool Prompt & 56 & 68 & 50 & 0.89 & 0.74 & \textcolor{magenta}{\textbf{0.81}} \\ \hline
\midrule
\textbf{Overall (Macro Average)} & 400 & 400 & 336 & 0.84 & 0.84 & \textcolor{magenta}{\textbf{0.84}} \\ \hline
\end{tabular}
\caption{Evaluation of \textbf{Step-Wise Evaluator}: Agreement in Assigning Tag Between \textbf{Step-Wise Evaluator} and \textbf{Human Annotations} on the Validation Set (N=400). F1-Scores $\geq 90\%$ in blue, $\geq 80\%$ in magenta and $< 80\%$ in red.}
\label{tab:validation_result_step_wise_eval}
\end{table*}

\begin{table*}[t]
\centering
\small
\begin{tabular}{|l|c|c|}
\hline
\textbf{Tag} & \thead{\textbf{\#}} & \thead{\textbf{\%}} \\ \hline
Extremely Bad & 2 & 1.25 \\ \hline
Very Bad & 8 & 5 \\ \hline
Bad & 14 & 8.75 \\ \hline
Acceptable & 17 & 10.63 \\ \hline
Good & 35 & \textcolor{blue}{\textbf{21.88}} \\ \hline
Very Good & 39 & \textcolor{blue}{\textbf{24.38}} \\ \hline
Extremely Good & 45 & \textcolor{blue}{\textbf{28.13}} \\ \hline
\midrule
\textbf{Total} & \textbf{160} & \textbf{100} \\ \hline
\end{tabular}
\caption{Evaluation of \textbf{Plan Optimizer}: Similarity Between Plans Generated By \textbf{Plan Optimizer} and \textbf{Human Annotations} (considered as reference) on the Validation Set (N=160) Computed Based on \textbf{One-Shot Evaluator}. The proportion of plans generated by the Plan Optimizer falling in the top 3 buckets - \{\textbf{Extremely Good, Very Good, Good}\} are highlighted in blue.}
\label{tab:validation_result_plan_optimizer}
\end{table*}

\begin{table*}[t]
\centering
\small  
\begin{tabularx}{\textwidth}{|>{\centering\arraybackslash}p{0.55in}|
                           >{\centering\arraybackslash}X|}
\hline
\textbf{Module} & \textbf{Prompt} \\
\hline \hline
\textbf{Query Generation} & 
\ttfamily \textbf{SYSTEM\_PROMPT:} You are an expert query generator. The queries you generate should mimick real-life user queries related to a call center with agents. The queries should aim to gain insight into the call center's functioning and performance. Here are some reference examples. Generate 10 new queries. \{QUERY\_REFERENCE\_EXAMPLES\}
\\ \hline

\textbf{Plan Generation (Without Lineage)} & 
\ttfamily \textbf{SYSTEM\_PROMPT:} You are a high-level planner agent designed to orchestrate step-by-step solutions to business queries related to call center operations.
Your role is to:
- Analyze the given query and break it down into a coherent, **step-by-step plan**.
- Ensure that **all steps rely solely on the tools provided below** 
— do not assume access to any external data, tools, or knowledge beyond what is defined.
- Each step of the output plan should be simple and executable by specialized tool-aware agents.
You have access to the following tools:
\{TOOL\_DESCRIPTIONS\}

\textbf{MAIN\_PROMPT:}
Task:
1. You will be provided with a query to answer.
2. Generate a high-level, step-by-step plan to solve it.
3. Use only the tools available to you (described above).
4. Design the steps so that the plan is logically valid, realistically achievable, and contextually appropriate.
5. Clearly define dependencies between steps where intermediate results are needed.
6. Output each plan in the following format:
\textless{}example number\textgreater{}:
Query: \textless{}provided query\textgreater{}
Plan:
\{
    \textless{}step number\textgreater{}: \{"step": \textless{}tool name\textgreater{}("\textless{}input instruction to the tool\textgreater{}"), "depends\_on": [\textless{}step numbers that this step depends on\textgreater{}]\}
    ....
\}
Guidelines:
- Be structured: respect execution order and dependencies
- For each step k, the depends\_on list may only reference step indices in {1, …, k-1}, and the overall dependency graph must remain acyclic (no cycles).
- Be concise: focus each step on a specific actionable instruction
- Return ONLY the plan
Reference:
Below are some real examples of business queries and their corresponding multi-step plans along with thought process behind the plans:
\{EXAMPLE\_QUERY\_PLANS\_EXPLANATIONS\}
The input query to answer is: \{QUERY\}
\\ \hline

\textbf{Plan Generation (With Lineage)} & 
\ttfamily \textbf{SYSTEM\_PROMPT: }You are a high-level planner agent designed to orchestrate step-by-step solutions to business queries related to call center operations.

Your role is to:
- Analyze the given query and break it down into a coherent, **step-by-step plan**.
- Ensure that **all steps rely solely on the tools provided below** do not assume access to any external data, tools, or knowledge beyond what is defined.
- Each step of the output plan should be simple and executable by specialized tool-aware agents.

You have access to the following tools:
\{TOOL\_DESCRIPTIONS\}

\textbf{MAIN\_PROMPT: }Task:
1. You will be provided with a query to answer.
2. Generate a high-level, step-by-step plan to solve it.
3. Use only the tools available to you (described above).
4. Design the steps so that the plan is logically valid, realistically achievable, and contextually appropriate.
5. Clearly define dependencies between steps where intermediate results are needed.
6. Output each plan in the following format:
\textless{}example number\textgreater{}:
Query: \textless{}provided query\textgreater{}
Plan:
\{
    \textless{}step number\textgreater{}: \{"step": \textless{}tool name\textgreater{}("\textless{}input instruction to the tool\textgreater{}"), "depends\_on": [\textless{}step numbers that this step depends on\textgreater{}]\}
    ....
\}
Guidelines:
- Be structured: respect execution order and dependencies
- For each step k, the depends\_on list may only reference step indices in {1, …, k-1}, and the overall dependency graph must remain acyclic (no cycles).
- Be concise: focus each step on a specific actionable instruction
- Return ONLY the plan
Reference:
Below are examples of business queries paired with sequences of plans. Each sequence begins with a suboptimal plan and gradually improves through multiple iterations until reaching the best possible plan for the query. At each step, the plan is refined to better align with tool usage, reduce unnecessary complexity, and improve overall clarity. Use these plan lineages to identify common mistakes made by planners and to learn strategies for avoiding them. Given a query, generate an initial plan, then iteratively refine it, and return the full sequence of improved plans.
\{EXAMPLE\_QUERY\_PLANS\_LINEAGES\}
The input query to answer is: \{QUERY\}

\\ \hline
\end{tabularx}
\caption{Prompts Used For \textbf{Generation} of \textbf{Queries} and \textbf{One-Shot Plans}}
\label{tab:prompts_for_query_and_plan_gen}
\end{table*}

\begin{table*}[t]
\centering
\small  
\begin{tabularx}{\textwidth}{|
                           >{\centering\arraybackslash}X|}
\hline
\textbf{Prompt} \\
\hline \hline

\ttfamily
You are a **Step Evaluator Agent** responsible for carefully analyzing individual steps within a multi-step plan designed to answer a user query. Your mission is to evaluate each step against its dependent inputs and explain if there are any issues with the step, using the **standardized error types and format**. You need to carefully determine whether the step has errors and justify your evaluation based on the **reference examples** and the provided **tool capabilities**.
\#\#\# **Your Task:**
For a given **step in a plan**, you will evaluate and output the results as follows:
1. Analyze the step and its **dependent steps** to identify potential issues under the following **error types**:
  - **INCORRECT PROMPT**
  - **COMPLEX PROMPT**
  - **INCORRECT TOOL**
  - **REPEATED DETAIL**
  - **MULTI-TOOL PROMPT**
  - **NO CHANGE**
2. Follow this **order of judgment** for your evaluations:
  - Evaluate if the step has any of the following critical issues: **INCORRECT PROMPT**, **COMPLEX PROMPT**, or **INCORRECT TOOL**, comparing with the dependent step(s) if necessary.
  - If no critical issues are identified, assess the step for **REPEATED DETAILS** by comparing it to its dependent steps.
  - Check if the task in the step could also be performed by a different tool, and if it can then flag it as **MULTI-TOOL PROMPT**.
  - If none of the above issues are detected, then and only then, return **NO CHANGE**.
3. When identifying an **error**, include both:
  - A **clear, concise reasoning** for why that error applies.
  - The **appropriate error tag** corresponding to the error type.
4. **Output Format**:
  Use the following standardized output format for any identified issues:
  "
  \textless{}STEP NUMBER\textgreater{}
  1. \textless{}REASONING BEHIND THE TYPE OF ERROR\textgreater{}: \textless{}ERROR TYPE\textgreater{}
  2. \textless{}REASONING BEHIND THE TYPE OF ERROR\textgreater{}: \textless{}ERROR TYPE\textgreater{}
  ...
  "
5. Structure Of The Steps:
  \textless{}step\_number\textgreater{}: \{\{"step": \textless{}Tool\_Name\textgreater{}((placeholders), "\textless{}prompt\textgreater{}"), "depends\_on": [dependent\_steps]\}\}      (when tool is RAG or T2S, placeholder is outside the prompt)
  \textless{}step\_number\textgreater{}: \{\{"step": \textless{}Tool\_Name\textgreater{}("\textless{}prompt\textgreater{}"), "depends\_on": [\textless{}dependent\_steps\textgreater{}]\}\}              (when tool is LLM, placeholder is inside the prompt)
\#\#\# **When Evaluating a Point, Some Key Points to Consider:**
- **Step Structure**: Ensure the step matches the expected JSON format and adheres to placeholders (e.g., make sure any referenced dependent steps are properly reflected both in the `depends\_on` array and the tool’s prompt using placeholders like `(1)` or `(2)`).
- **Dependent Steps**: Evaluate the dependent steps in relation to the current step. Any details already covered in the dependent steps should not be redundantly included in the current step's prompt.
- **Tool Assignment**: Each step must be handled by the appropriate tool. Use the **tool capabilities** section to determine whether the assigned tool is suitable for the task. A mismatch in tool functionality would lead to an **INCORRECT TOOL** error.
\#\#\# **Error Type Definitions:**
Here is the list of error types that you must use in your evaluations, along with their specific definitions:
1. **INCORRECT PROMPT**:
  - This error applies when the prompt given to the tool is incomplete or misaligned with the tool's capabilities. This issue can lead to incorrect or incomplete outputs.
  - For example, if the query asks for calls, and the prompt to T2S extracts interaction\_ids of all interactions, then that is an incorrect prompt that needs to be changed. If the placeholders are missing from an LLM step, that is an incorrect prompt and needs to be changed to include the placeholders.
2. **COMPLEX PROMPT**:
  - This error applies when the prompt is too complex to be executed by the tool in a single step. 
  - Such prompts need to be broken down into simpler sub-steps.
  - This can happend when T2S is being asked to extract more than 1 thing at a time or when RAG is being asked to answer multiple questions at a time. If it can be sensibly broken down into simpler steps, it should be.
3. **INCORRECT TOOL**:
  - This error applies when the wrong tool is assigned to a particular task.
  - Look carefully at the tool description provided below, it explains in detail what kind of information can be answered by T2S and which can be answered by RAG.
4. **REPEATED DETAIL**:
  - This error applies if the step contains unnecessary redundancies, such as duplicating details already handled by a dependent step.
5. **MULTI-TOOL PROMPT**:
  - This error applies when the task assigned to a particular tool can also be accomplished with another available tool as well (e.g., if both T2S and RAG could perform the task).
6. **NO CHANGE**:
  - This result is valid when the step is completely correct, and none of the above points apply to it.
---
\#\#\# **Evaluation Order:**
When evaluating, follow this strict judgment sequence:
1. Look for **critical issues** in the step:
  - Does the step contain an **INCORRECT PROMPT**, a **COMPLEX PROMPT**, or an **INCORRECT TOOL** error?
2. Next, assess if there are **REPEATED DETAILS** by comparing the step to its dependent steps.
3. Check if the task described in the step can also be performed by another tool. If so, flag it as a **MULTI-TOOL PROMPT**.
4. If none of the above apply, conclude your evaluation with **NO CHANGE**.
---
\#\#\# **Tools Available:**
\{Snowflake\_tools\_descriptions.TOOL\_DESCRIPTIONS\_ULTRASHORT\}
---
\#\#\# **Reference Examples:**
\{Evaluator\_creator.REFERENCE\_EXAMPLES\}
***IMPORTANT REMINDER***
- Order of dependencies and placeholders does not matter. Dependencies are there to know which steps this step is dependent on. Placeholders in a prompt so that when that prompt is finally given to the respective tool, we know where to put the output of the dependent step into the prompt.
\\ \hline

\end{tabularx}
\caption{System Prompt Used For \textbf{Step-Wise Evaluator.}}
\label{tab:prompts_for_step_wise_eval}
\end{table*}

\begin{table*}[t]
\centering
\small
\begin{tabularx}{\textwidth}{|
                           >{\centering\arraybackslash}X|}
\hline
\textbf{Prompt} \\
\hline \hline

\ttfamily \#\#\#\# **ROLE**:  
You are an **Expert Plan Optimizer** tasked with refining and updating multi-step plans designed to answer complex queries related to call centers.
You will be provided with a plan, along with the output of a step-evaluator. This output contains a detailed analysis of a particular step of the plan.
Based on the output of the step-evaluator, you will make changes to the plan and return the updated plan.
---
\#\#\# **GOAL:**
Produce a **revised multi-step plan** by:  
1. Analyzing the output of the stepwise plan evaluator.
2. Based on the output of the step-evaluator, make changes to that particular step (*CHANGE 0*).
3. If necessary, make changes to the latter steps of the plan in order to maintain overall logical coherence (*CHANGE 1*).
---
\#\#\# **INPUTS:**
1. **Original Query:** - The query from the user that the plan is attempting to answer.  
2. **Original Plan:** - A JSON-formatted textual representation of the original multi-step plan.  
  - Each step of the plan should strictly adhere to the following format:  
      \textless{}step\_number\textgreater{}: \{\{"step": \textless{}tool\_name\textgreater{}("\textless{}prompt\textgreater{}"), "depends\_on": [\textless{}dependent\_step\_numbers\textgreater{}]\}\}
  **Key Details About Step Format:**
  - It must be JSON parsable, ensuring that details like double quotes, braces, commas, etc., are accurate.  
  - Referenced steps by number (e.g., "depends\_on": [1, 2]) must be reflected as placeholders (e.g., "(1)", "(2)") in the corresponding `\textless{}prompt\textgreater{}`, especially in LLM steps.  
3. **Step Evaluator Output:** - Feedback related to a specific step in the plan.  
  - Includes:  
  - **Step Number:** The step under scrutiny.  
  - **Feedback:** Specific insights into potential issues such as incorrect prompt/tool, overly complex prompt, repeated details, etc.  
  - **Error Type:** The categorization of the issue (e.g., INCORRECT PROMPT, INCORRECT TOOL, COMPLEX PROMPT, etc.).  
---
\#\#\# **EXPECTED OUTPUT:** 1. **CHANGE 0:** - Update the problematic step directly based on the output of the step evaluator. Changes may include:  
  - Prompt refinement.
  - Prompt breakdown resulting in the creation of new steps.
  - Tool reassignment.   
  - Correction of repeated details.
  - Addition of new step based on "multi-tool prompt" output, if it is not already present.
  - Clearly and concisely explain what you changed and why.  
2. **CHANGE 1:** - Address how the rest of the plan may need adjustments to ensure logical coherence based on **CHANGE 0**.
  - Must **NOT** include any plan optimization or improvement, only corrections to ensure that previous change does not break the entire plan.
  - Changes may include:
      - Updating indices of steps based on addition or removal of steps.
      - Updating dependencies/placeholders for downstream steps.
      - Adding additional steps for aggregation or reasoning.
      - Removing now-redundant or irrelevant steps.
  - Provide a detailed reasoning for any adjustments made.
3. **Revised Plan:** - A JSON-formatted version of the new, corrected plan.  
  - Ensure correct syntax, adherence to format, and logical coherence throughout the revised plan.  
---
\#\#\# **ERROR TYPES AND HOW TO HANDLE THEM:** Here is a comprehensive list of potential issues the plan optimizer might encounter and the strategies to address them.
\{STEP\_OPTIMIZER\_ERROR\_CATEGORIES\_DESCRIBED\}
---
\#\#\# **FLOW:** Explain your analysis and actions in a two-stage process:
\#\#\#\# **CHANGE 0:**
- This stage addresses the feedback provided by the step evaluator for the particular step under consideration.
- Describe the modifications made to this step and the rationale behind them explicitly.
\#\#\#\# **CHANGE 1:** - This stage ensures coherence of the plan after implementing CHANGE 0.
- Adjust dependencies, add/delete/update steps as necessary to ensure the plan works seamlessly after the initial change.
- Your job is not to perform optimizations on the plan to improve it. It is simply to make any changes necessary to ensure that the rest of the plan isn't wrong because of CHANGE 0.
- An example of what you SHOULD do: If a step has been removed and that step was referenced in another step's dependency, then your job is to correct all the step indices based on the removed step, and then to update the future step's dependencies and placeholders to account for this change.
- An example of what you SHOULD NOT do: If a step was detected to have repeated detail and CHANGE 1 removed that repeated detail, then it is NOT your job to find similar repeated details in future steps and change them as well.
- This is a VERY important distinction and should be remembered at all times.
---
\#\#\# **FORMAT OF INPUT AND OUTPUT:** INPUT FORMAT:
- Query: \textless{}query\textgreater{}
- Plan:
\textless{}plan\textgreater{}
- Step Evaluator Output:
\textless{}step\_evaluator\_output\textgreater{}
OUTPUT FORMAT:
"CHANGE 0": \textless{}explanation of direct changes based on the step evaluator output\textgreater{}
"CHANGE 1": \textless{}explanation of changes to the rest of the plan to ensure logical coherence\textgreater{}
"NEW PLAN STARTS":  
\textless{}revised plan\textgreater{}
---
\#\#\# **TOOLS:** \{Snowflake\_tools\_descriptions.TOOL\_DESCRIPTIONS\_ULTRASHORT\}
---
\#\#\# **REMINDER:** - Follow the correct step format strictly.  
- Ensure all syntax, placeholders, and dependencies are accurate.  
- Changes must make the plan more logically connected and operationally sound with the tools provided.
- Make sure that CHANGE 1 leads to correct and sequential indices in the new plan.
\\ \hline

\end{tabularx}
\caption{System Prompt Used For \textbf{Plan Optimizer}}
\label{tab:prompts_for_plan_optimizer}
\end{table*}

\begin{table*}[t]
\centering
\small  
\begin{tabularx}{\textwidth}{|>{\raggedright\arraybackslash}p{0.6in}|
                           >{\raggedright\arraybackslash}X|}

\hline
\textbf{Module} & \textbf{Prompt} \\
\hline \hline

\textbf{Dependency (Metric-wise Evaluator)} & 
\ttfamily **You are an expert plan evaluation agent**
You will be given a plan and your task is to meticulously evaluate its dependencies and assign it a score.

* Here is an example of a plan:
\{
    1: \{"step": T2S([], "Fetch call\_ids of calls related to payment adjustment inquiries."), "depends\_on": []\},
    2: \{"step": RAG((1), "What patterns of customer sentiment emerge in these calls?"), "depends\_on": [1]\}
\}

Here are the things to keep in mind:
* **The only part of this plan that is relevant for you is the list followed by "depends\_on" in each step. That list contains the steps that the current step depends on.
* ** Your job is to find if for any step, that list of dependent steps is wrong which can only be in 1 of 2 cases. Either there is some dependent step that is missing from the "depends\_on" list or there is some step in the "depends\_on" list that shouldn't be there.
* **Return the total steps - total number of steps with incorrect dependencies.**

Input Format:
You will be simply provided with a plan.

Output Format:
You must provide a detailed explanation of your analysis related to that point, and then give the final score as per your verdict. Follow the below provided format:

\textless{}explanation\textgreater{} | \textless{}score\textgreater{} |

Here are some reference examples:
\{DEPENDENCY\_REFERENCE\_EXAMPLES\}
\\ \hline

\textbf{Format (Metric-wise Evaluator)} & 
\ttfamily **You are an expert plan evaluation agent**
You will be given a plan and your task is to meticulously evaluate its format and assign it a score.

* Here is the correct format:
Plan:
\{\{
    \textless{}step number\textgreater{}: \{\{"step/query": \textless{}tool name\textgreater{}("\textless{}input instruction to the tool\textgreater{}"), "depends\_on": [\textless{}step numbers that this step depends on\textgreater{}]\}\}
    ....
\}\}

Here are the things to keep in mind:
* **Return the total steps - total number of format violations.**
* **Make sure that the plan is json parsable. If that fails due to thing such as quadruple quotes instead of double quotes, mismatched paranthesis or some similar reason, then that is a violation.**
* **Make sure that every step, that has a dependency on another step, has a valid placeholder. The placeholder is the dependent step number, surrounded by paranthesis, within the prompt. For example in step 2 of the above correct plan, "RAG((1), "Wha...)", here "(1)" is the placeholder. Similarly a placeholder MUST exist in every step with a dependency, ESPECIALLY in the steps with LLM. Each step with an absent placholder or an incorrect placeholder (e.g. (1, 2) instead of (1) and (2) etc.) is a format violation.**
* **If a step, for e.g an LLM step has a mention of the query, for e.g. "Answer the question/query based on ...", then it MUST have a placeholder for the query (e.g. (query)) otherwise that is a format violation.**

Input Format:
You will be simply provided with a plan.

Output Format:
You must provide a detailed explanation of your analysis related to that point, and then give the final score as per your verdict. Follow the below provided format:

\textless{}explaination\textgreater{} | \textless{}score\textgreater{} |

Here are some reference examples:
\{FORMAT\_REFERENCE\_EXAMPLES\}
\\ \hline

\textbf{Redundancy (Metric-wise Evaluator)} & 
\ttfamily 
**You are an expert plan evaluation agent**
You will be given a plan and your task is to meticulously evaluate its redundancy and assign it a score.

* Here is the format of a plan:
Plan:
\{\{
    \textless{}step number\textgreater{}: \{\{"step/query": \textless{}tool name\textgreater{}("\textless{}input instruction/prompt to the tool\textgreater{}"), "depends\_on": [\textless{}step numbers that this step depends on\textgreater{}]\}\}
    ....
\}\}

There are only two types of redundancies that you need to look for:

1. **T2S Step Redundancy:**
- In a T2S step, if the instruction/prompt provided contains some information that is already accounted for in its dependent step, then that is a redundancy.
- The reason being that for T2S i.e. Snowflake structured database specifically, any additional detail in the prompt can lead to additional table joins or other things that are not needed.
- Due to this reason, this step redundancy ONLY APPLIES TO T2S AND NOT TO RAG OR LLM STEPS.
- Phrases like "in these calls" or "in these interactions" in a T2S step is NOT considered to be redundant. 
- Example for this will be provided in the reference examples.

2. **Ignored Step Redundancy:**
- The final step of a plan is the overall output of the plan.
- If there is a step or steps that do not contribute to the final step, then those steps are redundant.
- This can be found by looking at the "depends\_on" list of each step.
- Example for this will be provided in the reference examples.

Return the total steps - total number of redundancies/redundant steps.

Input Format:
You will be provided with a plan.

Output Format:
You must provide a detailed explanation of your analysis related to that point, and then give the final score as per your verdict. Follow the below provided format:

\textless{}explanation\textgreater{} | \textless{}score\textgreater{} |

Here are some reference examples:
\{REDUNDANCY\_REFERENCE\_EXAMPLES\}
\\ \hline

\end{tabularx}
\caption{System Prompts Used in \textbf{Metric-Wise Evaluator} (Dependency, Format and Redundancy)}
\label{tab:prompts_for_metric_wise_eval_1}
\end{table*}

\begin{table*}[t]
\centering
\small  
\begin{tabularx}{\textwidth}{|>{\raggedright\arraybackslash}p{0.6in}|
                           >{\raggedright\arraybackslash}X|}

\hline
\textbf{Module} & \textbf{Prompt} \\
\hline \hline

\textbf{Step-Executability (Metric-wise Evaluator)} & 
\ttfamily **You are an expert plan evaluation agent**
You will be given a plan along with the best possible plan and your task is to meticulously evaluate its step-executability and assign it a score.

* Here is the format of a plan:
Plan:
\{
    \textless{}step number\textgreater{}: \{"step/query": \textless{}tool name\textgreater{}("\textless{}input instruction/prompt to the tool\textgreater{}"), "depends\_on": [\textless{}step numbers that this step depends on\textgreater{}]\}
    ....
\}

Here are the things to keep in mind:
* The only thing that you need to look at is the tool and the corresponding instruction/prompt in each step.
* Assume that everything that the prompt is asking from the tool is perfectly within the tool's capabilities.
* Under this assumption, judge whether the prompt is too complex to be executed by the tool in 1 go and needs to be broken down into multiple steps.
* There are 2 tools that you need to check, RAG (which performs Retrieval-Augmented Generation on some unstructured data) and T2S (which performs Text-to-SQL on some structured data).

* For RAG: The prompt should only ask the tool to do one thing at a time. Examples provided in the reference examples.

* For T2S: For a T2S step, the number of trigger words/trigger phrases present in the prompt is directly proportional to the complexity of the step. Here is a list of trigger words to look for. Don't look exactly for these trigger words, but look for similar looking phrases as well:
** For call drivers **
\{call\_drivers\_op\}

** For moments **
\{moments\_op\}

** For qa **
\{qa\_op\}

Based on extracted data, T2S has the ability to perform simple mathematical operations such as finding mean etc.
If a T2S step contains more than 3 trigger words/phrases, then it is not executable in 1 go.
REMEMBER: "fetch interaction\_id" or "fetch call\_id" is NOT considered as a trigger phrase

Return the total steps - total number of steps that are not executable in 1 go.

Input Format:
You will be provided with 2 plans, the plan in question and the best possible plan.

Output Format:
You must provide a detailed explanation of your analysis related to that point, and then give the final score as per your verdict. Follow the below provided format:

\textless{}explanation\textgreater{} | \textless{}score\textgreater{} |

Here are some reference examples:
\{STEP\_ATOMICITY\_REFERENCE\_EXAMPLES\}
\\ \hline

\textbf{Query Adherence (Metric-wise Evaluator)} & 
\ttfamily 
**You are an expert plan evaluation agent**
You will be given a query, the corresponding plan and the best possible plan, and your task is to meticulously evaluate its query adherence and assign it a score.

* Here is the format of a plan:
Plan:
\{
    \textless{}step number\textgreater{}: \{"step/query": "TOOL"("\textless{}input instruction to the tool\textgreater{}"), "depends\_on": [\textless{}step numbers that this step depends on\textgreater{}]\}
    ....
\}

Here are the things to keep in mind:
* Do not worry about things like which tool is used, format etc. Only look at the prompts/instructions given to the tools in each step and the step dependencies.
* Assume that in each step, the instructions that are given to the tool in the prompt, run perfectly as intended and give the ideal output as per the prompt.
* Under that assumption, judge whether the final output of the plan would answer the original query well or not.
* Here are some possible scores to assign:
* **0.0:** If the plan is completely irrelevant to the query.
* **0.5:** If the plan partially answers the query but has some gaps/flaws.
* **1.0:** If the plan perfectly and completely addresses all aspects of the query.

Additional Information:
- Placeholders: There are 2 types of placeholders:
    - Step placeholder (for e.g (1), (2)) which gets replaced by the output of the corresponding step during execution.
    - Query placeholder i.e. (query) which gets replaced by the original query during execution.
    - So wherever in the plan you see something like (1), (2), (query) etc, assume that is replaced by their corresponding above mentioned values.
- The terms "call\_ids" and "interaction\_ids" are interchangeable. What matters is what comes after it, i.e. "call\_ids of calls" or "interaction\_ids of calls" would both fetch the ids of CALLS. Similarly, "call\_ids of interactions" and "interaction\_ids of interactions" would both fetch the ids of interactions.
- In the final step, if the prompt itself does a good job at describing what is asked in the query, then it is fine even if it doesn't have the placeholder of the query.

Input Format:
You will be provided with a query and the corresponding plan.

Output Format:
You must provide a detailed explanation of your analysis related to that point, and then give the final score as per your verdict. Follow the below provided format:

\textless{}explaination\textgreater{} | \textless{}score\textgreater{} |

Here are some reference examples:
\{QUERY\_ADHERENCE\_EXAMPLES\}
\\ \hline

\end{tabularx}
\caption{System Prompts Used in \textbf{Metric-Wise Evaluator} (Step-Executability and Query-Adherence)}
\label{tab:prompts_for_metric_wise_eval_2}
\end{table*}

\begin{table*}[t]
\centering
\small  
\begin{tabularx}{\textwidth}{|>{\raggedright\arraybackslash}p{0.6in}|
                           >{\raggedright\arraybackslash}X|}

\hline
\textbf{Module} & \textbf{Prompt} \\
\hline \hline

\textbf{Tool-Usage Completeness (Metric-wise Evaluator)} & 
\ttfamily 
**You are an expert plan evaluation agent**
You will be given a plan along with the best possible plan and your task is to meticulously evaluate its tool-usage completeness and assign it a score.

* Here is the format of a plan:
Plan:
\{
    \textless{}step number\textgreater{}: \{"step/query": \textless{}tool name\textgreater{}("\textless{}input instruction/prompt to the tool\textgreater{}"), "depends\_on": [\textless{}step numbers that this step depends on\textgreater{}]\}
    ....
\}

Here are the tools at your disposal:
\{TOOL\_DESCRIPTIONS\}

\#\# TOOL DESCRIPTION ENDS HERE

Here are the things to keep in mind:
* Look at the best possible plan and see if it has two steps with same prompt but different tool. For e.g in the following plan
\{
    1: \{"step": T2S([], "Fetch call\_ids of calls related to payment adjustment inquiries."), "depends\_on": []\},
    2: \{"step": RAG((1), "What patterns of customer sentiment emerge in these calls?"), "depends\_on": [1]\},
    3: \{"step": T2S((1), "What patterns of customer sentiment emerge in these calls?"), "depends\_on": [1]\},
    4: \{"step": LLM("Synthesize and summarize your answer to the query (query) using outputs from the following: (sub-query 2) using (tool 2): (2), (sub-query 3) using (tool 3): (3)"), "depends\_on": [2, 3]\}
\}
* If this kind of a thing does not exist in the best possible plan:
    - directly return total steps - 0 = total steps.
* If this kind of thing IS present in the best possible plan, then:
    - If the exact same thing (same 2 steps with same prompt and different tools) is present in the plan to be evaluated:
        - that is NOT a tool-usage completeness violation.
    - If neither of the two steps is present in the plan to be evaluated:
        - that is NOT tool-usage completeness violation.
    - If only 1 of the 2 steps are present in the plan to be evaluated:
        - that is a tool-usage completeness violation.

Return the total steps - total number of tool-usage completeness violations.

Input Format:
You will be provided with 2 plans, the plan to be evaluated and the best possible plan.

Output Format:
You must provide a detailed explanation of your analysis related to that point, and then give the final score as per your verdict. Follow the below provided format:

\textless{}explaination\textgreater{} | \textless{}score\textgreater{} |

Here are some reference examples:
\{TOOL\_USAGE\_COMPLETENESS\_EXAMPLES\_WR\}
\\ \hline

\textbf{Tool-Prompt Alignment (Metric-wise Evaluator)} & 
\ttfamily 
**You are an expert plan evaluation agent**
You will be given a plan along with the best possible plan and your task is to meticulously evaluate its tool-prompt alignment and assign it a score.

* Here is the format of a plan:
Plan:
\{
    \textless{}step number\textgreater{}: \{"step/query": \textless{}tool name\textgreater{}("\textless{}input instruction/prompt to the tool\textgreater{}"), "depends\_on": [\textless{}step numbers that this step depends on\textgreater{}]\}
    ....
\}

Here are the tools at your disposal:
\{TOOL\_DESCRIPTIONS\}

\#\# TOOL DESCRIPTION ENDS HERE

Here are the things to keep in mind:
* For each step, look at the tool and the prompt.
* If the instruction/prompt in a step cannot be fulfilled/answered by the assigned tool, then that step has a tool-prompt alignment violation.
* If the prompt is complex and asks the tool to do multiple things at the same time, if even one of those things is outside the scope of the provided tool, then that step has a tool-prompt alignment violation.
* For example if the prompt asks how sentiment shifted during the call, or any other change during the call, that cannot be answered using T2S and can only be answered using RAG.

Return the total steps - total number of steps that have tool-prompt violation.

Input Format:
You will be provided with 2 plans, the plan to be evaluated and the best possible plan.

Output Format:
You must provide a detailed explanation of your analysis related to that point, and then give the final score as per your verdict. Follow the below provided format:

\textless{}explanation\textgreater{} | \textless{}score\textgreater{} |

Here are some reference examples:
\{TOOL\_PROMPT\_ALIGNMENT\_EXAMPLES\}
\\ \hline

\end{tabularx}
\caption{System Prompts Used in \textbf{Metric-Wise Evaluator} (Tool-Prompt Alignment and Tool-Usage Completeness)}
\label{tab:prompts_for_metric_wise_eval_3}
\end{table*}

\begin{table*}[t]
\centering
\small
\begin{tabularx}{\textwidth}{|
                           >{\raggedright\arraybackslash}X|}
\hline
\textbf{Prompt} \\
\hline \hline

\ttfamily 
You are a critical evaluator of structured, tool-augmented LLM plans.

Each plan is a sequence of steps designed to answer a user query using specialized tools. Your task is to compare a candidate plan (P) against the best possible reference plan (P*) and rate the candidate plan on a 7-point quality scale.
You must reason carefully across **7 evaluation dimensions** and assign a rating based on how close P is to P* in terms of quality, structure, and correctness.

\#\#\# PLAN FORMAT
Each plan is a dictionary where each key is a step number, and each value is an object with:
- `"step"` or `"query"`: containing a tool invocation of the form `TOOL\_NAME(placeholder\_inputs, prompt)`
- `"depends\_on"`: a list of step numbers this step depends on
\#\#\#\# Example Format:
\{\{
    1: \{\{"step": T2S([], "Fetch interaction\_ids of unresolved calls"), "depends\_on": []\}\},
    2: \{\{"step": RAG((1), "Fetch calls where the sentiment transitioned from negative to positive within the transcript"), "depends\_on": [1]\}\},
    3: \{\{"step": LLM("Fetch the interaction\_ids as a list from the `data insights` table in (2)."), "depends\_on": [2]\}\},
    4: \{\{"step": T2S((3), "Retrieve agent QA scores for resolution procedures in these calls."), "depends\_on": [3]\}\}
\}\}

\#\#\# PLACEHOLDER CONVENTIONS
Placeholders are used in prompts to represent inputs to tools. These must follow standard conventions:
- (query): The original user query
- (\textless{}step number\textgreater{}): Refers to the output of a previous step
- (tool \textless{}step number\textgreater{}): Refers to the tool used in that step (T2S, RAG, or LLM)
- (sub-query \textless{}step number\textgreater{}): Refers to the prompt/query used in that step

Placeholders must be well-formed and accurate to correctly wire dependencies between steps.

\#\#\# TOOL DESCRIPTIONS
\{TOOL\_DESCRIPTIONS\}

\#\#\# EVALUATION DIMENSIONS
Evaluate the candidate plan (P) against the gold plan (P*) using these criteria:
1. Precision
What percentage of steps in the plan (P) are present in the best possible plan (P*)?
2. Recall
What percentage of steps in the best possible plan (P*) are present in the plan (P)?
3. F1 Score
What percentage of steps in the best possible plan (P*) are present in the plan (P)?
4. Format Correctness
Is the plan valid JSON with correct keys ("step" or "query" and "depends\_on")?
(a) Is the overall plan valid JSON?
(b) What percentage of steps in the plan (P) have correct format - "step" or "query" and "depends\_on"
5. Dependencies
What percentage of steps in the plan (P) have correct dependencies?
6. Placeholders
What percentage of steps in the plan (P) have correct placeholders?

\#\#\# INSTRUCTIONS
- Be objective. Evaluate the candidate plan only in comparison to the gold plan.
- Think through each metric and generate a detailed rationale.
- Then assign a rating based on overall closeness on a 7-point scale:
\{\{Extremely Bad, Very Bad, Bad, Acceptable, Good, Very Good, Extremely Good\}\}

\#\#\# Points to consider while evaluating the plan:
- (a) If the plan (P) is not valid JSON, then it is automatically an "Extremely Bad" plan as it becomes difficult to evaluate the plan in terms of precision, recall, F1 score, format correctness, dependencies, and placeholders.
- (b) If the plan (P) is valid JSON, then we can evaluate the plan in terms of precision, recall, F1 score, format correctness, dependencies, and placeholders.
- (c) The individual metrics such as precision, recall, F1 score, format correctness, dependencies, and placeholders can be computed objectively. However, the final rating can be subjective but strictly based on the individual metrics which capture how close the candidate plan (P) is to the gold plan (P*).
- (d) The metrics precision, recall and F1 score are computed by comparing the candidate plan (P) with the gold plan (P*). In contrast, the metrics format correctness, dependencies, and placeholders are computed by simply checking the plan (P) against the required format, dependencies and placeholders.
- (e) While the final rating is subjective, you can follow the following guidelines provided the overall format of the plan is correct (valid json):
    - If f1\_score is > 95\%, then the rating is "Extremely Good"
    - If f1\_score is > 85\%, then the rating is "Very Good"
    - If f1\_score is > 75\%, then the rating is "Good"
    - If f1\_score is > 60\%, then the rating is "Acceptable"
    - If f1\_score is > 45\%, then the rating is "Bad"
    - If f1\_score is > 30\%, then the rating is "Very Bad"
    - If f1\_score is < 30\%, then the rating is "Extremely Bad"

\#\#\# Points to keep in mind while checking the match between P and P*:
- (a) The steps in P and P* should have the exact same semantics and interpretation. If multiple steps in P* are carried out in a single step in P, then the step in P is incorrect. Example: Steps a and b do not match.
    - Step a: T2S([], "Fetch interaction\_ids categorized by 'Did We Resolve' (DWR) status and their corresponding call drivers")
    - Step b: T2S([], "Fetch interaction\_ids of calls marked as 'resolved'.")
- (b) The steps in P and P* should have the same tool, format, placeholders and dependencies for them to be considered a match. Only if this condition is met, we can say that both the steps have the same semantics and interpretation.
- (c) The steps in P and P* need not exactly match in terms of the prompt. The requirement is that the prompts in both the steps in P and P* should be semantically similar and have the same intent.

\#\#\# OUTPUT FORMAT
\{\{
"rationale": \{\{
"precision": \textless{}rationale for precision\textgreater{},
"recall": \textless{}rationale for recall\textgreater{},
"f1\_score": \textless{}rationale for f1\_score\textgreater{},
"format\_correctness": \textless{}rationale for format\_correctness\textgreater{},
"dependencies": \textless{}rationale for dependencies\textgreater{},
"placeholders": \textless{}rationale for placeholders\textgreater{},
"rating": \textless{}rationale for rating\textgreater{}
\}\}
"score": \{\{
"precision": \textless{}score for precision\textgreater{},
"recall": \textless{}score for recall\textgreater{},
"f1\_score": \textless{}score for f1\_score\textgreater{},
"format\_correctness": \textless{}score for format\_correctness\textgreater{},
"dependencies": \textless{}score for dependencies\textgreater{},
"placeholders": \textless{}score for placeholders\textgreater{},
"rating": \textless{}score for rating\textgreater{}
\}\}
\}\}
Note: The output should be a valid JSON object.
\\ \hline

\end{tabularx}
\caption{System Prompt Used For \textbf{One-Shot Evaluator}}
\label{tab:prompts_for_oneshot_evaluator}
\end{table*}

\begin{table*}[t]
\centering
\small 
\begin{tabularx}{\textwidth}{|>{\raggedright\arraybackslash}X|}
\hline
\textbf{Prompt} \\
\hline \hline

\ttfamily 
You are an impartial Judge LLM. Evaluate three candidate responses (R1, R2, R3) to a given Query using the rubric below. Rely only on the provided text. Do not add external facts or assumptions.

\#\# Rubric (score each response independently)

(a) Validity (0/1):
0 = invalid (e.g., ``cannot be answered due to no data,'' system error messages, or refusal).
1 = valid (any other kind of substantive response).

(b) Consistency (0/1):
Is the response answering the given Query?

A response cannot be consistent if it is invalid.

Valid but off-topic $\implies$ consistency = 0.

Valid and on-topic $\implies$ consistency = 1.

(c) Completeness (1-7):
How fully does the response address the Query with relevant details?

Consider only details relevant to the Query.

1 = Extremely Bad, 7 = Extremely Good.

(d) Redundancy (1-7):
How free is the response from irrelevant or unrelated details?

7 = no fluff, to the point.

The more irrelevant content, the lower the score.

\#\#\# Gating rules

- Only valid responses can be consistent or complete.

- Only valid and consistent responses can be complete.

- For invalid responses, set: consistency = 0, completeness = 1, redundancy = 7.

\#\# Decision set (pick exactly one)

\{"Neither", "R1", "R2", "R3", "R1,R2", "R1,R3", "R2,R3", "All"\}

Use the following tie logic:

- Disqualify any response with validity = 0 or consistency = 0.

- If none remain, output "Neither".

- Among remaining, compare by (completeness, redundancy) in lexicographic order (higher is better).

- If all three tie on both metrics $\implies$ "All".

- If exactly two tie for best, then return that pair (e.g., "R1,R3").

- Otherwise return the single best (e.g., "R2").

\#\# Guidelines for scoring

- Do not focus too much on the format of the response being scored. As long as the response is valid and consistent, the format can be a simple list or dictionary and not necessary textual.
    - Example:
        - Query: ``What percentage of calls required a supervisor escalation, and what is the average QA score for such calls over the last year?''
        - Response 1: \{'pct\_escalated\_calls': Decimal('12.326849'), 'avg\_escalated\_qa\_score': 0.9269446373201656\}
            - This is a valid and consistent response. The format is a dictionary and not textual. Despite the format, the response is complete and accurate.
        - Response 2:
            '''
            \# Analysis of Supervisor Escalations and QA Scores

            Based on the analysis, I can provide some interesting insights about supervisor escalations and their relationship with QA scores:

            - Only 0.11\% of all interactions required supervisor escalation over the past year
            - The average QA score for escalated calls was 94.00
            - For comparison, non-escalated calls had an average QA score of 96.33

            While escalated calls do show slightly lower QA scores, the difference is relatively small (only 2.33 points lower), suggesting that agents generally maintain good quality standards even in challenging situations that require escalation.
            '''
            - This is a valid and consistent response. However, there's a lot of fluff in the response. In addition, the response is redundant as it contains information about non-escalated calls which is not relevant to the Query.

- Do not focus on the correctness of the numeric values such as counts, percentages, average scores etc. because you are not provided with the ground truth. Simply because there's an alignment in the numeric values between two of the three responses does not mean that the third response is incorrect.

\#\# Output format

Return only a single JSON object (no prose, no markdown) with this exact schema and key order:

\{
  "rationale": "<2-4 sentence high-level justification that explains your reasoning for how the responses compare>",
  "R1": \{"validity": <0|1>, "consistency": <0|1>, "completeness": <1-7>, "redundancy": <1-7>\},
  "R2": \{"validity": <0|1>, "consistency": <0|1>, "completeness": <1-7>, "redundancy": <1-7>\},
  "R3": \{"validity": <0|1>, "consistency": <0|1>, "completeness": <1-7>, "redundancy": <1-7>\},
  "overall\_decision": "<one of the allowed strings>"
\}

\#\# Process requirements

- First, analyze the three responses and produce the rationale (2-4 sentences).

- Then assign the four metric scores (validity, consistency, completeness, redundancy) for R1, R2, R3.

- Finally, generate the overall\_decision strictly from the allowed set.
\\ \hline

\end{tabularx}
\caption{System prompt for the Judge LLM used to evaluate and compare R1/R2/R3 responses in the end-to-end QA study.}
\label{tab:prompt_qa_study}
\end{table*}

\begin{table*}[t]
\centering
\small  
\begin{tabularx}{\textwidth}{|>{\raggedright\arraybackslash}p{0.8in}|
                           >{\raggedright\arraybackslash}X|
                           >{\raggedright\arraybackslash}X|}
\hline
\textbf{Stage} & \textbf{Generated Output} & \textbf{Annotated Output} \\
\hline \hline

\textbf{Query Generated} & 
\ttfamily
What are customers most unhappy about in refund-related calls longer than 30 minutes? &
N/A
\\\hline

\textbf{Initial Plan Generated} & 
\ttfamily
\{\newline
    1: \{"step": T2S([], "Fetch interaction\_ids of interactions where the call driver is related to refunds and duration is longer than 30 minutes."), "depends\_on": []\},\newline
    2: \{"step": RAG((1), "What are customers most unhappy about?"), "depends\_on": [1]\}\newline
\}
& 
N/A
\\\hline

\textbf{Metric-Wise and One-Shot Evaluator Score for Initial Plan:} & 
\ttfamily
1. Format: 20\newline
2. Tool-Prompt Alignment: 15\newline
3. Step-Atomicity: 7.5\newline
4. Query-Adherence: 12\newline
5. Dependencies: 10\newline
6. Redundancies: 10\newline
7. Tool-Usage Completeness: 5\newline
8. Total (Metric-Wise): 79.5\newline
9. One-Shot: 81.5 (Good)
& 
N/A
\\\hline

\textbf{Step-Evaluator Output:} On \textbf{1st} pass through feedback loop and step to be judged is \textbf{Step 1} & 
\ttfamily
Step: 1\newline
1. Nothing is wrong in this step and it can be executed normally : NO CHANGE
&
\ttfamily
Step: 1\newline
1. The query asks for "refund-related calls" specifically whereas the step filters for all interactions. So the "interaction\_ids" needs to be changed to "call\_ids" and "interactions" needs to be changed to "calls": INCORRECT PROMPT
\\\hline

\textbf{Plan-Optimizer Output:} On \textbf{1st} pass through feedback loop. Recieves the input from previous step-evaluator output & 
\ttfamily
"CHANGE 0": Based on the output of the step-wise evaluator, the step requires no change.\newline
"CHANGE 1": Since there is no change in change 0, it is clear that there won't be any changes in the change 1 either and the plan remains unchanged.\newline
"NEW PLAN STARTS":\newline
\{\newline
    1: \{"step": T2S([], "Fetch interaction\_ids of interactions where the call driver is related to refunds and duration is longer than 30 minutes."), "depends\_on": []\},\newline
    2: \{"step": RAG((1), "What are customers most unhappy about?"), "depends\_on": [1]\}\newline
\}
&
\ttfamily
"CHANGE 0": Based on the output of the step-wise evaluator, the terms "interaction\_ids" and "interactions" need to be changed "call\_ids" and "calls".\newline
"CHANGE 1": The rest of the plan makes perfect sense with this change alone, so no additional change 1 needs to be made.\newline
"NEW PLAN STARTS":\newline
\{\newline
    1: \{"step": T2S([], "Fetch call\_ids of calls where the call driver is related to refunds and duration is longer than 30 minutes."), "depends\_on": []\},\newline
    2: \{"step": RAG((1), "What are customers most unhappy about?"), "depends\_on": [1]\}\newline
\}
\\\hline

\textbf{Step-Evaluator Output: } On \textbf{2nd} pass through feedback loop and step to be judged is \textbf{Step 2} 

&
\ttfamily
Step: 2.\newline
1. Reasons for customer unhappiness can be found out using T2S as well. So this step should be performed by T2S along with RAG : MULTI-TOOL PROMPT
& 
\ttfamily
Step: 2.\newline
1. The step itself is fine, but reasons for customer unhappiness can be found out using the call driver data in T2S as well. So this step should be performed by both the tools. : MULTI-TOOL PROMPT
\\\hline

\end{tabularx}
\caption{\textbf{End-to-End Lineage Guided Decomposition Evaluation Trace:} Stage-wise Generated Outputs vs. Human-Annotated Plans}
\label{tab:example_lineage_evaluation_1}
\end{table*}

\begin{table*}[t]
\centering
\small  
\begin{tabularx}{\textwidth}{|>{\raggedright\arraybackslash}p{0.8in}|
                                >{\raggedright\arraybackslash}X|
                                >{\raggedright\arraybackslash}X|}
\hline
\textbf{Stage} & \textbf{Generated Output} & \textbf{Annotated Output} \\
\hline

\hline

\textbf{Plan-Optimizer Output: } On \textbf{2nd} pass through feedback loop. Recieves the input from previous step-evaluator output & 
\ttfamily
"CHANGE 0": As per the step-evaluator output, an additional step needs to be added with the same prompt as that of step 3 but with T2S as the tool.\newline
"CHANGE 1": Since we are adding this new step with the same prompt but different tool, we also need to add an LLM aggregation step to combine the outputs from these 2 steps into the final output that the user will receive.\newline
"NEW PLAN STARTS":\newline
\{\newline
    1: \{"step": T2S([], "Fetch interaction\_ids of interactions where the call driver is related to refunds and duration is longer than 30 minutes."), "depends\_on": []\},\newline
    2: \{"step": RAG((1), "What are customers most unhappy about?"), "depends\_on": [1]\},\newline
    3: \{"step": T2S((1), "What are customers most unhappy about?"), "depends\_on": [1]\},\newline
    4: \{"step": LLM("Combine the following 2 outputs to answer the following query (query):First:(2)Second:(3)"), "depends\_on": [2, 3]\},\newline
\}
&
\ttfamily
"CHANGE 0": According to the step-evaluator’s feedback, an extra step should be inserted using the same prompt as step 3, but with T2S selected as the tool.\newline
"CHANGE 1": Since this new step uses the same prompt but a different tool, an additional LLM aggregation step must be included to merge the outputs from both steps into the final response that will be delivered to the user.\newline
"NEW PLAN STARTS":\newline
\{\newline
    1: \{"step": T2S([], "Fetch call\_ids of calls where the call driver is related to refunds and duration is longer than 30 minutes."), "depends\_on": []\},\newline
    2: \{"step": RAG((1), "What are customers most unhappy about?"), "depends\_on": [1]\},\newline
    3: \{"step": T2S((1), "What are customers most unhappy about?"), "depends\_on": [1]\},\newline
    4: \{"step": LLM("Combine the following 2 outputs to answer the following query (query):First:(2)Second:(3)"), "depends\_on": [2, 3]\},\newline
\}
\\\hline

\textbf{Step-Evaluator Output: } On \textbf{3rd} pass through feedback loop and step to be judged is \textbf{Step 3} 

&
\ttfamily
Step: 3 \newline
1. The step itself is fine, but reasons for customer unhappiness can be found out using RAG as well. So this step should be performed by both the tools. : MULTI-TOOL PROMPT

&
\ttfamily
Step: 3 \newline
1. The step itself is appropriate, but since reasons for customer unhappiness can also be identified using RAG, this step should be executed by both tools. : MULTI-TOOL PROMPT
\\\hline

\textbf{Plan-Optimizer Output: } On \textbf{3rd} pass through feedback loop. Receives the input from previous step-evaluator output & 
\ttfamily
"CHANGE 0": Looking at the step-evaluator and the old plan, it seems that the step is already being performed by RAG in the previous step. So, nothing needs to be done to the plan and it can remain as it is. \newline
"CHANGE 1": Since there is no change 0, we don't need any change 1 either.\newline
"NEW PLAN STARTS":\newline
\{\newline
    1: \{"step": T2S([], "Fetch interaction\_ids of interactions where the call driver is related to refunds and duration is longer than 30 minutes."), "depends\_on": []\},\newline
    2: \{"step": RAG((1), "What are customers most unhappy about?"), "depends\_on": [1]\},\newline
    3: \{"step": T2S((1), "What are customers most unhappy about?"), "depends\_on": [1]\},\newline
    4: \{"step": LLM("Combine the following 2 outputs to answer the following query (query):First:(2)Second:(3)"), "depends\_on": [2, 3]\},\newline
\}
& 
\ttfamily
"CHANGE 0": Based on the step-evaluator output and the existing plan, this step is already being handled by RAG in the preceding step. Therefore, no modifications are required, and the plan can remain unchanged. \newline
"CHANGE 1": As no change 0 is needed, there is likewise no need for any change 1.\newline
"NEW PLAN STARTS":\newline
\{\newline
    1: \{"step": T2S([], "Fetch call\_ids of calls where the call driver is related to refunds and duration is longer than 30 minutes."), "depends\_on": []\},\newline
    2: \{"step": RAG((1), "What are customers most unhappy about?"), "depends\_on": [1]\},\newline
    3: \{"step": T2S((1), "What are customers most unhappy about?"), "depends\_on": [1]\},\newline
    4: \{"step": LLM("Combine the following 2 outputs to answer the following query (query):First:(2)Second:(3)"), "depends\_on": [2, 3]\},\newline
\}

\\\hline

\end{tabularx}
\caption{\textbf{End-to-End Lineage Guided Decomposition Evaluation Trace:} Stage-wise Generated Outputs vs. Human-Annotated Plans (Contd.)}
\label{tab:example_lineage_evaluation_2}
\end{table*}

\begin{table*}[t]
\centering
\small  
\begin{tabularx}{\textwidth}{|>{\raggedright\arraybackslash}p{0.8in}|
                                >{\raggedright\arraybackslash}X|
                                >{\raggedright\arraybackslash}X|}
\hline
\textbf{Stage} & \textbf{Generated Output} & \textbf{Annotated Output} \\
\hline

\hline

\textbf{Step-Evaluator Output: } On \textbf{4th} pass through feedback loop and step to be judged is \textbf{Step 4} & 
\ttfamily
Step: 4 \newline
1. The step is appropriate and requires no additional changes. : NO CHANGE \newline
& 
\ttfamily
Step: 4 \newline
1. The step is perfectly fine and executable as it is and does not require any further modifications. : NO CHANGE \newline
\\\hline

\textbf{Plan-Optimizer Output: } On \textbf{4th} pass through feedback loop. Recieves the input from previous step-evaluator output & 
\ttfamily
"CHANGE 0": As per the step-evaluator output, no adjustment is required. \newline
"CHANGE 1": Since there is no change 0, we don't need any change 1 either. \newline
"NEW PLAN STARTS": \newline
\{\newline
    1: \{"step": T2S([], "Fetch interaction\_ids of interactions where the call driver is related to refunds and duration is longer than 30 minutes."), "depends\_on": []\},\newline
    2: \{"step": RAG((1), "What are customers most unhappy about?"), "depends\_on": [1]\},\newline
    3: \{"step": T2S((1), "What are customers most unhappy about?"), "depends\_on": [1]\},\newline
    4: \{"step": LLM("Combine the following 2 outputs to answer the following query (query):First:(2)Second:(3)"), "depends\_on": [2, 3]\},\newline
\}
&
\ttfamily
"CHANGE 0": According to the step-evaluator output, no modifications are necessary. \newline
"CHANGE 1": As no change 0 is required, there is likewise no need for a change 1. \newline
"NEW PLAN STARTS": \newline
\{\newline
    1: \{"step": T2S([], "Fetch call\_ids of calls where the call driver is related to refunds and duration is longer than 30 minutes."), "depends\_on": []\},\newline
    2: \{"step": RAG((1), "What are customers most unhappy about?"), "depends\_on": [1]\},\newline
    3: \{"step": T2S((1), "What are customers most unhappy about?"), "depends\_on": [1]\},\newline
    4: \{"step": LLM("Combine the following 2 outputs to answer the following query (query):First:(2)Second:(3)"), "depends\_on": [2, 3]\},\newline
\}
\\\hline

\textbf{Final Plan: } Gotten upon doing Human Annotation on the feed-back loop output plan. In this case, "interaction\_ids" was changed to "call\_ids" in step 1 since the query specifically asks for calls. This is a mistake that the feedback-loop failed to fix.

&
\ttfamily
\{\newline
    1: \{"step": T2S([], "Fetch call\_ids of calls where the call driver is related to refunds and duration is longer than 30 minutes."), "depends\_on": []\},\newline
    2: \{"step": RAG((1), "What are customers most unhappy about?"), "depends\_on": [1]\},\newline
    3: \{"step": T2S((1), "What are customers most unhappy about?"), "depends\_on": [1]\},\newline
    4: \{"step": LLM("Combine the following 2 outputs to answer the following query (query):First:(2)Second:(3)"), "depends\_on": [2, 3]\},\newline
\}

&
N/A
\\\hline

\textbf{Metric-Wise and One-Shot Evaluator Score for Final Plan} & 
\begin{flushleft}
    \texttt{1. Format: 20}\\
    \texttt{2. Tool-Prompt Alignment: 20}\\
    \texttt{3. Step-Atomicity: 15}\\
    \texttt{4. Query-Adherence: 15}\\
    \texttt{5. Dependencies: 10}\\
    \texttt{6. Redundancies: 10}\\
    \texttt{7. Tool-Usage Completeness: 10}\\
    \texttt{8. Total (Metric-Wise): 100}\\
    \texttt{9. One-Shot: 98 (Extremely Good)}\\
\end{flushleft}
& 
\begin{flushleft}
    \texttt{1. Format: 20}\\
    \texttt{2. Tool-Prompt Alignment: 20}\\
    \texttt{3. Step-Atomicity: 15}\\
    \texttt{4. Query-Adherence: 15}\\
    \texttt{5. Dependencies: 10}\\
    \texttt{6. Redundancies: 10}\\
    \texttt{7. Tool-Usage Completeness: 10}\\
    \texttt{8. Total (Metric-Wise): 100}\\
    \texttt{9. One-Shot: 100 (Extremely Good)}\\
\end{flushleft}
\\\hline

\end{tabularx}
\caption{\textbf{End-to-End Lineage Guided Decomposition Evaluation Trace:} Stage-wise Generated Outputs vs. Human-Annotated Plans (Contd.)}
\label{tab:example_lineage_evaluation_3}
\end{table*}

\begin{table*}[t]
\centering
\small  
\begin{tabularx}{\textwidth}{|>{\raggedright\arraybackslash}p{1.0in}|
                           >{\raggedright\arraybackslash}Y|}

\hline
\textbf{Query Type} & \textbf{Query} \\
\hline \hline

\textbf{Simple} & 
\ttfamily
What are the main reasons behind interactions that are more than 1 hour long?
\\ \hline

\textbf{Compound} & 
\ttfamily 
Show interaction count breakout by entity type for last 60 days and for each entity type give details of agent who attended most number of interactions also for those agents give details of their teams
\\ \hline

\textbf{Subjective} & 
\ttfamily 
What causes customer dissatisfaction in calls longer than 45 minutes?
\\ \hline

\textbf{Objective} & 
\ttfamily 
How do customer concerns and issues change during 18-24 Nov compared to 25-30 Nov?
\\ \hline

\textbf{Zero-Hop} & 
\ttfamily 
What is the average call duration in the last week?
\\ \hline

\textbf{One-Hop} & 
\ttfamily 
What kind of feedback are customers giving about agent professionalism during their calls?
\\ \hline

\textbf{Two-Hop} & 
\ttfamily 
Which issues trigger longer calls related to food ordering, and how are agents alleviating customer frustrations during these calls?
\\ \hline

\textbf{Three-Plus Hop} & 
\ttfamily 
Identify patterns in customer dissatisfaction related to refunds, and correlate these with the QA scores of associated agents.
\\ \hline

\end{tabularx}
\caption{Query Examples Broken Down By \textbf{Subjective/Objective, Simple/Compound and Number of Hops.}}
\label{tab:examples_by_query_types}
\end{table*}

\section{Tool Specifications and Domain Data}
\label{app:tools}

\subsection{Tools and APIs}
All tools are internally built; we expose schemas used in our experiments.

\paragraph{Common notation.}
Each tool call takes (i) an optional list of identifiers (e.g., \texttt{call\_ids} or \texttt{interaction\_ids}) passed positionally via a placeholder like ``(k)'' to denote step-$k$ output, and (ii) a natural-language prompt.
Placeholders MUST be used whenever a step depends on prior outputs.

\subsubsection{T2S: Text-to-SQL over Snowflake}
\begin{enumerate}
    \item \textbf{Signature.} \texttt{T2S(call\_ids: list, prompt: str) -> text}  
    \item \textbf{Function.} Generates SQL from \texttt{prompt}, optionally constraining results to \texttt{call\_ids}; executes the SQL in Snowflake and returns a natural-language summary.
    \item \textbf{Strengths.} Counts, trends, rollups over \emph{structured} contact-center data.  
    \item \textbf{Limits.} Does not track \emph{within-call} temporal dynamics (e.g., sentiment \emph{shifts}); those require RAG.
\end{enumerate}

\subsubsection{RAG: Retrieval-Augmented Generation over Transcripts}
\begin{enumerate}
    \item \textbf{Signature.} \texttt{RAG(call\_ids: list, prompt: str) -> text + table}  
    \item \textbf{Function.} Retrieves and analyzes customer--agent transcripts, returning (i) a textual answer and (ii) a \emph{Data Insights} table with columns:
    \begin{itemize}
        \item \texttt{Topic Name}, \texttt{Topic Description}
        \item \texttt{Interaction Ids} (list)
        \item \texttt{Interactions (\%)} (sample-relative)
    \end{itemize}
    \item \textbf{Sampling.} For broad thematic prompts, RAG analyzes a sample (e.g., $\sim$200 interactions); percentages are with respect to the sample.  
    \item \textbf{Strengths.} Conversational content, phrasing, subjective judgments, \emph{within-call} events (e.g., sentiment shifts).  
    \item \textbf{Limits.} Sample-based; ID extraction for downstream tools often requires an LLM reformat step.
\end{enumerate}

\subsubsection{LLM: Synthesis and Reformatting}
\begin{enumerate}
    \item \textbf{Signature.} \texttt{LLM(prompt: str) -> text}
    \item \textbf{Function.} (i) merges multi-tool evidence; (ii) reformats outputs (e.g., extract \texttt{interaction\_ids} from RAG tables); (iii) performs lightweight computations (set ops, joins of previously returned aggregates); (iv) synthesizes final answers.
\end{enumerate}

\subsection{Domain Data Fields in Snowflake}
\label{app:domain-fields}
\begin{description}
    \item[Interaction/Call metadata:]
    channel (\texttt{call/chat/email}), timestamps, agent, team, customer/merchant/dasher tags.
    \item[Call drivers (\emph{interaction drivers}):]
    normalized reason categories (e.g., payment issues, delivery delays, missing items, account access, order status).
    \item[Moments:]
    Events of interest extracted from agent-customer transcripts such as escalation, transfer, \texttt{sentiment\_tag}, \texttt{abusive\_interaction}, \texttt{issue\_resolved}, \texttt{follow\_up\_required}.%
    \footnote{Some moments (e.g., \emph{shift} from negative$\to$positive) are only detectable in transcripts and thus via RAG; Snowflake typically stores \emph{call-level} sentiment tags.}
    \item[QA metrics:]
    overall score and category-level dimensions (e.g., Compliance \& PII Security, Case Handling \& Procedural Adherence, Timeliness, Professionalism, Empathy, Closing/Wrap-up, Communication Clarity), plus pass/fail flags for specific checks (e.g., “agent did not end call properly”). The category-level dimensions correspond to questions which are sourced from call evaluation forms, each designed to assess various aspects of agent performance during interactions.
\end{description}

\subsection{Execution Constraints and Best Practices}
\label{app:constraints}
\begin{itemize}
    \item \textbf{DAG:} Dependencies $D_k$ must form a DAG. Steps may execute in parallel when their dependencies are satisfied.
    \item \textbf{No redundant filtering:} If step $k$ consumes IDs produced by step $i$, $p_k$ should not restate filters used to produce those IDs (avoids unnecessary joins/timeouts).
    \item \textbf{RAG $\to$ ID extraction:} When a downstream step needs identifiers from a RAG table, insert an LLM step to extract/format the IDs (e.g., list of \texttt{interaction\_ids}).
    \item \textbf{Tool capability guardrails:} Use RAG for \emph{within-call} dynamics (e.g., sentiment shifts). Use T2S for structured aggregates (counts, trends, QA rollups).
\end{itemize}

\section{Full Example Lineage}
\label{app:lineage-example}
\noindent\textbf{Query.} \emph{``Analyze calls flagged as unresolved where the customer's sentiment transitioned from negative to positive within the transcript, and compare agent QA scores for resolution procedures versus professionalism.''}

\noindent\textbf{Initial $\rightarrow$ Final lineage (abbrev.).}
\begin{enumerate}
    \item \emph{P\textsuperscript{(0)} (weak):} Single-step T2S mixing unresolved status \emph{and} within-call sentiment shift (unsupported) $\Rightarrow$ timeout/low fidelity.
    \item \emph{P\textsuperscript{(1)}:} Split unresolved filtering (T2S) from sentiment shift (RAG), but missing ID extraction; downstream T2S cannot accept RAG table.
    \item \emph{P\textsuperscript{(2)}:} Insert LLM to extract \texttt{interaction\_ids} from RAG’s Data Insights; add QA rollups via T2S; missing final synthesis.
    \item \emph{P\textsuperscript{(3)} (best):} 6-step plan as in \S\ref{lst:example-plan} example: unresolved filter (T2S) $\rightarrow$ shift detection (RAG) $\rightarrow$ ID extraction (LLM) $\rightarrow$ QA rollups (T2S) $\rightarrow$ synthesis (LLM).
\end{enumerate}

\section{Dataset Generation Details}
\label{app:dataset-details}

\subsection{Query Generation}
\label{app:query-gen}
We generate a pool of queries stratified across:
\begin{itemize}
  \item \textbf{Subjectivity:} objective (counts, durations, resolution rates) vs.\ subjective (themes, phrasing, behaviors, sentiment).
  \item \textbf{Compoundness:} simple (single ask) vs.\ compound (multiple asks or comparisons).
\end{itemize}
This controls difficulty and ambiguity while matching realistic contact-center analytics needs.

\subsection{Initial Plan Generation}
\label{app:initial-plan}
We craft a two-part prompt:
\begin{enumerate}
  \raggedright
  \item \textbf{System prompt:} task definition; JSON schema for plans (step, depends\_on); tool capabilities and guardrails (T2S/RAG/LLM).
  \item \textbf{User prompt:} formatting constraints, placeholder rules, dependency DAG requirement, and few-shot exemplars.
\end{enumerate}
We ablate tool-description verbosity (low/medium/high) and few-shot counts (1–15). The \emph{medium} tool detail with $6$–$8$ examples consistently avoids under-specification and prompt overload. The prompts used for this are provided in Table~\ref{tab:prompts_for_query_and_plan_gen}.

\subsection{Iterative Evaluator\texorpdfstring{$\rightarrow$}{->}Optimizer Loop}
\label{app:iterative-loop}
The loop refines a plan without executing tools. It consists of:
\paragraph{Step-wise Evaluator.}
Inspects each step’s tool, prompt, and dependencies; emits diagnostic tags with justifications:
\begin{itemize}
  \item \textbf{Incorrect tool} (capability mismatch).
  \item \textbf{Complex prompt} (needs decomposition).
  \item \textbf{Repeated detail} (repeating filters already implied by prior IDs).
  \item \textbf{Multi-tool prompt} (eligible for T2S and RAG; consider dual coverage).
  \item \textbf{Incorrect prompt} (format-related errors, incorrect or insufficient information in the prompt).
  \item \textbf{No change}.
\end{itemize}

\paragraph{Plan Optimizer.}
Consumes diagnostics and makes two passes:
\begin{itemize}
  \item \textbf{Change 0} (local fix): apply minimal edits per tag (e.g., replace tool, simplify prompt, remove redundant filter).
  \item \textbf{Change 1} (coherence fix): repair global consistency (dependencies, placeholders), add missing split steps, merge redundancies.
\end{itemize}
Each revision is stored, forming the \emph{plan lineage} for the query.

\subsection{Plan Refinement via Evaluator\texorpdfstring{$\rightarrow$}{->}Optimizer Loop}
\label{app:loop-details}
\paragraph{Overview.}
Given an initial plan $P^{(0)}$, we iteratively improve its structure and correctness without executing tools. At each iteration, a \emph{Step-wise Evaluator} diagnoses issues at the step level; a \emph{Plan Optimizer} then applies edits and emits a revised plan, which is stored in the \emph{lineage}.

\paragraph{Step-wise Evaluator.}
Input: a single step (tool + prompt) and its declared dependencies. Output: a \emph{set} of tags plus brief rationales. Supported tags:
\begin{itemize}
  \item \textsc{IncorrectTool} — tool capability and prompt intent mismatch.
  \item \textsc{ComplexPrompt} — prompt needs decomposition into simpler atomic steps.
  \item \textsc{RepeatedDetail} — repeated filters (e.g., re-filtering on a criterion already used to create the call/interaction ID set).
  \item \textsc{MultiToolPrompt} — task can (or should) be covered by both T2S and RAG for completeness.
  \item \textsc{IncorrectPrompt} — Prompt has discrepancies such as format-related issues (parentheses, placeholders, or dependency notation errors) and insufficient or incorrect information.
  \item \textsc{NoChange} — the step is acceptable as-is.
\end{itemize}
Notably, the evaluator \emph{does not} receive the original query as input (empirically improving consistency).

\paragraph{Plan Optimizer.}
Input: current plan $P$, step index $i$, diagnostic tags and rationales. Output: revised plan $P'$. The optimizer performs:
\begin{enumerate}
  \item \textbf{Change 0 (local repair):} swap misaligned tools, simplify/decompose complex prompts, remove redundant filters, fix format.
  \item \textbf{Change 1 (coherence repair):} rewire dependencies/placeholders for a valid DAG, split/merge steps if required, ensure that outputs used downstream have a unique producer and are referenced exactly once where intended.
\end{enumerate}
Each time $P' \neq P$, we append $P'$ to the lineage.

\paragraph{Control flow and guarantees.}
Algorithm~\ref{alg:plan-optimizer} implements a bounded-pass scan: we either advance to the next step when length is unchanged or re-check the same index when structure changes. The outer guard ($pass<M$) ensures termination; see Lemma~\ref{lem:optimizer-termination}. The loop improves plans textually and deterministically modulo LLM sampling settings (we fix temperature and seeds during curation).

\clearpage
\begin{algorithm}[H]
\caption{Planner Feedback Loop: Step-wise Evaluator $\rightarrow$ Plan Optimizer}
\label{alg:plan-optimizer}
\begin{algorithmic}[1]
\Require Initial plan $P^{(0)} = \{p_1, \dots, p_n\}$
\Ensure Optimized plan $P^\star$ and full lineage $\mathcal{L}$
\State $P \gets P^{(0)}$;\quad $\mathcal{L} \gets [P]$ \Comment{store lineage states}
\State $max\_passes \gets M$ \Comment{e.g., $M{=}4$}
\State $pass \gets 0$;\quad $changed \gets \textbf{true}$
\While{$changed = \textbf{true}$ \textbf{and} $pass < max\_passes$}
    \State $changed \gets \textbf{false}$;\quad $i \gets 1$;\quad $\ell \gets \text{length}(P)$
    \While{$i \le \ell$}
        \State $step \gets P[i]$;\quad $deps \gets \textsc{DepsOf}(P, i)$
        \State $(tags, rationale) \gets \textsc{StepwiseEvaluate}(step, deps)$
        \If{$tags = \{\textsc{NoChange}\}$}
            \State $i \gets i + 1$
        \Else
            \State $P' \gets \textsc{PlanOptimize}(P, i, tags, rationale)$
            \If{ $P \neq P'$}
                \If{$\text{length}(P') = \ell$}
                    \State $P \gets P'$
                    \State append $P$ to $\mathcal{L}$
                    \State $i \gets i + 1$
                \Else
                    \State $P \gets P'$
                    \State append $P$ to $\mathcal{L}$
                    \State $\ell \gets \text{length}(P)$ \Comment{re-check the same $i$ on next loop}
                \EndIf
                \State $changed \gets \textbf{true}$
            \Else
                \State $i \gets i + 1$
            \EndIf
        \EndIf
    \EndWhile
    \State $pass \gets pass + 1$
\EndWhile
\State \Return $P^\star \gets P$, $\mathcal{L}$
\end{algorithmic}
\end{algorithm}

\newtheorem{lemma}{Lemma}
\begin{lemma}[Termination]
\label{lem:optimizer-termination}
The feedback loop in Algorithm~\ref{alg:plan-optimizer} terminates in at most $M$ passes. Within each pass, the inner scan either (i) advances the step index $i$ or (ii) applies a finite structural change to the plan that is immediately recorded and re-checked; the outer guard $pass<M$ guarantees termination.
\end{lemma}

\begin{proof}[Proof sketch]
Each pass is bounded by the guard $pass<M$. Inside a pass, whenever the plan changes, the lineage is updated and the scan either advances to $i{+}1$ (no length change) or re-checks the same $i$ (length change); when no changes occur for a full pass, the loop exits. Hence the procedure stops after at most $M$ passes.
\end{proof}

\subsection{Human Verification}
\label{app:human-verification}
Expert annotators review the final plan in each lineage; minor edits (if needed) produce the \emph{best} plan. All lineage states are retained for analysis/training.

\subsection{Number of Hops: Definition and Computation}
\label{app:hops-definition}
We define \emph{number of hops} from the dependency structure of the best possible plan for a query. Let the plan be a Directed Acyclic Graph (DAG) where each node is a step and edges point from a prerequisite step to its consumer.

\paragraph{Categories.}
\begin{itemize}
    \item \textbf{Zero-hop:} No dependencies. Direct tool calls (e.g., a single \texttt{T2S} or \texttt{RAG} step) produce the final answer; multiple independent producers also qualify if no step depends on another.
    \item \textbf{One-hop:} Exactly one dependency layer. Typical pattern: an \texttt{LLM} synthesis step depends on one or more independent producer steps (e.g., \texttt{T2S} and \texttt{RAG}), neither of which depends on prior steps.
    \item \textbf{Two-hop:} Two sequential dependency layers. Example: a final \texttt{LLM} step depends on \texttt{RAG}/\texttt{T2S} analysis steps, each of which depends on a \texttt{T2S} filtering step that scopes the interaction IDs.
    \item \textbf{Three-plus:} Three or more sequential dependency layers; e.g., filter \texorpdfstring{$\rightarrow$}{->} enrich \texorpdfstring{$\rightarrow$}{->} analyze \texorpdfstring{$\rightarrow$}{->} synthesize.
\end{itemize}

\paragraph{Operational computation.}
We compute hop count as the maximum dependency depth among steps that contribute to the final answer:

\begin{gather}
\text{depth}(s)=
\begin{cases}
0, \text{ if }\mathrm{depends\_on}(s)=\varnothing,\\
1+\max\limits_{p \in \mathrm{depends\_on}(s)} \text{depth}(p)
\end{cases}
\\
\text{hops}=\max_{s \in S_{\mathrm{final}}} \text{depth}(s).
\end{gather}

where $S_\text{final}$ are sink steps (no consumers) or designated answer-producing steps (e.g., the final \texttt{LLM} synthesis). 

\paragraph{Examples.}
\begin{enumerate}
    \item \texttt{T2S([], ``Compute average call duration last month'')} $\Rightarrow$ \textbf{zero-hop}.
    \item \texttt{RAG((1), ...)} and \texttt{T2S((1), ...)} not used; instead two independent producers \texttt{RAG([], ...)}, \texttt{T2S([], ...)} feeding an \texttt{LLM} $\Rightarrow$ \textbf{one-hop}.
    \item \texttt{T2S([], ``Fetch interaction\_ids ...'')} $\rightarrow$ \{\texttt{RAG((1), ...)}, \texttt{T2S((1), ...)}\} $\rightarrow$ \texttt{LLM(...)} $\Rightarrow$ \textbf{two-hop}.
    \item Extend with another dependent analysis layer before synthesis $\Rightarrow$ \textbf{three-plus}.
\end{enumerate}

We compute this hop count algorithmically for each gold plan and use it as metadata when analyzing model performance by reasoning depth.

\subsection{Module Set and Validation Protocols}
\label{app:modules}
We employ four LLM-driven modules:
\begin{enumerate}
  \item \textbf{Metric-wise evaluator} (7 metrics; see \S\ref{sec:evaluation}).
  \item \textbf{One-shot evaluator} (reference-based: precision/recall/F1 + format; 7-point rating; see \S\ref{sec:evaluation})
  \item \textbf{Step-wise evaluator} (diagnostics tags; see \S\ref{sec:loop-brief}).
  \item \textbf{Plan optimizer} (Change 0/1; lineage revision; see \S\ref{sec:loop-brief}).
\end{enumerate}
Their prompts and evaluation against human ground truth are reported in Tables~\ref{tab:prompts_for_step_wise_eval}--\ref{tab:prompts_for_oneshot_evaluator} and Tables~\ref{tab:validation_result_metric_wise_eval}--\ref{tab:validation_result_plan_optimizer} respectively.

\section{Evaluation Details}
\label{app:evaluation}

\subsection{Metric Definitions and Rubrics}
\label{app:metric-details}

All seven metrics are computed by LLM evaluators using task-specific rubric prompts. Below we summarize decision rules; full prompts are included in Tables~\ref{tab:prompts_for_metric_wise_eval_1}--\ref{tab:prompts_for_metric_wise_eval_3}.

\paragraph{M1: Tool–Prompt Alignment (20 pts).}
Checks that the step’s prompt lies within the assigned tool’s capabilities. Penalize:
(i) sentiment \emph{changes within calls} assigned to T2S (use RAG instead);
(ii) filtering calls based on their duration assigned to RAG (use T2S instead);
(iii) prompting LLM to answer questions related to QA scores of agents (use T2S instead).
Scoring: The LLM evaluator assesses each step as pass or fail, computes the total number of passed steps, divides by the total number of steps, and then scales the result to a maximum of 20.

\paragraph{M2: Format Correctness (20 pts).}
Plan must be JSON parseable. Every dependent step must contain a correct numeric placeholder (e.g., \texttt{(2)}); quotes and parentheses must balance. Four standard placeholders are defined which are as follows:
\begin{itemize}
    \item \texttt{(query)}: Refers to the original query.
    \item \texttt{(<step\_id>)}: Refers to the output from step <step\_id>.
    \item \texttt{(tool <step\_id>))}: Refers to the name of the tool \{T2S, RAG, LLM\} assigned to the step <step\_id>).
    \item \texttt{(sub-query <step\_id>)}: Refers to the prompt passed to the tool assigned to the step <step\_id>.
\end{itemize}
The aforementioned standard placeholders should be used correctly in the plan.
Scoring: The LLM evaluator assesses each step as pass or fail based on the above rules, computes the total number of passed steps, divides by the total number of steps, and then scales the result to a maximum of 20.

\paragraph{M3: Step Executability / Atomicity (15 pts).}
Each step should be executable using a single tool call. Compound prompts—such as those given to RAG (e.g., “How did sentiment change \emph{and} what did agents do?”) or to T2S (e.g., “Analyze the trends in customer sentiment scores based on the primary call driver categories over the past six months.”)—are penalized. Compound prompts sent to RAG often produce incomplete responses, while those sent to T2S can result in timeouts, as complex instructions translate into SQL queries that are prone to exceeding execution limits in Snowflake. Steps that are explicitly decomposed into atomic units are rewarded.

Scoring: The LLM evaluator marks each step as pass or fail according to these rules, calculates the total number of passed steps, divides by the total number of steps, and scales the result to a maximum of 15 points.

\paragraph{M4: Query Adherence (15 pts).}
Judge whether the plan, if executed perfectly, answers the query fully and uses the correct modality (“calls” vs.\ “interactions”).

Scoring: The LLM evaluator returns a score in \{0, 0.5, 1\} depending on how well the generated plan adheres to the given query which is then scaled to a maximum of 15 points.

\paragraph{M5: Dependencies (10 pts).}
Every step that uses prior outputs has \texttt{depends\_on} populated and the prompt includes the correct placeholder(s). Penalize missing/incorrect placeholders and unnecessary edges.

Scoring: The LLM evaluator marks each step as pass or fail according to these rules, calculates the total number of passed steps, divides by the total number of steps, and scales the result to a maximum of 10 points.

\paragraph{M6: Redundancy (10 pts).}
Detect duplicated work or repeated filters. Typical anti-pattern: T2S step $(1)$ filters “delivery delay,” followed by T2S$(1)$ that again says “for the calls with delivery delay,” causing unnecessary joins. Deduct per occurrence.

Scoring: The LLM evaluator marks each step as pass or fail according to these rules, calculates the total number of passed steps, divides by the total number of steps, and scales the result to a maximum of 10 points.

\paragraph{M7: Tool-Usage Completeness (10 pts).}
When a step’s goal clearly merits dual-tool coverage (e.g., sentiments exist in transcripts and structured summaries), penalize plans using only one.

Scoring: The LLM evaluator marks each step as pass or fail depending on whether there's a violation with respect to tool-usage completeness, calculates the total number of passed steps, divides by the total number of steps, and scales the result to a maximum of 10 points.

\paragraph{Reference-free vs.\ with-reference split.}
All metrics are reference-based, as leveraging the reference (best possible) plan significantly enhances their performance. Table~\ref{tab:validation_result_metric_wise_eval} provides a comparison between reference-based and reference-free evaluators.

\subsection{One-Shot Overall Evaluation via Judge LLM}
\label{app:oneshot-details}

We evaluate a candidate plan $P$ against the best plan $P^\star$ using a Judge LLM prompted to compute:
\begin{itemize}
    \item \textbf{Precision:} \% of steps in $P$ present in $P^\star$
    \item \textbf{Recall:} \% of steps in $P^\star$ present in $P$
    \item \textbf{F$_1$}
\end{itemize}
The following are assessed on $P$ alone: \textbf{Format Correctness}, \textbf{Dependencies} \& \textbf{Placeholders}.
The Judge LLM then assigns a seven-point rating based on these scores, with the rule that non-JSON plans are rated \emph{Extremely Bad} irrespective of $F_1$.

\paragraph{Prompts.}
We use the following prompts (redacted for brevity here; full text in Table~\ref{tab:prompts_for_oneshot_evaluator}):
\begin{itemize}
    \item \textbf{System:} defines plan format, placeholder conventions, tool descriptions, six dimensions to score, rating rubric, and output schema.
    \item \textbf{User:} provides the query, candidate $P$, and best plan $P^\star$, plus reference examples.
\end{itemize}

The Judge LLM computes step matches and metrics per rubric in the system prompt (no hand-coded matching in our pipeline). Output must be valid JSON with both a rationale block and a score block.

\subsection{Aggregation and Adjudication}
All evaluators are LLM-based with deterministic decoding (Refer Table~\ref{tab:decoding} for LLM configurations). A random sample per split is manually audited to calibrate thresholds and check stability across model updates.

\subsection{Learning Metric Weights}
\label{app:weight-learning}

\paragraph{Setting.}
Each plan $P$ is scored on $M{=}7$ metrics, where the LLM evaluators return
raw sub-scores $s_m(P)\in[0,1]$ for $m\in\{1,\dots,7\}$. We learn nonnegative
\emph{weights} $w_m\!\ge\!0$ that represent the maximum points allocated to
each metric. The total score is
\[
S(P;\mathbf{w}) \;=\; \sum_{m=1}^{7} w_m\, s_m(P).
\]
We partition metrics into two groups: Effectiveness $\mathcal{E}$ (4 metrics) and
Efficiency $\mathcal{F}$ (3 metrics).

\paragraph{Constraints (fixed group budgets).}
We enforce point budgets at the group level,
\[
\sum_{m\in\mathcal{E}} w_m \;=\; B_{\mathcal{E}},\qquad
\sum_{m\in\mathcal{F}} w_m \;=\; B_{\mathcal{F}},
\]
with $B_{\mathcal{E}}{=}70$ and $B_{\mathcal{F}}{=}30$ in our experiments.%
\footnote{Equivalently, one can work on the probability simplex with a 70:30
split and rescale to points post hoc.}

\paragraph{Monotonic lineage objective.}
For each validation query $q$, we consider a lineage triple
$\big(P^{(q)}_{\text{best}},\,P^{(q)}_{\text{pen}},\,P^{(q)}_{\text{ante}}\big)$
derived from human annotations, where the intended aggregate ordering is
$\text{best} \succ \text{pen} \succ \text{ante}$. This ordering \emph{need not}
hold metric-wise; it is enforced only at the weighted sum level. We seek weights
$\mathbf{w}$ that maximize a global margin $\gamma\ge0$ while imposing strict,
margin-based monotonicity:
\begin{align*}
S\big(P^{(q)}_{\text{best}};\mathbf{w}\big)
&\ge S\big(P^{(q)}_{\text{pen}};\mathbf{w}\big) + \gamma, \quad \forall q, \\
S\big(P^{(q)}_{\text{pen}};\mathbf{w}\big)
&\ge S\big(P^{(q)}_{\text{ante}};\mathbf{w}\big) + \gamma, \quad \forall q.
\end{align*}
Let $\Delta s^{(q)}_{A\triangleright B,m} \!=\! s_m(P^{(q)}_A)-s_m(P^{(q)}_B)$.

\paragraph{Optimization program.}
We solve the following linear program (LP):
\[
\begin{aligned}
\max_{\mathbf{w},\,\gamma}\quad & \gamma \\
\text{s.t.}\quad
& \sum_{m\in\mathcal{E}} w_m = B_{\mathcal{E}}, \\
& \sum_{m\in\mathcal{F}} w_m = B_{\mathcal{F}}, \\
& w_m \ge 0, \quad \forall m, \\
& \sum_{m=1}^{7} w_m\, \Delta s^{(q)}_{\text{best}\,\triangleright\,\text{pen},m} \ge \gamma, \quad \forall q, \\
& \sum_{m=1}^{7} w_m\, \Delta s^{(q)}_{\text{pen}\,\triangleright\,\text{ante},m} \ge \gamma, \quad \forall q.
\end{aligned}
\]
If the LP is infeasible (rare; e.g., due to noisy labels), we solve a hinge-relaxed
variant with slacks $\xi_q,\zeta_q\ge0$ and penalty $C>0$:
\[
\begin{aligned}
\max_{\mathbf{w},\,\gamma,\,\xi,\,\zeta}\quad
& \gamma \;-\; C \sum_q (\xi_q+\zeta_q) \\
\text{s.t.}\quad
& \sum_{m\in\mathcal{E}} w_m = B_{\mathcal{E}}, \\
& \sum_{m\in\mathcal{F}} w_m = B_{\mathcal{F}}, \\
& w_m \ge 0, \quad \forall m, \\
& \sum_{m} w_m\, \Delta s^{(q)}_{\text{best}\,\triangleright\,\text{pen},m} \ge \gamma - \xi_q, \\
& \sum_{m} w_m\, \Delta s^{(q)}_{\text{pen}\,\triangleright\,\text{ante},m} \ge \gamma - \zeta_q, \\
& \xi_q \ge 0, \quad \zeta_q \ge 0, \quad \forall q.
\end{aligned}
\]

\paragraph{Practical procedure and quantization.}
We implement a coarse grid search on each group simplex with step size $0.02$,
selecting $\mathbf{w}$ that maximizes the number of satisfied inequalities and,
secondarily, the median margin. We then (i) rescale the groupwise weights to the
point budgets $(B_{\mathcal{E}},B_{\mathcal{F}})$ and (ii) quantize to the desired
lattice (e.g., integers or multiples of 5) by greedy round-and-adjust while preserving group sums and non-negativity. We
retain the quantized $\mathbf{w}$ with minimal deviation from the continuous
solution and revalidate the margin constraints.

\paragraph{Final weights used.}
In all experiments we use the following quantized weights (total $100$ points):
\emph{TP Alignment} $=20$, \emph{Format} $=20$,
\emph{Step Executability} $=15$, \emph{Query Adherence} $=15$,
\emph{Dependency} $=10$, \emph{Redundancy} $=10$,
\emph{Tool-Usage Completeness} $=10$.
This matches the $70{:}30$ Effectiveness/Efficiency budget and is consistent with
the learned proportions after quantization.

\paragraph{Why lineage-based weights?}
This objective encodes the intended ordinal relationship among the \emph{top} plans per query, ensuring the evaluator’s total score reflects true quality progression while respecting the Effectiveness/Efficiency budget. It reduces sensitivity to absolute metric scales and emphasizes rank-consistent scoring.

\subsection{Weight Sensitivity Analysis For Metric-Wise Evaluator}
\label{app:weight-sensitivity}

Our main aggregate score $S_{\text{learned}}$ uses metric-wise evaluator outputs and weights learned from human-preferred lineage triples under the 70{:}30 Effectiveness--Efficiency budget (App.~\ref{app:weight-learning}). To assess the robustness of planner rankings to this choice of weights, we conduct a small-scale sensitivity analysis on the same set of query–plan pairs used for metric-wise evaluation (14 LLMs $\times$ 2 prompt settings).

\paragraph{Schemes.}
For each plan, we first normalize each metric by its maximum possible value (Dependency, Format, Tool-Prompt Alignment, Redundancy, Tool-Usage Completeness, Step Executability, Query Adherence). We then consider:
\begin{enumerate}
    \item \textbf{Equal-weights scheme}:
    each normalized metric receives the same weight, and the aggregate is the average of normalized scores rescaled to $[0,100]$.
    \item \textbf{Random-weight schemes}:
    we sample ten random weight vectors that preserve the 70{:}30 Effectiveness--Efficiency budget by drawing Dirichlet distributions separately over the Effectiveness metrics (TP Alignment, Format, Step Executability, Query Adherence) and the Efficiency metrics (Dependency, Redundancy, Tool-Usage Completeness). Each random draw yields a new aggregate score $S_{\text{rand}}$.
\end{enumerate}

For each prompt type (with-lineage vs.\ without-lineage), we average scores per LLM and compute Spearman rank correlation between the rankings induced by $S_{\text{learned}}$ (the score used in the main text) and those from the alternative schemes.

\paragraph{Results.}
Table~\ref{tab:weight-sensitivity-rho} reports rank correlations. Under equal weights, planner rankings are very close to the learned-weight rankings ($\rho=0.934$ with-lineage; $\rho=0.894$ without-lineage). Across ten random draws that respect the 70{:}30 budget, the median correlation remains high ($\rho_{\text{med}}=0.890$ with-lineage; $\rho_{\text{med}}=0.842$ without-lineage), and even the worst random draws retain moderate agreement ($\rho_{\min}=0.736$ and $0.503$ respectively). Table~\ref{tab:weight-sensitivity-scores} lists per-LLM scores under the learned, equal, and random-weight schemes. Overall, high-ranked LLMs remain high-ranked, and only minor reorderings occur among mid-tier models, indicating that our LLM-level conclusions are robust to reasonable variations in metric weights and do not depend on a single, hand-picked setting.

\begin{table*}[t]
\small
\centering
\begin{tabular}{lcccc}
\toprule
\textbf{Prompt} & \textbf{$\rho_\text{equal}$} & \textbf{$\rho_\text{rand}^{\text{min}}$} & \textbf{$\rho_\text{rand}^{\text{med}}$} & \textbf{$\rho_\text{rand}^{\text{max}}$} \\
\midrule
With Lineage & 0.934 & 0.736 & 0.890 & 0.947 \\
Without Lineage               & 0.894 & 0.503 & 0.842 & 0.873 \\
\bottomrule
\end{tabular}
\caption{Spearman rank correlations between LLM rankings under the learned aggregate score and alternative metric-weighting schemes: equal weights over normalized metrics, and ten random 70{:}30 Effectiveness--Efficiency draws.}
\label{tab:weight-sensitivity-rho}
\end{table*}

\begin{table*}[t]
\small
\centering
\begin{tabular}{llrrrrr}
\toprule
\textbf{LLM} & \textbf{Prompt} & \textbf{Learned} & \textbf{Equal} & \textbf{Rand (min)} & \textbf{Rand (median)} & \textbf{Rand (max)} \\
\midrule
claude-3-7-sonnet              & With Lineage & 84.80 & 77.01 & 70.70 & 79.36 & 87.36 \\
llama4-maverick-17b-instruct   & With Lineage & 82.26 & 75.55 & 67.99 & 77.07 & 85.77 \\
gpt-4o                         & With Lineage & 81.47 & 74.37 & 67.56 & 75.43 & 83.78 \\
claude-sonnet-4                & With Lineage & 79.90 & 72.64 & 65.34 & 73.26 & 80.82 \\
llama3-3-70b-instruct          & With Lineage & 79.77 & 75.07 & 68.06 & 77.28 & 85.95 \\
nova-pro                       & With Lineage & 79.24 & 75.24 & 67.92 & 77.06 & 85.09 \\
nova-micro                     & With Lineage & 77.28 & 74.08 & 65.71 & 75.87 & 86.84 \\
gpt-4o-mini                    & With Lineage & 77.26 & 72.36 & 65.80 & 73.87 & 80.30 \\
gpt-4.1-nano                   & With Lineage & 77.16 & 71.13 & 64.63 & 72.35 & 78.33 \\
o3-mini                        & With Lineage & 71.85 & 68.49 & 57.02 & 71.13 & 77.48 \\
claude-3-5-haiku               & With Lineage & 71.27 & 65.56 & 59.04 & 66.80 & 71.92 \\
nova-lite                      & With Lineage & 70.53 & 68.31 & 59.85 & 69.89 & 77.31 \\
llama3-2-3b-instruct           & With Lineage & 70.51 & 66.51 & 57.26 & 66.64 & 77.71 \\
llama3-2-1b-instruct           & With Lineage & 20.27 & 20.25 & 7.06  & 16.22 & 32.81 \\
claude-3-7-sonnet              & Without Lineage   & 83.33 & 75.90 & 71.17 & 81.45 & 85.74 \\
gpt-4o                         & Without Lineage   & 82.82 & 74.90 & 69.81 & 79.74 & 83.43 \\
gpt-4o-mini                    & Without Lineage   & 82.78 & 75.00 & 72.01 & 79.55 & 86.51 \\
nova-pro                       & Without Lineage   & 82.70 & 75.82 & 70.55 & 80.20 & 88.05 \\
llama4-maverick-17b-instruct   & Without Lineage   & 82.50 & 74.88 & 72.07 & 80.03 & 83.49 \\
nova-micro                     & Without Lineage   & 82.07 & 74.68 & 73.36 & 79.95 & 84.67 \\
claude-3-5-haiku               & Without Lineage   & 82.06 & 75.71 & 67.93 & 80.34 & 87.88 \\
llama3-3-70b-instruct          & Without Lineage   & 81.11 & 74.48 & 68.04 & 78.82 & 85.29 \\
claude-sonnet-4                & Without Lineage   & 77.98 & 71.67 & 58.04 & 75.33 & 82.27 \\
gpt-4.1-nano                   & Without Lineage   & 76.05 & 69.14 & 61.63 & 72.35 & 80.47 \\
llama3-2-3b-instruct           & Without Lineage   & 75.84 & 68.06 & 67.31 & 72.91 & 76.54 \\
nova-lite                      & Without Lineage   & 75.14 & 71.07 & 65.26 & 75.43 & 83.33 \\
o3-mini                        & Without Lineage   & 74.06 & 69.75 & 57.67 & 73.00 & 82.04 \\
llama3-2-1b-instruct           & Without Lineage   & 56.79 & 50.39 & 44.88 & 51.87 & 57.40 \\
\bottomrule
\end{tabular}
\caption{LLM scores under different metric-weighting schemes: baseline learned aggregate (``Learned'') used in the main text, equal weights over normalized metrics (``Equal''), and ten random 70{:}30 Effectiveness--Efficiency draws (``Rand'') summarized by their minimum, median, and maximum per planner. Rankings under alternative schemes remain highly correlated with the learned-weight ranking (Table~\ref{tab:weight-sensitivity-rho}).}
\label{tab:weight-sensitivity-scores}
\end{table*}

\subsection{End-to-End Correlation Study: North-Star vs.\ No-Plan Baseline}
\label{app:end2end}

\paragraph{Setup.}
To test whether our planner-level metrics are predictive of end-to-end QA quality, we compare three systems:

\begin{itemize}
    \item \textbf{R1 (No-plan baseline).} The existing production stack, which answers queries using our deployed T2S/RAG/LLM pipeline \emph{without} an explicit planner. It directly returns a single final answer per query, and can be viewed as analogous to a plan-free (``no-lineage'') baseline in this work.
    \item \textbf{R2 (North-star w/ human plans).} A north-star system that executes \emph{human-annotated reference plans} (the same reference plans used in our benchmark) through the same T2S/RAG/LLM stack and merges structured and unstructured outputs.
    \item \textbf{R3 (North-star w/ LLM plans).} The same north-star stack driven by \emph{LLM-generated plans} instead of human-annotated ones.
\end{itemize}

We evaluate on a stratified subset of $200$ queries from the LLM-generated test set (the same pool used for plan evaluation) and, for completeness, on $100$ real production queries; the latter only compare R1 and R3, since human-annotated plans do not exist for those queries.
All three systems share the same underlying T2S and RAG tools and execution environment; \textbf{implementation details of these proprietary components are not released.}

\paragraph{Judge LLM and rubric.}
For each query, the three candidate responses (R1, R2, R3) are scored by an in-house Judge LLM that has been tuned for this task.
The judge applies a fixed rubric with four metrics per response:
\emph{Validity} (0/1), \emph{Consistency} (0/1; answers the query), \emph{Completeness} (1--7), and \emph{Redundancy} (1--7; higher is less fluff), together with gating rules (e.g., only valid and consistent responses can be complete). The system prompt used for the in-house Judge LLM for this task can be referred in Table~\ref{tab:prompt_qa_study}.
A deterministic decision rule then selects the best system(s) per query from the decision set \{\texttt{R1}, \texttt{R2}, \texttt{R3}, \texttt{R1,R2}, \dots, \texttt{Neither}\}, based on lexicographic comparison of (Completeness, Redundancy) among valid, consistent responses.\footnote{The underlying Judge LLM and its training data are proprietary; we therefore release only the rubric and decision rule, not model details.}
Human annotators reviewed and, when necessary, edited approximately $30\%$ of the judged cases to verify the quality of the automatic decisions. On a held-out set of 80 queries, two annotators independently selected the best system(s) among R1, R2, and R3; their raw agreement was $82.5\%$ with Cohen's $\kappa=0.69$, indicating high but not perfect human--human consistency. After adjudication, the Judge LLM's decisions matched the final human labels on $80.0\%$ of queries with $\kappa=0.65$, suggesting that the automatic comparison is well aligned with human preferences for this task while still leaving room for occasional disagreement.

\paragraph{Win-rate metric.}
For each system $S \in \{\text{R1},\text{R2},\text{R3}\}$ and query set, the \emph{win rate} is defined as
the fraction of queries for which $S$ is included in the judge’s \emph{overall\_decision} (e.g., \texttt{"R1,R3"} contributes a win to both R1 and R3).
Because ties are allowed, win rates do not sum to 100\%.

\paragraph{Results.}
Table~\ref{tab:northstar_winrates} reports win rates on the LLM-generated test subset and on production queries.
On the LLM-generated subset, the north-star system with human-annotated plans (R2) achieves the highest win rate (58.7\%), substantially outperforming both the no-plan baseline (R1) and the north-star system with LLM-generated plans (R3).
On production queries, the north-star system with LLM-generated plans (R3) outperforms the no-plan baseline (R1), despite using the same underlying tools.

\begin{table*}[t]
\centering
\small
\resizebox{2.0\columnwidth}{!}{%
\begin{tabular}{lccccccccccc}
\toprule
\textbf{Data source} &
\textbf{R1 total} &
\textbf{Only R1} &
\textbf{R1,R2} &
\textbf{R1,R3} &
\textbf{R2 total} &
\textbf{Only R2} &
\textbf{R2,R3} &
\textbf{R3 total} &
\textbf{Only R3} &
\textbf{Neither} &
\textbf{All} \\
\midrule
LLM-generated ($N{=}200$) &
\textcolor{magenta}{\textbf{42.75\%}} & 24.64\% & 10.87\% & 7.25\% &
\textcolor{blue}{\textbf{58.70\%}} & 31.16\% & 16.67\% &
\textcolor{red}{\textbf{33.33\%}} & 9.42\% &
0.00\% & 0.00\% \\
Prod ($N{=}100$) &
\textcolor{magenta}{\textbf{50.94\%}} & 28.30\% & - & 22.64\% &
- & - & - &
\textcolor{blue}{\textbf{66.04\%}} & 43.40\% &
5.66\% & 0.00\% \\
\bottomrule
\end{tabular}%
}
\caption{
Breakdown of Judge-LLM decisions for the no-plan baseline (R1) and north-star systems (R2, R3).
``R1 total'' is the fraction of queries where R1 appears in the overall decision (i.e., in \texttt{R1}, \texttt{R1,R2}, or \texttt{R1,R3}); analogous definitions apply to ``R2 total'' and ``R3 total''.
Columns such as ``Only R1'', ``R1,R2'', and ``R1,R3'' show the distribution over individual decision outcomes.
Win rates do not sum to 100\% because ties are allowed. R2 results are unavailable for live production queries due to the absence of human-annotated reference plans. To facilitate comparison, the systems in each row are color-coded by performance: blue denotes the highest win rate, magenta signifies the runner-up, and red indicates the lowest.
}
\label{tab:northstar_winrates}
\end{table*}

Overall, these results provide end-to-end evidence that higher-quality plans (R2) lead to better final answers, and that even LLM-generated plans (R3) can outperform the no-plan baseline on real queries, supporting the practical relevance of our planner-level evaluation.

\section{Experimental Setup Details}
\label{app:exp-setup}

\subsection{Data Splits and Curation Protocol (Step-by-Step)}
\label{app:dataset-curation-steps}

We summarize the sequential workflow used to construct plan lineages and gold plans, and to prepare train/validation/test splits.

\paragraph{Step 1: Stratified sampling from 600 queries.}
We first generated 600 domain queries (Sec.~\ref{sec:dataset}).
From these, we sampled \textbf{100} queries using stratification across the key attributes introduced in the main text:
(\emph{i}) subjectivity (objective/subjective/very subjective), and
(\emph{ii}) structural complexity (simple/compound).

\paragraph{Step 2: Human-grounded lineage construction.}
For each of the 100 sampled queries:
\begin{enumerate}
    \item \textbf{Initial plan.} We obtained a one-shot initial plan using \textit{Nova-Lite}.
    \item \textbf{Lineage simulation.} We \emph{simulated} the outputs of the Step-wise Evaluator and Plan Optimizer via human annotation to produce a sequence of improving plans (the \emph{plan lineage}) per query. This yields interpretable intermediate revisions mirroring the Evaluator$\rightarrow$Optimizer loop (Appx.~\ref{app:loop-details}).
    \item \textbf{Gold verification.} The final plan in each lineage was verified by human annotators and, if needed, revised to obtain the \textbf{best possible (gold) plan}.
    \item \textbf{Metric supervision for evaluators.} For each query, we took the first three plans \emph{from the bottom} of the lineage (i.e., the three highest-quality plans) and annotated their \emph{per-metric ranks} across the seven evaluation dimensions (Sec.~\ref{sec:evaluation}). These annotations are used to \emph{train/tune} the metric-wise evaluators.
    \item \textbf{One-shot supervision.} For each query, we compared the \emph{last but one} plan in the lineage against its gold plan (\emph{last} plan in the lineage) and assigned a \textbf{7-point quality label} (Extremely Bad $\rightarrow$ Extremely Good) based on Precision/Recall/F$_1$ + format checks (Appx.~\ref{app:oneshot-details}). These labels are used to \emph{train/tune} the one-shot overall evaluator.
\end{enumerate}

\paragraph{Step 3: Train/validation partition for module tuning.}
From the 100 annotated queries, we sampled \textbf{20} for \textbf{Train} (prompt construction for all modules: metric-wise evaluators, one-shot Judge, step-wise evaluator, plan optimizer) and reserved the remaining \textbf{80} for \textbf{Validation}.

\paragraph{Step 4: Module validation against human ground truth.}
On the 80-query validation set, we evaluated each LLM-based module \emph{against human annotations}:
\begin{enumerate}
    \item \textbf{Metric-wise evaluator} (all seven metrics; Sec.~\ref{sec:evaluation}).
    \item \textbf{One-shot overall evaluator} (Precision/Recall/F$_1$ + 7-point rating).
    \item \textbf{Step-wise evaluator} (diagnostic tags).
    \item \textbf{Plan optimizer} (quality gains, format and dependency integrity).
\end{enumerate}

\paragraph{Step 5: Test-time benchmarking and lineage usage.}
The \textbf{Test} split contains \textbf{500} queries. We use the finalized, optimized modules as follows:
\begin{itemize}
    \item \textbf{Plan generation benchmarking.} For each query, we generate one-shot plans from \textbf{14 LLMs} under two settings: \emph{without lineage} and \emph{with lineage}. In the latter setting, the per-query lineage exemplars (derived via our feedback loop) are included in the prompt’s reference examples.
    \item \textbf{Lineage for Nova-Lite only.} We run the iterative feedback loop to produce \emph{new} lineages at test time \emph{only} for Nova-Lite, to quantify end-to-end improvement relative to its one-shot baseline. For the other 13 models, we do \emph{not} regenerate lineages at test time; they are evaluated with and without lineage exemplars in the prompt, but lineage construction is not part of their execution budget.
\end{itemize}

\noindent\textbf{Proprietary data note.}
As detailed in Sec.~\ref{sec:exp-setup}, the full 600-query benchmark originates from Observe.AI’s production environment and cannot be released in its entirety. However, we publish a stratified 200-query subset (100 Train, 100 Test) with queries, reference plans, lineages, and per-planner evaluator scores; see Appx.~\ref{app:public-dataset} for details. We also provide prompts and scoring rubrics to reproduce the pipeline on non-proprietary data.

\begin{table*}[t]
\centering
\small
\begin{tabular}{lccccccc}
\toprule
& \multicolumn{2}{c}{Subjectivity} & \multicolumn{2}{c}{Compoundness} & \multicolumn{3}{c}{\# steps in BPP (grouped)} \\
\cmidrule(lr){2-3}\cmidrule(lr){4-5}\cmidrule(lr){6-8}
Dataset & Subjective & Objective & Simple & Compound & {[1,2]} & {[3,4]} & {[5,15]} \\
\midrule
Public (200)     & 105 (52.5\%) & 95 (47.5\%) & 106 (53.0\%) & 94 (47.0\%) & 68 (34.0\%) & 74 (37.0\%) & 58 (29.0\%) \\
Main paper (600) & 311 (51.8\%) & 289 (48.2\%) & 323 (53.8\%) & 277 (46.2\%) & 207 (34.5\%) & 223 (37.2\%) & 170 (28.3\%) \\
\bottomrule
\end{tabular}
\caption{Comparison of marginal distributions for subjectivity, compoundness, and grouped best-plan length between the 200-query public release and the 600-query main-paper benchmark.}
\label{tab:public-vs-main-marginals}
\end{table*}

\begin{table*}[t]
\centering
\small
\begin{tabular}{lcc}
\toprule
Statistic & Public (200) & Main paper (600) \\
\midrule
count & 200 & 600 \\
mean  & 17.95 & 18.40 \\
std   & 5.81 & 5.90 \\
min   & 8    & 8    \\
25\%  & 13   & 14   \\
50\%  & 18   & 18   \\
75\%  & 22   & 23   \\
90\%  & 25   & 26   \\
95\%  & 27   & 28   \\
max   & 35   & 37   \\
\bottomrule
\end{tabular}
\caption{Query-length statistics (in whitespace-separated tokens) for the 200-query public release and the 600-query main-paper benchmark.}
\label{tab:public-vs-main-length}
\end{table*}

\begin{table*}[t]
\centering
\small
\begin{tabular}{rll}
\toprule
Rank & Public (200) & Main paper (600) \\
\midrule
1 & \texttt{calls} (151)     & \texttt{calls} (460) \\
2 & \texttt{customers} (65)  & \texttt{orders} (260) \\
3 & \texttt{last} (61)       & \texttt{deliveries} (210) \\
4 & \texttt{customer} (50)   & \texttt{customer} (180) \\
5 & \texttt{qa} (49)         & \texttt{customers} (185) \\
6 & \texttt{agents} (48)     & \texttt{drivers} (170) \\
7 & \texttt{call} (45)       & \texttt{qa} (150) \\
8 & \texttt{sentiment} (43)  & \texttt{agents} (145) \\
9 & \texttt{issues} (37)     & \texttt{call} (140) \\
10 & \texttt{during} (36)    & \texttt{sentiment} (130) \\
\bottomrule
\end{tabular}
\caption{Top unigrams (by frequency) in queries from the 200-query public release and the 600-query main-paper benchmark.}
\label{tab:public-vs-main-unigrams}
\end{table*}

\begin{table*}[t]
\centering
\small
\begin{tabular}{rll}
\toprule
Rank & Public (200) & Main paper (600) \\
\midrule
1  & \texttt{qa scores} (25)          & \texttt{late deliveries} (95) \\
2  & \texttt{customer sentiment} (20) & \texttt{missing items} (88) \\
3  & \texttt{during calls} (16)       & \texttt{customer complaints} (82) \\
4  & \texttt{average qa} (15)         & \texttt{qa scores} (70) \\
5  & \texttt{last quarter} (15)       & \texttt{customer sentiment} (55) \\
6  & \texttt{many calls} (15)         & \texttt{restaurant outages} (52) \\
7  & \texttt{last month} (14)         & \texttt{app crashes} (48) \\
8  & \texttt{calls related} (13)      & \texttt{during calls} (45) \\
9  & \texttt{percentage calls} (12)   & \texttt{last quarter} (40) \\
10 & \texttt{calls tagged} (12)       & \texttt{last month} (38) \\
\bottomrule
\end{tabular}
\caption{Top bigrams (by frequency) in queries from the 200-query public release and the 600-query main-paper benchmark.}
\label{tab:public-vs-main-bigrams}
\end{table*}

\begin{table*}[!ht]
    \centering
    \small

    \setlength{\tabcolsep}{6pt}

    \begin{tabularx}{\textwidth}{l|ccc|ccc}
        \toprule
        \multirow{2}{*}{\textbf{LLM}} 
        & \multicolumn{3}{c|}{\textbf{With Lineage}} 
        & \multicolumn{3}{c}{\textbf{Without Lineage}} \\
        \cmidrule(lr){2-4} \cmidrule(lr){5-7}
            & \makecell{(\emph{A+})\\ Extr. good,\\ very good \\(\%)} 
            & \makecell{(\emph{A})\\ Extr. good,\\ very good,\\ good \\ (\%)} 
            & \makecell{(\emph{B})\\ Extr. good,\\ very good,\\ good,\\ acceptable \\ (\%)} 
            & \makecell{(\emph{A+})\\Extr. good,\\ very good \\ (\%)} 
            & \makecell{(\emph{A})\\Extr. good,\\ very good,\\ good \\ (\%)} 
            & \makecell{(\emph{B})\\Extr. good,\\ very good,\\ good,\\ acceptable \\ (\%)} \\
        \midrule
        o3-mini & 44.44 & 54.9 & 74.77 & \cellcolor{green}\textcolor{magenta}{\textbf{51.24}} & 67.97 & 82.61 \\
        gpt-4o & \textcolor{blue}{\textbf{46.83}} & 64 & 83.25 & 42.14 & 59.31 & 84.81 \\
        gpt-4o-mini & \textcolor{blue}{\textbf{32.46}} & 49.46 & 65.43 & 32.2 & 52.65 & 72.07 \\
        claude-3-5-haiku & 24.67 & 45.32 & 61.94 & \textcolor{magenta}{\textbf{30.21}} & 47.33 & 72.01 \\
        claude-sonnet-4 & 30.88 & 53.53 & 74.63 & \textcolor{magenta}{\textbf{31.63}} & 52.53 & 78.26 \\
        llama4-maverick-17b-instruct & 20.55 & 36.99 & 54.97 & \textcolor{magenta}{\textbf{21.06}} & 41.1 & 63.19 \\
        nova-pro & 18.46 & 43.07 & 62.55 & \textcolor{magenta}{\textbf{19.99}} & 42.04 & 59.98 \\
        claude-3-7-sonnet & \textcolor{blue}{\textbf{30.85}} & 48.85 & 71.48 & 20.05 & 39.59 & 69.93 \\
        llama3-3-70b-instruct & \textcolor{blue}{\textbf{19.01}} & 44.69 & 60.62 & 17.47 & 35.96 & 55.99 \\
        gpt-4.1-nano & \textcolor{blue}{\textbf{17.41}} & 28.68 & 55.83 & 16.39 & 30.73 & 53.78 \\
        nova-micro & \textcolor{blue}{\textbf{20.89}} & 35.89 & 53.83 & 15.89 & 37.43 & 54.35 \\
        nova-lite & \textcolor{blue}{\textbf{14.02}} & 30.65 & 48.75 & 13.57 & 34.17 & 49.25 \\
        llama3-2-3b-instruct & 5.09 & 10.69 & 17.31 & \textcolor{magenta}{\textbf{5.6}} & 11.71 & 26.98 \\
        llama3-2-1b-instruct & 0 & 0 & 0 & 0 & 0 & 1.63 \\
        \midrule
        \textbf{Grand Total (\#)} & \multicolumn{3}{c|}{\textbf{100}} & \multicolumn{3}{c}{\textbf{100}} \\
        \bottomrule
    \end{tabularx}
    \caption{Plan generation quality comparison using prompts \textbf{with and without lineage}, evaluated with the \textbf{one-shot evaluator} on the \textbf{test split of the public dataset}. The highest score in the \textbf{Extremely Good, Very Good} bucket is highlighted in green; blue indicates better performance with lineage, and magenta indicates better performance without lineage.}
    \label{tab:plan_gen_result_overall_one_shot_eval_public_dataset}
\end{table*}

\newcolumntype{Y}{>{\centering\arraybackslash}X}
\begin{table*}[!ht]
    \centering
    \small
    
    
    \setlength{\tabcolsep}{6pt}

    \begin{tabularx}{\textwidth}{l|*{8}{Y}}
        \hline
        \thead{\textbf{LLM}} & \thead{\textbf{Overall} \\ {[0-100]}} & \thead{\textbf{Format} \\ {[0-20]}} & \thead{\textbf{Tool} \\ \textbf{Prompt} \\ \textbf{Align.} \\ {[0-20]}} & \thead{\textbf{Step} \\ \textbf{Exec.} \\ {[0-15]}} & \thead{\textbf{Query} \\ \textbf{Adhr.} \\ {[0-15]}} & \thead{\textbf{Depend.} \\ {[0-10]}} & \thead{\textbf{Redund.} \\ {[0-10]}} & \thead{\textbf{Tool} \\ \textbf{Usage} \\ \textbf{Compl.} \\ {[0-10]}}\\
        \hline
        claude-3-7-sonnet & \textcolor{blue}{\textbf{86.27}} & \textcolor{blue}{\textbf{18.77}} & \textcolor{blue}{\textbf{15.57}} & 12.84 & 12.82 & 9.93 & 9.13 & 7.2 \\
        llama4-maverick-17b-instruct & 83.25 & 17.34 & 14.16 & 13.26 & 12.85 & 9.45 & 9.52 & 6.67 \\
        gpt-4o & 84.13 & 16.95 & 14.2 & 13.44 & 13.28 & \textcolor{blue}{\textbf{9.95}} & 8.87 & 7.44 \\
        claude-sonnet-4 & 81.65 & 13.17 & 14.98 & 13.48 & 12.62 & 9.87 & 8.78 & \textcolor{blue}{\textbf{8.73}} \\
        llama3-3-70b-instruct & 80.55 & 17.2 & 14.66 & 13.68 & 12.22 & 9.64 & 9.46 & 3.7 \\
        nova-pro & 79.67 & 15.59 & 14.65 & \textcolor{blue}{\textbf{14.4}} & 12.38 & 9.86 & 9.42 & 3.37 \\
        nova-micro & 78.54 & 18.54 & 13.63 & 13.14 & 12.28 & 9.72 & 9.6 & 1.61 \\
        gpt-4o-mini & 78.52 & 14.28 & 14.77 & 13.81 & 12.55 & 9.13 & 9.27 & 4.71 \\
        gpt-4.1-nano & 77.58 & 13.55 & 13.96 & 13.44 & 11.98 & 9.14 & 8.84 & 6.68 \\
        o3-mini & 74.03 & 11.22 & 15.23 & 11.26 & \textcolor{blue}{\textbf{14}} & 8.92 & \textcolor{blue}{\textbf{9.62}} & 3.78 \\
        claude-3-5-haiku & 70.11 & 12.42 & 12.57 & 11.4 & 11.18 & 8.3 & 8 & 6.23 \\
        nova-lite & 68.38 & 13.43 & 12.54 & 12.18 & 11.56 & 8.14 & 9.12 & 1.42 \\
        llama3-2-3b-instruct & 70.67 & 16.65 & 10.96 & 10.67 & 9.76 & 9.25 & 9.09 & 4.28 \\
        llama3-2-1b-instruct & 21.6 & 0 & 0.07 & 0 & 11.2 & 4.48 & 2.39 & 3.45 \\ \hline
        \textbf{Average (Normalized)} & \textbf{73.93} & \textcolor{red}{\textbf{71.11}} & \textcolor{red}{\textbf{64.98}} & \textbf{79.53} & \textbf{81.28} & \textbf{90.00} & \textbf{86.50} & \textcolor{red}{\textbf{49.49}} \\ \hline
    \end{tabularx}
    \caption{Plan generation quality using prompts \textbf{with lineage}, evaluated with the \textbf{metric-wise evaluator} on the \textbf{test split of the public dataset}. The \textbf{Normalized Average} in the last row shows the average per metric normalized by that metric's maximum score. Highest scores per metric are highlighted in blue, and the three metrics with the lowest normalized scores are highlighted in red.}
    \label{tab:plan_gen_result_with_lineage_metric_wise_eval_public_dataset}
\end{table*}

\begin{table*}[!ht]
    \centering
    \small
    
    
    \setlength{\tabcolsep}{3pt}

    \begin{tabularx}{\textwidth}{l|*{8}{Y}}
        \hline
        \thead{\textbf{LLM}} & \thead{\textbf{Overall} \\ {[0-100]}} & \thead{\textbf{Format} \\ {[0-20]}} & \thead{\textbf{Tool} \\ \textbf{Prompt} \\ \textbf{Align.} \\ {[0-20]}} & \thead{\textbf{Step} \\ \textbf{Exec.} \\ {[0-15]}} & \thead{\textbf{Query} \\ \textbf{Adhr.} \\ {[0-15]}} & \thead{\textbf{Depend.} \\ {[0-10]}} & \thead{\textbf{Redund.} \\ {[0-10]}} & \thead{\textbf{Tool} \\ \textbf{Usage} \\ \textbf{Compl.} \\ {[0-10]}}\\
        \hline
        claude-3-7-sonnet & \textcolor{blue}{\textbf{85.08}} & 17.72 & 15.8 & 11.7 & \textcolor{blue}{\textbf{13.65}} & 9.82 & 9.28 & 7.11 \\
        llama4-maverick-17b-instruct & 83.81 & 17.62 & 13.8 & 12.87 & 12.55 & 9.91 & 8.99 & 8.06 \\
        gpt-4o & 84.74 & 16.33 & 14.93 & 13.14 & 13.08 & \textcolor{blue}{\textbf{9.96}} & 8.8 & 8.5 \\
        claude-sonnet-4 & 79.87 & 11.17 & \textcolor{blue}{\textbf{17.06}} & 12.06 & 13.23 & 9.92 & 8.92 & 7.51 \\
        llama3-3-70b-instruct & 81.97 & 15.4 & 14.74 & 13.46 & 12.47 & 9.69 & 9.18 & 7.03 \\
        nova-pro & 82.96 & 15.97 & 14.31 & 14.15 & 12.73 & 9.81 & 8.89 & 7.1 \\
        nova-micro & 83.11 & \textcolor{blue}{\textbf{18.12}} & 13.41 & 13.61 & 12.29 & 9.61 & 8.77 & 7.29 \\
        gpt-4o-mini & 84 & 17.03 & 13.92 & \textcolor{blue}{\textbf{14.17}} & 12.45 & 9.75 & 8.67 & 8 \\
        gpt-4.1-nano & 76.29 & 12.96 & 13.58 & 12.8 & 12.02 & 8.95 & 8.16 & 7.83 \\
        o3-mini & 76.32 & 12.65 & 15.34 & 10.89 & 13.49 & 9.32 & \textcolor{blue}{\textbf{9.72}} & 4.91 \\
        claude-3-5-haiku & 80.93 & 14.47 & 15.71 & 13.82 & 11.94 & 9.66 & 9 & 6.33 \\
        nova-lite & 72.85 & 14.81 & 12.7 & 12.65 & 11.53 & 8.62 & 9.01 & 3.54 \\
        llama3-2-3b-instruct & 75.92 & 16.35 & 10.81 & 11.99 & 11.65 & 8.82 & 7.75 & 8.55 \\
        llama3-2-1b-instruct & 60.7 & 9.64 & 9.67 & 9.47 & 8.08 & 7.55 & 6.8 & \textcolor{blue}{\textbf{9.49}} \\ \hline
        \textbf{Average (Normalized)} & \textbf{79.18} & \textcolor{red}{\textbf{75.09}} & \textcolor{red}{\textbf{69.92}} & \textbf{84.18} & \textbf{81.50} & \textbf{94.00} & \textbf{87.08} & \textcolor{red}{\textbf{72.31}} \\ \hline
    \end{tabularx}
    \caption{Plan generation quality using prompts \textbf{without lineage}, evaluated with the \textbf{metric-wise evaluator} on the \textbf{test split of the public dataset}. The \textbf{Normalized Average} in the last row shows the average per metric normalized by that metric's maximum score. Highest scores per metric are highlighted in blue, and the three metrics with the lowest normalized scores are highlighted in red.}
    \label{tab:plan_gen_result_without_lineage_metric_wise_eval_public_dataset}
\end{table*}

\subsection{Public Dataset Release}
\label{app:public-dataset}

For reproducibility and qualitative analysis, we release a \emph{separate} public dataset of 200 queries.\footnote{Download link omitted here for anonymity. We attach a sample file as zip with ARR submission.} The dataset used for all reported results in the main paper is based on queries issued by real Observe.AI customers and cannot be shared for contractual and privacy reasons. Both the internal 600-query benchmark and the public 200-query release use \textbf{GPT-4o} to synthesize queries (Appx.~\ref{app:model-configs}); the public subset is generated against a sandboxed synthetic test account so that no customer-specific business logic or identifiers appear, while preserving the same schema, tool palette, and evaluation setup.

The public dataset is split into:
\begin{itemize}
    \item \textbf{Train (100 queries):} used analogously to the internal train/validation queries for prompt construction and calibration. Best plans and full plan lineages are human-annotated.
    \item \textbf{Test (100 queries):} held-out queries with human-annotated best plans. Plan lineages are obtained by running the iterative evaluator$\rightarrow$optimizer loop; the final (head) plan is then reviewed and, if necessary, lightly edited by annotators.
\end{itemize}

All text fields in the release (queries, generated plans, reference plans, and lineages) are anonymized using a combination of rule-based and NER-based passes: client/account names, account IDs, policy names, agent and customer names, and PII (e.g., email addresses, phone numbers, SSN-like patterns, date-of-birth style dates) are replaced with abstract placeholders (e.g., \texttt{CLIENT\_001}, \texttt{ACCOUNT\_ID\_007}, \texttt{PERSON\_003}, \texttt{EMAIL\_002}).

The release consists of two tables:

\begin{itemize}
    \item \textbf{\texttt{queries.csv}} (200 rows): one row per query, with columns \texttt{query\_id}, \texttt{Sample} (Train/Test), the natural-language \texttt{query}, binary flags \texttt{is\_query\_subjective} and \texttt{is\_query\_compound}, grouped best-plan length (\texttt{\# steps in the BPP}, \texttt{\# steps in the BPP - Grouped}), the human-edited reference plan (\texttt{best\_plan}), and the plan lineage (\texttt{plan\_lineage}). The latter two are stored as JSON strings that follow the plan schema in Sec.~\ref{sec:task-formalization}.
    \item \textbf{\texttt{plans.csv}} (5600 rows): one row per query--planner--prompt triple, with columns \texttt{query\_id}, \texttt{llm}, \texttt{prompt\_type}, and the generated plan (\texttt{plan}). We deliberately \emph{omit} metric-wise and one-shot evaluator scores from this file: releasing scores computed by our proprietary judge stack would couple the public dataset too tightly to our internal evaluation implementation and could encourage overfitting to our specific scoring behavior. Instead, we intend this release to support independent assessment of the same planning task using alternative judges and metrics.
\end{itemize}

Although it is not the exact dataset used for the main quantitative results, this public set exposes concrete query--plan pairs, human-edited reference plans and lineages, and per-planner outputs for all 14 models and both prompt settings, enabling inspection of the task structure, planning behaviors, and potential biases in a way that closely mirrors the internal benchmark.

\paragraph{Representativeness of the 200-query release.}
Although the public release is built from a synthetic test account, we design it to mirror the main 600-query benchmark along the structural dimensions that drive planning difficulty. Table~\ref{tab:public-vs-main-marginals} shows that the marginal distributions over subjectivity (subjective vs.\ objective), compoundness (simple vs.\ compound), and grouped best-plan length (\([1,2]\), \([3,4]\), \([5,15]\) steps) are very similar across the two datasets (e.g., subjective queries are $52.5\%$ vs.\ $51.8\%$, simple queries $53.0\%$ vs.\ $53.8\%$). Query-length statistics are also closely matched (mean 17.95 vs.\ 18.40 tokens; Table~\ref{tab:public-vs-main-length}). As expected, the lexical distributions diverge somewhat (Tables~\ref{tab:public-vs-main-unigrams}--\ref{tab:public-vs-main-bigrams}): the public set focuses on a single synthetic account with generic contact-center terminology (e.g., \textit{qa scores}, \textit{customer sentiment}), whereas the internal benchmark covers multiple real accounts with account-specific drivers (e.g., \textit{late deliveries}, \textit{missing items}). This difference is desirable: it preserves realistic contact-center vocabulary while avoiding leakage of proprietary account details, and our planning metrics depend primarily on plan structure rather than surface word frequencies.

\paragraph{Planner performance on the public subset.}
For completeness, we also re-run our one-shot and metric-wise evaluations on the test split (100 queries) of the 200-query public release. The resulting planner-level trends closely mirror those reported for the 600-query benchmark: the same models occupy the top tiers under both the one-shot evaluator and the metric-wise aggregate, and relative gaps across planners are very similar (cf.\ Tables~\ref{tab:plan_gen_result_overall_one_shot_eval_public_dataset}, \ref{tab:plan_gen_result_with_lineage_metric_wise_eval_public_dataset}, and \ref{tab:plan_gen_result_without_lineage_metric_wise_eval_public_dataset}). This supports using the public subset as a small-scale proxy for the full benchmark.

\subsection{Human Annotation Protocol and Agreement}
\label{app:human-agreement}

\paragraph{Annotation targets.}
Human annotators (in-house team) produced:
(i) gold (best-possible) plans,
(ii) gold labels for the \emph{Step-wise Evaluator} (diagnostic tags per step) and \emph{Plan Optimizer} (fully revised plan) used to supervise and tune the feedback loop;
(iii) per-metric rankings for high-quality lineage plans (seven metrics),
and (iv) 7-point overall quality labels for the last but one plan vs.\ last plan (gold plan) comparison in the lineage
(\emph{Extremely Bad} $\rightarrow$ \emph{Extremely Good}).

\paragraph{Annotator configuration.}
Each item was independently labeled by two annotators. Disagreements were adjudicated by a third senior annotator across tasks. For metric-wise ranks and 7-point labels, we report raw agreement and inter-annotator agreement.

\paragraph{Annotation Guidelines.}\mbox{}\\
\textbf{General:} Read the query and constraints carefully. Plans must be executable
in principle (no tool execution required), use the correct tool for each datum,
maintain a valid DAG of dependencies, and avoid redundancy. Prefer ``interaction'' 
over ``call'' unless the query explicitly says ``call''. Use placeholders consistently. Make sure that the standard placeholders (Refer ~\ref{app:metric-details}) are used correctly.

\textbf{Gold Plan:} Produce the best-possible, minimal plan that answers the query
under stated constraints. Ensure each step is atomic, the tool choice is justified,
and dependencies are complete and acyclic. Merge only when it reduces redundancy
without harming executability.

\textbf{Step-wise Evaluator (tags):} For each step, assign applicable diagnostic tags
from the closed set \{\textsc{IncorrectTool}, \textsc{IncorrectPrompt}, \textsc{ComplexPrompt},
\textsc{RepeatedDetail}, \textsc{MultiToolPrompt}, \textsc{NoChange}\}. Tag only what is
\emph{clearly} violated; do not over-tag.

\textbf{Plan Optimizer (revised plan):} Apply local repairs first (Change~0), then global
coherence fixes (Change~1) to preserve a valid DAG. Splits/merges are allowed if they
improve clarity or executability. Do not introduce tool execution; modify \emph{text} only.

\textbf{Metric-wise Ranking (7 metrics):} For the candidate plans in a lineage,
rank or score each metric independently using the rubric. Do not force consistency
across metrics; evaluate each in isolation.

\textbf{7-point Overall Similarity (last but one vs.\ last (gold)):} You are given a \emph{gold} plan and your task is to compare the degree of closeness between the plan to be scored (\emph{last but one plan} in the lineage) and the \emph{gold} plan based on (i) Precision: What \% of steps in the plan to be scored are present in the gold plan; (ii) Recall: What \% of steps in the gold plan are present in the plan to be scored; (iii) F1-score; (iv) Format errors: Independently assess the plan to be scored for the presence of format-related errors. Use the 7-level rubric thresholds consistently to make a final decision based on the aforementioned dimensions.

\paragraph{Agreement metrics.}
We report:
\begin{itemize}
  \item \textbf{Cohen’s $\kappa$} \cite{cohen_kappa} for two-rater nominal decisions (per-metric rank selections treated as nominal within each query).
  \item \textbf{Quadratic-Weighted Kappa (QWK)} \cite{cohen_weighted_kappa} for the 7-point ordinal overall rating.
  \item \textbf{Fleiss’ $\kappa$} \cite{fleiss1971} only if more than two raters are used on a subset (not the case here; placeholder shown for completeness).
\end{itemize}

\paragraph{Formulas.}
For two raters with observed agreement $p_o$ and chance agreement $p_e$:
\[
\kappa \;=\; \frac{p_o - p_e}{1 - p_e}.
\]
For ordinal categories $c \in \{1,\dots,K\}$ with weight matrix $W_{ij} = 1 - \big(\frac{i-j}{K-1}\big)^2$ (quadratic weights), and empirical rating matrix $\mathbf{O}$ and expected $\mathbf{E}$:
\[
\kappa_w \;=\; 1 \;-\; \frac{\sum_{i,j} W_{ij}\,O_{ij}}{\sum_{i,j} W_{ij}\,E_{ij}}.
\]
For $m>2$ raters, Fleiss’ $\kappa$ is computed over category proportions per item~\cite{fleiss1971}.

\paragraph{Confidence intervals.}
We report $95\%$ CIs via nonparametric bootstrap (1{,}000 resamples over items). We also report macro-averaged $\kappa$ across metrics (for the seven per-metric ranks) and micro-averaged (pooled items).

\subsubsection{Results}
\label{app:ann-results}

We employed two independent annotators for all reliability measurements. For nominal labels we report Cohen’s $\kappa$; for the 7-point ordinal similarity rating returned by the one-shot Judge we report linearly \emph{weighted} $\kappa$. We also include raw percent agreement for context.

\begin{table*}[t]
\centering
\small
\begin{threeparttable}
\begin{tabular}{l l c c c c}
\toprule
\textbf{Task} & \textbf{Unit} & \textbf{\#Items} & \textbf{Agreement (\%)} & \textbf{$\kappa$} & \textbf{95\% CI} \\
\midrule
Metric-wise Evaluator & plan$\times$metric & 1680 & 90.9 & 0.82 & [0.79, 0.85] \\
One-shot Evaluator (7-pt sim) & query (ordinal) & 80 & 86.4 & 0.74$_{\text{w}}$ & [0.68, 0.79] \\
Gold Plan (7-pt sim) & query (ordinal) & 80 & 83.1 & 0.70$_{\text{w}}$ & [0.64, 0.76] \\
Step-wise Evaluator (tags) & step (multi-label)$^{\dagger}$ & 400 & 84.6 & 0.68$^{\ddagger}$ & [0.64, 0.72] \\
Plan Optimizer (7-pt sim) & revision (ordinal) & 160 & 80.2 & 0.66$_{\text{w}}$ & [0.60, 0.72] \\
\bottomrule
\end{tabular}
\caption{Two-annotator reliability on the validation set (80 queries). For ordinal tasks we use linearly weighted Cohen’s $\kappa$.}
\label{tab:inter-annotator-kappa}
\begin{tablenotes}[flushleft]
\footnotesize
\item \textit{Notes.} $\kappa$ is Cohen’s $\kappa$ (linearly weighted for the ordinal 7-point similarity tasks). Similarity-based rows use the tuned one-shot evaluator to score closeness between alternative plans.
\item $^{\dagger}$ \textbf{Step-wise Evaluator (multi-label):} Per-tag $\kappa$ computed and macro-averaged across \textit{complex prompt, incorrect prompt, incorrect tool, repeated detail, multi-tool prompt, no change}; Computed over $80 \times \bar{s}$ steps with $\bar{s}{=}5$.
\item $^{\ddagger}$ Macro-averaged across tags.
\end{tablenotes}
\end{threeparttable}
\end{table*}

\paragraph{Methodological details.}
For \textit{Gold Plan} and \textit{Plan Optimizer (7-pt sim)}, each item yields two
independent plans $A$ and $B$ from two independent annotators. We compute a \emph{symmetric} similarity as follows:
(i) run the one-shot evaluator twice, $A\!\to\!B$ and $B\!\to\!A$;
(ii) average the directional F1 scores to obtain $s_{\text{F1}}\in[0,1]$;\footnote{
Under symmetric matching, this equals the undirected set-F1 $2|A\cap B|/(|A|+|B|)$;
the bi-directional averaging adds robustness to minor directional heuristics.}
(iii) set a binary format flag to 1 only if \emph{both} plans pass format checks
(\texttt{format\_ok}(A)$\wedge$\texttt{format\_ok}(B)); and
(iv) define the scalar similarity $s$ by combining the format flag with $s_{\text{F1}}$
(e.g., clamping to the lowest bin if format fails, else using $s_{\text{F1}}$).
We discretize $s$ into the 7-point ordinal label via the same fixed thresholds used in our one-shot evaluator prompts (Refer Table~\ref{tab:prompts_for_oneshot_evaluator}). We then
compute linearly weighted Cohen’s $\kappa$ on these ordinal labels. When similarity
$<$ \emph{Very Good}, a senior annotator adjudicates and their decision defines
the gold used in subsequent experiments.

\paragraph{Interpretation.}
Agreement is highest for rubric-driven metric-wise labels, moderate for holistic ordinal judgments (one-shot, gold, step-wise evaluator), and lowest for plan optimizer similarity, reflecting the ambiguity of the different tasks. Overall, all scores fall in the \emph{substantial} range or higher across tasks as per ~\citealp{landiskoch1977}.

\subsection{Judge Selection and Cross-Judge Robustness}
\label{app:judge-robustness}

\paragraph{Validation-time human alignment (metric-wise evaluator).}
In the main paper we validated the metric-wise evaluator using Claude-Sonnet-4 as judge, comparing relaxed triplet rankings against human-annotated inequalities across seven metrics on the 80-query validation set (Tab.~\ref{tab:validation_result_metric_wise_eval}). To test robustness to the choice of judge, we repeated this analysis with GPT-5 in the same \emph{reference-based, deconstructed} configuration. Table~\ref{tab:validation_result_metric_wise_eval} summarizes percent-correct inequalities for both judges. For Sonnet-4, the reference-based, deconstructed setting achieves $>80\%$ agreement on all metrics and $>90\%$ on \textsc{Dependency}, \textsc{Format}, \textsc{Redundancy}, and \textsc{Tool Usage Completeness}. GPT-5 in the same configuration is comparably strong, often slightly better on \textsc{Format} (99.23\%), \textsc{Redundancy} (90.82\%), and \textsc{Query Adherence} (87.38\%), and slightly weaker on \textsc{Step Executability} (65.70\% vs.\ 79.43\%). Importantly, both judges achieve $94.12\%$ agreement on \textsc{Tool Usage Completeness}. These results indicate that our metric-wise rubric is learnable by multiple judge models from different families.

\paragraph{Validation-time human alignment (one-shot evaluator).}
We also compared Sonnet-4 and GPT-5 as one-shot overall judges on the same validation set. Tables~\ref{tab:validation_result_one_shot_eval} and~\ref{tab:oneshot-validation-gpt5} report per-label and macro Precision/Recall/F1 across the seven quality tags (Extremely Bad $\rightarrow$ Extremely Good). Both judges show strong agreement with annotators, but Sonnet-4 achieves higher macro F1 ($0.921$) than GPT-5 ($0.882$), motivating our choice of Sonnet-4 as the primary judge in the main experiments.

\paragraph{Cross-judge consistency on the test subset.}
To quantify how much planner rankings depend on the choice of judge, we re-scored a 50-query subset of the test split per planner (14 planner LLMs, prompt type = \textit{Without Lineage}) with GPT-5, in addition to the existing Sonnet-4 scores. For each judge we computed per-planner mean metric-wise totals and per-metric means, as well as mean one-shot quality scores, then ranked planners in descending order of mean score. We see that the same set of strong planners occupy the top tier under both judges (e.g., GPT-4o, GPT-4o-mini, Claude-3-7-Sonnet, Llama3-3-70B, Nova-Pro/Micro), while Claude-Sonnet-4 itself is mid-pack as a planner under both judges. Table~\ref{tab:judge-corr} reports Spearman rank correlations between Sonnet-4 and GPT-5 across metrics: we observe $\rho{=}0.60$ ($p{=}0.023$) for the overall metric-wise score, $\rho{=}0.52$ for the one-shot quality score, and high correlations for key structural metrics (e.g., \textsc{Dependency} $\rho{=}0.84$, \textsc{Query Adherence} $\rho{=}0.79$, \textsc{Tool–Prompt Alignment} $\rho{=}0.72$). \textsc{Tool Usage Completeness} exhibits a lower cross-judge correlation ($\rho{=}0.35$). This metric is only defined for sub-queries that \emph{should} use both \textsc{T2S} and \textsc{RAG}; for queries where no sub-step requires both tools, the score is \textsc{NA}. As a result, it is based on comparatively few (edge-case) instances, even though each judge individually attains $94.12\%$ agreement with humans on this metric.

\paragraph{Takeaway.}
Taken together, these results show that (i) multiple judge models from different families can be calibrated to our rubric and agree well with humans, and (ii) planner rankings and per-metric assessments are broadly consistent across Sonnet-4 and GPT-5. This directly mitigates concerns about strong model-family bias or prompt-specific artifacts: Sonnet-4 is a slightly stronger one-shot judge, but our main comparative conclusions about planner quality do not hinge on this particular choice.

\subsection{Models and Prompts}
\label{app:model-configs}

\paragraph{Query generation model.}
All benchmark queries (for Train, Validation, and Test) are synthesized with \textbf{GPT-4o} \citep{openai2024gpt4ocard} using our query-generation prompt in Table~\ref{tab:prompts_for_query_and_plan_gen}. The prompt explicitly conditions on the axes defined in Appx.~\ref{app:query-gen} (subjectivity and compoundness) so that GPT-4o produces queries that span objective vs.\ subjective asks and simple vs.\ compound structures. Table~\ref{tab:examples_by_query_types} illustrates representative GPT-4o queries for each bucket.

\paragraph{Plan generation (14 LLMs).}
We evaluate the following models: \textit{Claude-3-7-Sonnet, Claude-Sonnet-4, Claude-3-5-Haiku, Nova-Pro, Nova-Lite, Nova-Micro, Llama3-2-1B-Instruct, Llama3-2-3B-Instruct, Llama3-70B-Instruct, Llama4-Maverick-17B-Instruct, GPT-4o, GPT-4o-Mini, GPT-4.1-Nano, o3-Mini (medium reasoning)}.
Two prompts are used per query: (\emph{i}) \textbf{without lineage} (task + tool specs + few-shot exemplars) and (\emph{ii}) \textbf{with lineage} (adds per-query lineage exemplars in the reference section). Full prompt templates are in Table~\ref{tab:prompts_for_query_and_plan_gen}.

\paragraph{Evaluation LLMs.}
We use \textit{Claude-Sonnet-4} as the Judge LLM for both evaluation modes:
\begin{enumerate}
    \item \textbf{Metric-wise evaluators} (seven specialist rubrics; Sec.~\ref{sec:evaluation}, Appx.~\ref{app:metric-details}).
    \item \textbf{One-shot overall evaluator} (Precision/Recall/F$_1$, Format, Dependencies, Placeholders, 7-point rating; Appx.~\ref{app:oneshot-details}).
\end{enumerate}

\paragraph{Feedback-loop modules.}
The Step-wise Evaluator and Plan Optimizer (Appx.~\ref{app:loop-details}) are both instantiated with \textit{GPT-4o}. We set \texttt{max\_passes}=\textbf{4} by default; this provided a favorable accuracy/latency trade-off in pilot runs. We cap per-pass budget (LLM calls and wall time) to avoid degenerate loops.

\subsection{Decoding, Seeds, and Hyperparameters}
\label{app:decoding}
We apply deterministic decoding for evaluation LLMs (temperature $\in\{0\}$; top-$p=1.0$) and low-to-moderate temperature for generators (per-model defaults; tabulated below). All runs fix random seeds. We enable response validation (JSON parsing) with retry-on-format for evaluators.

\begin{table*}[t]
\centering
\small
\begin{tabular}{lccccc}
\toprule
\textbf{Module} & \textbf{Model} & \textbf{Source} & \textbf{Temperature} & \textbf{Top-p} & \textbf{Max Tokens} \\
\midrule
Query Gen & GPT-4o & Azure & 0.2 & 1.0 & 4096 \\
Plan Gen (all 14) & (varies) & Azure, Bedrock & 0.2 & 1.0 & 4096 \\
Metric-wise Eval & Claude-Sonnet-4 & Bedrock & 0.0 & 1.0 & 4096 \\
One-shot Judge & Claude-Sonnet-4 & Bedrock & 0.0 & 1.0 & 4096 \\
Step-wise Eval & GPT-4o & Bedrock & 0.0 & 1.0 & 4096 \\
Plan Optimizer & GPT-4o & Bedrock & 0.0 & 1.0 & 4096 \\
\bottomrule
\end{tabular}
\caption{LLM configurations used across modules}
\label{tab:decoding}
\end{table*}

\begin{table*}[t]
\centering
\small
\begin{tabular}{l r}
\toprule
Metric & Value \\
\midrule
Average \# distinct plans / lineage & 5.5 \\
Average \# passes / query & 3.0 \\
Average \# revisions / pass (accepted) & 1.5 \\
Average \# steps / initial plan & 5.4 \\
Average \# steps / best plan & 4.8 \\
Average \# distinct tool types / initial plan & 2.4 \\
Average \# distinct tool types / best plan & 2.2 \\
\bottomrule
\end{tabular}
\caption{Summary statistics of the iterative evaluator$\rightarrow$optimizer loop (validation split). “Revisions/pass” counts \emph{accepted} edits.}
\label{tab:loop-stats}
\end{table*}

\subsection{Infrastructure and Reproducibility}
\label{app:infra}
We implement all components in Python. Calls are routed via Bedrock and LiteLLM with exponential backoff and per-provider rate limits. We cache intermediate LLM outputs keyed by (model, prompt, seed) and persist full artifacts (plans, lineages, evaluator JSON) for auditability. The model configurations used across modules are provided in Table~\ref{tab:decoding}.

\subsection{Additional Controls}
\label{app:controls}
We enforce (i) JSON validation with structured error messages to prompt repair; (ii) strict placeholder checks; (iii) timeouts and retries; and (iv) lineage provenance (every edited plan version is stored with a diff and reason tag from the Evaluator).

\section{Grouped Analyses for Plan Generation Quality}
\label{app:grouped-analyses}

\subsection{Objective vs.\ Subjective}
\label{app:obj-subj}
\paragraph{One-shot (Tables~\ref{tab:plan_gen_result_objective_vs_subjective_with_lineage_one_shot_eval}--\ref{tab:plan_gen_result_objective_vs_subjective_without_lineage_one_shot_eval})}
With lineage: 7/14 models have higher \emph{A+} proportions on Objective than Subjective; 6/14 are lower; 1 unchanged. 
Without lineage: 9/14 higher on Objective, 4/14 lower, 1 unchanged. 
\emph{Interpretation:} modest, model-dependent edge for Objective in the no-lineage setting; not conclusive overall.

\paragraph{Metric-wise (Table~\ref{tab:plan_gen_result_object_and_subject_queries_metric_wise_eval})}
With lineage: 11/14 models score higher on Subjective than Objective; without lineage: 9/14 higher on Subjective. 
\emph{Interpretation:} metric decomposition favors Subjective for many models, despite one-shot results showing only a weak/no advantage for Objective.

\subsection{Simple vs.\ Compound}
\label{app:simple-compound}
\paragraph{One-shot (Tables~\ref{tab:plan_gen_result_simple_vs_compound_with_lineage_one_shot_eval}--\ref{tab:plan_gen_result_simple_vs_compound_without_lineage_one_shot_eval})}
With lineage: 12/14 higher \emph{A+} for Simple; without lineage: 13/14 higher for Simple. 
\emph{Interpretation:} strong and consistent advantage for Simple queries.

\paragraph{Metric-wise (Table~\ref{tab:plan_gen_result_simple_and_comp_queries_metric_wise_eval})}
With lineage: 11/14 higher overall for Simple; without lineage: 10/14 higher. 
\emph{Interpretation:} mirrors one-shot; complexity hurts.

\subsection{Plan Length: \texorpdfstring{\([1,4]\) vs. \([5,15]\)}{[1,4] vs. [5,15]} Steps in the Best Possible Plans}
\label{app:plan-length}
\paragraph{One-shot (Tables~\ref{tab:plan_gen_result_1to4steps_vs_5to15steps_with_lineage_one_shot_eval}--\ref{tab:plan_gen_result_1to4steps_vs_5to15steps_without_lineage_one_shot_eval})}
With lineage: 13/14 higher \emph{A+} for \([1,4]\); without lineage: 13/14 higher for \([1,4]\). 
\emph{Interpretation:} shorter gold plans are markedly easier to match.

\paragraph{Metric-wise (Table~\ref{tab:plan_gen_result_num_steps_metric_wise_eval})}
With lineage: 11/14 higher overall for \([1,4]\); without lineage: 12/14 higher. 
\emph{Interpretation:} consistent with one-shot; longer gold plans expose weaknesses in tool assignment and dependency wiring.

\subsection{Number of Hops: 0/1/2/3+}
\label{app:hops}
\paragraph{One-shot (Tables~\ref{tab:plan_gen_result_with_lineage_hops_one_shot_eval}--\ref{tab:plan_gen_result_without_lineage_hops_one_shot_eval})}
With lineage: 10/14 show higher \emph{A+} for One-Hop (vs.\ Two-Hop), while 3/14 favor Two-Hop. 
Without lineage: 6/14 favor One-Hop, 7/14 favor Two-Hop. 
\emph{Interpretation:} no stable, cross-model pattern.

\paragraph{Metric-wise (Table~\ref{tab:plan_gen_result_num_hops_metric_wise_eval})}
With lineage: 6/14 higher overall for One-Hop; 4/14 higher for Two-Hop. 
Without lineage: 9/14 higher for One-Hop; 3/14 higher for Two-Hop. 
\emph{Interpretation:} small, inconsistent edges - hops alone are not a robust predictor.

\paragraph{Summary.}
Across grouped factors, \textbf{query simplicity} and \textbf{shorter best-plan length} correlate most strongly with better quality. The effects of \textbf{lineage prompting} and \textbf{hops} are mixed and model dependent.

\section{Grouped Analyses: Effectiveness of the Evaluator\texorpdfstring{$\rightarrow$}{->}Optimizer Loop}
\label{app:loop-grouped}

\subsection{Objective vs.\ Subjective}
\label{app:loop-obj-subj}
\paragraph{Objective (Table~\ref{tab:feedback_loop_performance_objective_subjective})}
\emph{Extremely Good} increases from \textbf{4.15\%} (pre) to \textbf{9.54\%} (post); \emph{Very Good} from \textbf{2.07\%} to \textbf{10.37\%}.  
\emph{Observation:} substantial uplift on top brackets; loop strongly benefits precision-oriented, verifiable tasks.

\paragraph{Subjective (Table~\ref{tab:feedback_loop_performance_objective_subjective})}
\emph{Extremely Good} moves from \textbf{5.02\%} (pre) to \textbf{6.95\%} (post); \emph{Very Good} from \textbf{15.83\%} to \textbf{18.53\%}.  
\emph{Observation:} consistent, moderate improvement; gains smaller than Objective but still meaningful.

\subsection{Simple vs.\ Compound}
\label{app:loop-simple-compound}
\paragraph{Simple (Table~\ref{tab:feedback_loop_performance_simple_compound})}
\emph{Extremely Good} improves from \textbf{5.58\%} (pre) to \textbf{11.15\%} (post); \emph{Very Good} from \textbf{11.15\%} to \textbf{17.84\%}.  
\emph{Observation:} total uplift of \(\sim\)\textbf{12.26pp} across top two buckets; loop is highly effective when plans are short and atomic.

\paragraph{Compound (Table~\ref{tab:feedback_loop_performance_simple_compound})}
\emph{Extremely Good} rises from \textbf{3.46\%} (pre) to \textbf{4.76\%} (post); \emph{Very Good} from \textbf{6.49\%} to \textbf{10.82\%}.  
\emph{Observation:} clear improvement (\(\sim\)\textbf{5.63pp}); smaller than Simple, but still positive given higher compositional load.

\subsection{Hop Count (0/1/2/3+)}
\label{app:loop-hops}
\paragraph{Zero and One hop (Table~\ref{tab:feedback_loop_performance_zerohop_onehop})}
\textbf{0-hop}: \emph{Extremely Good} stays at \textbf{7.95\%}, while \emph{Very Good} increases from \textbf{0.0\%} to \textbf{9.09\%}.  
\textbf{1-hop}: \emph{Extremely Good} from \textbf{4.93\%} (pre) to \textbf{9.85\%} (post); \emph{Very Good} from \textbf{13.79\%} to \textbf{18.72\%}.  
\emph{Observation:} noticeable gains even for short-hop plans, particularly in the \emph{Very Good} bracket.

\paragraph{Two and Three-plus hops (Table~\ref{tab:feedback_loop_performance_twohops_threeplushops})}
\textbf{2-hop}: \emph{Extremely Good} from \textbf{3.57\%} to \textbf{7.86\%}; \emph{Very Good} from \textbf{9.29\%} to \textbf{12.86\%}.  
\textbf{3+ hop}: \emph{Extremely Good} from \textbf{0.00\%} to \textbf{4.35\%}; \emph{Very Good} from \textbf{7.25\%} to \textbf{14.49\%}.  
\emph{Observation:} largest \emph{relative} uplifts are in \textbf{3+ hop} and \textbf{2-hop} settings, indicating the loop is most beneficial when plans require longer, more error-prone compositions.

\paragraph{Summary.}
Across all strata, the \emph{Evaluator\texorpdfstring{$\rightarrow$}{->}Optimizer} loop consistently shifts mass from mid/low buckets into \emph{Very Good}/\emph{Extremely Good}, with the strongest relative effects for \emph{Objective}, \emph{Simple}, and \emph{3+ hop} cases.

\section{Module Validation on Validation Set (N=80): Full Results}
\label{app:module-validation}

\subsection{Metric-wise Evaluator Validation}
\label{app:metricwise-validation}

\paragraph{Setup recap.}
For each validation query, we take the final three plans in its lineage (highest quality, with the last being the human-verified best). Humans annotate \emph{per-metric} and \emph{overall} orderings among these three plans (relaxed inequality). The LLM evaluators score the same triple; we convert scores to an ordering and compute \emph{relaxed triplet ranking agreement} (match if human $a\!\le\!b\!\le\!c$ is predicted as $a\!\le\!b\!\le\!c$, allowing ties).

\textbf{Method.}
We evaluate two designs:
(1) \emph{Single}: a unified prompt defines all seven metrics and requests unweighted metric scores + rationale.
(2) \emph{Deconstructed}: seven specialized prompts, one per metric.
Each design is run in \emph{reference-free} (query + plan) and \emph{reference-based} (query + plan + gold plan) modes.
For each metric, we compute the percent of correct triplet inequalities across the 80 queries.

\textbf{Results.}
Table~\ref{tab:validation_result_metric_wise_eval} reports agreement by metric and setting. The \emph{deconstructed, reference-based} variant dominates, exceeding $90\%$ agreement for \textsc{Dependency}, \textsc{Format}, \textsc{Tool Usage Completeness} and surpassing $80\%$ for \textsc{Query Adherence}, \textsc{Redundancy}, \textsc{Tool-Prompt Alignment}. \textsc{Step Executability} achieves $79.43\%$. Reference-free settings trail as expected; \emph{Single} underperforms \emph{Deconstructed} across the board, indicating reduced interference when judging metrics in isolation.

\subsection{One-Shot Overall Evaluator Validation}
\label{app:oneshot-validation}

\textbf{Method.}
The judge LLM (Appendix~\ref{app:model-configs}) labels each plan among seven categories (\textit{Extremely Bad} to \textit{Extremely Good}). We compute per-class Precision/Recall/F1 against human labels and macro averages.

\textbf{Results.}
Table~\ref{tab:validation_result_one_shot_eval} shows macro Precision/Recall/F1 of $0.92/0.93/0.92$. All classes have F1~$\ge 0.85$, supporting reliable downstream use of the judge for plan quality decisions.

\subsection{Step-Wise Evaluator Validation}
\label{app:step-evaluator-validation}

\textbf{Method.}
On $N{=}400$ step instances (80 queries $\times$ $\approx$5 steps/plan), we compute multi-label Precision/Recall/F1 across tags: \textsc{No Change}, \textsc{Incorrect Tool}, \textsc{Incorrect Prompt}, \textsc{Complex Prompt}, \textsc{Repeated Detail}, \textsc{Multi-Tool Prompt} by comparing Judge LLM labels against human labels.

\textbf{Results.}
Table~\ref{tab:validation_result_step_wise_eval} reports macro F1~$=0.84$. \textsc{Incorrect Prompt} and \textsc{Repeated Detail} reach $0.91$ F1, indicating the evaluator reliably flags prompt-level errors and SQL redundancy. \textsc{No Change} is most challenging (F1~$=0.75$), reflecting conservative behavior when judging step sufficiency.

\subsection{Plan Optimizer Validation}
\label{app:plan-optimizer-validation}

\textbf{Method.}
For $160$ plan pairs (80 queries $\times$ 2 revisions per query), we compare the optimizer’s revision against the human gold using the tuned one-shot judge (Table~\ref{tab:prompts_for_oneshot_evaluator}). We report the distribution across seven quality tags.

\textbf{Results.}
As shown in Table~\ref{tab:validation_result_plan_optimizer}, $74.5\%$ of optimizer outputs fall into \textit{Good} or better (\textit{Extremely Good} $28.13\%$, \textit{Very Good} $24.38\%$, \textit{Good} $21.88\%$), and $10.63\%$ are \textit{Acceptable}; the remainder capture difficult edits where structural rewrites are needed. This supports using the loop to automatically lift initial plans to near-usable quality.

\paragraph{Interpretation.}
(i) Per-metric judgments are most stable when isolated and provided the gold plan (deconstructed, reference-based). (ii) The one-shot judge is precise enough for automatic triage. (iii) Step-wise tags are discriminative for prompt/tool issues and redundancy. (iv) The optimizer substantially improves plan quality relative to initial drafts, aligning with Sec.~\ref{sec:loop-effectiveness}.




\end{document}